\documentclass[letterpaper,11pt]{article}
\usepackage[utf8]{inputenc}
\usepackage{times}  
\usepackage{helvet} 
\usepackage{courier}  
\usepackage[hyphens]{url}  
\usepackage{graphicx} 
\usepackage[ bottom=1.25in, top=1.25in, left=1.25in, right=1.25in]{geometry}
\usepackage{float}
\usepackage{caption} 
\usepackage{authblk}
\usepackage{hyperref}

\title{Sampling Ex-Post Group-Fair Rankings}
\date{}
\author[$1$]{Sruthi Gorantla}
\author[$2*$]{Amit Deshpande}
\author[$1*$]{Anand Louis}

\affil[$1$]{\small{Indian Institute of Science, Bengaluru, India. \texttt{\{gorantlas, anandl\}@iisc.ac.in}}}
\affil[$2$]{\small{Microsoft Research, Bengaluru, India. \texttt{amitdesh@microsoft.com}}}

\usepackage{natbib}

\usepackage{amsthm, amssymb, amsmath}
\usepackage{algorithm}
\usepackage[ruled,linesnumbered,algo2e]{algorithm2e}

\SetCommentSty{mycommfont}
\usepackage[font=itshape]{quoting}

\usepackage{ soul, cleveref, xcolor}

\let\oldnl\nl
\newcommand{\nonl}{\renewcommand{\nl}{\let\nl\oldnl}}

\newtheorem{theorem}{Theorem}[section]
\newtheorem*{theorem*}{Theorem}

\newtheorem*{claim*}{Claim}

\newtheorem*{proposition*}{Proposition}
\newtheorem{lemma}[theorem]{Lemma}
\newtheorem*{lemma*}{Lemma}
\newtheorem{axiom}[theorem]{Axiom}
\newtheorem*{axiom*}{Axiom}
\newtheorem{corollary}[theorem]{Corollary}

\newtheorem*{conjecture*}{Conjecture}

\newtheorem*{fact*}{Fact}

\newtheorem*{hypothesis*}{Hypothesis}

\theoremstyle{definition}
\newtheorem{definition}[theorem]{Definition}

\newtheorem*{remark}{Remark}

\usepackage{enumitem}

\renewcommand{\le}{\leqslant}
\renewcommand{\geq}{\geqslant}
\renewcommand{\ge}{\geqslant}

\newcommand{\paren}[1]{\left(#1 \right )}

\newcommand{\sparen}[1]{\left[#1 \right ]}

\newcommand{\set}[1]{\left\{#1\right\}}

\newcommand{\abs}[1]{\left\lvert#1\right\rvert}

\newcommand{\ceil}[1]{\left\lceil #1 \right\rceil}
\newcommand{\floor}[1]{\left\lfloor #1 \right\rfloor}

\newcommand{\Z}{{\mathbb Z}}
\newcommand{\N}{{\mathbb Z}_{\geq 0}}
\newcommand{\R}{\mathbb R}

\newcommand{\I}{\mathbb{I}}

\newcommand{\E}{\mathbb E}

\newcommand{\cU}{\mathcal{U}}

\newcommand{\cP}{\mathcal{P}}

\newcommand{\cV}{\mathcal{V}}
\newcommand{\cH}{\mathcal{H}}
\newcommand{\cO}{\mathcal{O}}
\newcommand{\cD}{\mathcal{D}}
\newcommand{\cX}{\mathcal{X}}

\newcommand{\unif}{{\sc Sampling-Oracle}}
\newcommand{\vol}[2]{{\sf Vol}_{#1}\paren{#2}}
\newcommand{\jinl}{j \in [\ell]}

\newcommand\numberthis{\addtocounter{equation}{1}\tag{\theequation}}
\usepackage{caption}
\usepackage{subcaption}

\newcommand{\bY}{\boldsymbol{Y}}
\newcommand{\bX}{\boldsymbol{X}}
\newcommand{\vecY}{\paren{ Y_1, Y_2, \ldots, Y_k}}
\newcommand{\vecX}{\paren{X_1, X_2, \ldots, X_{\ell}}}
\newcommand{\vecy}{\paren{ y_1, y_2, \ldots, y_k}}
\newcommand{\vecx}{\paren{x_1, x_2, \ldots, x_{\ell}}}

\allowdisplaybreaks
\usepackage[parfill]{parskip}

\begin{document}

	\maketitle
	\def\thefootnote{*}\footnotetext{Equal contribution}\def\thefootnote{\arabic{footnote}}
	\begin{abstract}
		Randomized rankings have been of recent interest to achieve ex-ante fairer exposure and better robustness than deterministic rankings. We propose a set of natural axioms for randomized group-fair rankings and prove that there exists a unique distribution $\cD$ that satisfies our axioms and is supported only over ex-post group-fair rankings, i.e., rankings that satisfy given lower and upper bounds on group-wise representation in the top-$k$ ranks. Our problem formulation works even when there is implicit bias, incomplete relevance information, or only ordinal ranking is available instead of relevance scores or utility values.
		
		We propose two algorithms to sample a random group-fair ranking from the distribution $\cD$ mentioned above. Our first dynamic programming-based algorithm samples ex-post group-fair rankings uniformly at random in time $O(k^2\ell)$, where $\ell$ is the number of groups. Our second random walk-based algorithm samples ex-post group-fair rankings from a distribution $\delta$-close to $\cD$ in total variation distance and has expected running time $O^*(k^2\ell^2)$\footnote{$O^*$ suppresses logarithmic and error terms.}, when there is a sufficient gap between the given upper and lower bounds on the group-wise representation. The former does exact sampling, but the latter runs significantly faster on real-world data sets for larger values of $k$. We give empirical evidence that our algorithms compare favorably against recent baselines for fairness and ranking utility on real-world data sets.
	\end{abstract}
	
	\section{Introduction}
	
	Ranking individuals using algorithms has become ubiquitous in many applications such as college admissions \cite{BCCKP2019}, recruitment \cite{GAK2019}, and credit allocation \cite{schufa}, among others.
	In many such scenarios, individuals who belong to certain demographic groups based on race, gender, age, etc., face discrimination due to human and historical biases \cite{bias1,bias2,bias3}.
    Algorithms learning from biased data exacerbate representational harms for certain groups in the top ranks \cite{amazon,upturn}, leading to loss of opportunities.
	One way to mitigate representational harms is by imposing explicit representation-based fairness constraints that the ranking output by the algorithm must contain a certain minimum and maximum number of candidates from each group \cite{GAK2019,CMV2020}.
	Many fair processes such as the Rooney rule \cite{rooney}, the $4/5$-th rule \cite{fourfifths}, and fairness metrics such as demographic parity impose representation-based constraints.
	
	A large body of work has proposed deterministic post-processing of rankings to satisfy representation-based group-fairness constraints \cite{CSV2018,GAK2019,WZW2018,ZBCHMB2017,GDL2021}. These methods essentially merge the group-wise ranked lists to create a common ranking that satisfies representation-based constraints.
	However, these methods contain two critical flaws. First, deterministic rankings cannot create opportunities for \textit{all} the groups at the top, especially when the number of groups is large. Second, observations of merit are often noisy in the real world \cite{bias2} and could contain \emph{implicit bias} towards protected groups \cite{bias1}. Hence, inter-group comparisons of merit while merging the group-wise rankings could lead to a loss of opportunities for certain groups.
    For example, suppose multiple companies intend to hire for a limited number of similar open positions and use the same recruitment system to rank a common candidate pool for job interviews. 
    In a deterministic top-$k$ ranking based on biased merit scores, every company would see the same ranking, where the protected groups could be systematically ranked lower. Hence, equal representation at the top-$k$ may not translate into equal opportunities.

	We consider randomized group-fair rankings as a way to create opportunity for every group in the top ranks (or, more generally, any rank). We assume that we are given only the ordinal rankings of items within each group (i.e., intra-group ordering without any scores) and no comparison of items across different groups (i.e., inter-group comparisons). 
	This assumption is strong but circumvents implicit bias and allows us to consider group-fair rankings even under incomplete or biased data about pairwise comparisons.
    Our randomized ranking algorithms output ex-post group-fair rankings, i.e., every sampled ranking output is group-fair and satisfies representation-based fairness constraints as a stronger ex-post (or actual) guarantee instead of a weaker ex-ante (or expected) guarantee.

	The rest of the paper is organized as follows: In \Cref{sec:related_works}, we survey related work and summarize its limitations for our problem. In \Cref{sec:group_fairness} we describe our axiomatic approach to define a distribution over ex-post group-fair rankings (Axioms~\ref{axm:intra_group}-\ref{axm:inter_group}) and show that there is a unique distribution that satisfies our axioms (\Cref{thm:uniqueness}). The same distribution satisfies sufficient representation of every group in any consecutive ranks in the top-$k$ (\Cref{cor:rep_at_prefix}), a natural characteristic derived from our axioms. In \Cref{sec:algorithms}, we give an efficient dynamic programming-based algorithm (\Cref{alg:dp}) and a random walk-based algorithm (\Cref{alg:random_walk}) to sample an ex-post group-fair ranking from the above distribution.
    We also extend our algorithms to handle representation-based constraints on the top prefixes (\Cref{subsec:prefix}). 
    In \Cref{sec:experiments}, we empirically validate our theoretical and algorithmic guarantees on real-world datasets.
    Finally, \Cref{sec:conclusion} contains some limitations of our work and open problems. All the proofs and additional experimental results are included in the appendix.

    \section{Related Work}
    \label{sec:related_works}
    
	Representation-based fairness constraints for ranking $\ell$ groups in top $k$ ranks are typically given by numbers $L_j$ and $U_j$, for each group $j \in [\ell]$, representing the lower and upper bound on the group representation respectively.
	They capture several fairness notions; setting $L_j = U_j = \frac{k}{\ell}, \forall j\in[\ell]$ ensures equal representation for all groups (see section 5 of \cite{Zehlike_part1}), whereas $L_j = U_j = p_j\cdot k, \forall j \in [\ell]$, ensures proportional representation\footnote{It may not always be possible to satisfy equal or proportional representation constraints exactly. In that case the algorithms need to say that the instance is infeasible.}, where $p_j$ is the proportion of the group $j$ in the population \cite{GDL2021,GS2020,GAK2019}.
	These constraints have also been studied in other problems such as fair subset selection \cite{SYV2018}, fair matching \cite{GOTO201640}, and fair clustering \cite{CKLV2017}.

	Previous work \cite{Castillo2019} has tried to formalize the general principles of fair ranking as treating similar items consistently, maintaining a sufficient presence of items from minority groups, and proportional representation from every group. 
	There has been work on quantifying fairness requirements \cite{YS2017,GAK2019,BCDQ2019,NCGW2020,KVR2019}, which has predominantly proposed deterministic algorithms for group-fair ranking \cite{YS2017,GAK2019,GDL2021,CSV2018,Zehlike_part1}.

	Our recruitment example in the previous section shows the inadequacy of deterministic ranking to improve opportunities. Recent works have also observed this and proposed randomized ranking algorithms to achieve equality or proportionality of expected exposure \cite{DME+2020,SJ2018,BGW2018,robustltr,expohedron}.
    All of them require utilities or scores of the items to be ranked and hence, are susceptible to implicit bias or incomplete information about the true utilities.
    Moreover, they do not give ex-post guarantees on the representation of each group, which can be a legal or necessary requirement if the exposure cannot be computed efficiently and reliably \cite{MSdR2022}.

    Another recent workaround is to model the uncertainty in merit. Assuming access to the true merit distribution, \citet{SKJ2021} give a randomized ranking algorithm for a notion of individual fairness. On the other hand, \citet{CMV2020} try to model the systematic bias. Under strong distributional assumptions, they show that representation constraints are sufficient to achieve a fair ranking. However, their assumptions may not hold in the real world as unconscious human biases are unlikely to be systematic \cite{bias1,bias2}.

	Most aligned to our work is a heuristic randomized ranking algorithm -- \textit{fair $\epsilon$-greedy} -- proposed by \cite{GS2020}. Similar to our setup, they eschew comparing items across different groups. 
	However, their algorithm does not come with any theoretical guarantees, and it does not always sample ex-post group-fair rankings. Moreover, it works only when no upper-bound constraints exist on the group-wise representation.
	To the best of our knowledge, our work is the first to propose a distribution over ex-post group-fair rankings, using lower and upper bounds on the group-wise representations, and to give provably correct and efficient sampling algorithms for it.

    We note here that previous work has also used randomization in ranking, recommendations, and summarization of ranked results to achieve other benefits such as controlling polarization \cite{CKSV2019controlling}, mitigating data bias \cite{CKV2020}, and promoting diversity \cite{CKSDKV2018fair}.
	
	\section{Group Fairness in Ranking}
	\label{sec:group_fairness}
	
	Given a set $N := [n]$ of items, a \textit{top-$k$ ranking} is a selection of $k<n$ items followed by the assignment of each rank in $[k]$ to exactly one of the selected items.
 We use index $i$ to refer to a rank, index $j$ to refer to a group, and $a$ to refer to elements in the set $N$.
	Let $a,a'\in N$ be two different items such that the item $a$ is assigned to rank $i$ and item $a'$ is assigned to rank $i'$.
	Whenever $i < i'$ we say that item $a$ is \textit{ranked lower} than item $a'$.
	Going by the convention, we assume that being ranked at lower ranks gives items better visibility \cite{GDL2021}. 
	Throughout the paper, we refer to a top-$k$ ranking by just ranking.
	The set $N$ can be partitioned into $\ell$ disjoint groups of items depending on a sensitive attribute.
	A group-fair ranking is any ranking that satisfies a set of group fairness constraints. 
	Our fairness constraints are \textit{representation constraints}; lower and upper bounds, $L_j, U_j \in [k]$ respectively, on the number of top $k$ ranks assigned to group $j$, for each group $\jinl$.
	Throughout the paper, we assume that we are given a ranking of the items within the same group for all groups.
	We call these rankings \textit{in-group rankings}.
	We now take an axiomatic approach to characterize a random group-fair ranking.
	
	\subsection{Random Group-Fair Ranking}
	\label{subsec:definition}
	The three axioms we state below are natural consistency and fairness requirements for distribution over all the rankings.

	\begin{axiom}[In-group consistency]
		\label{axm:intra_group}
		For any ranking sampled from the distribution, for all items $a,a'$ belonging to the same group $j\in [\ell]$, item $a$ is ranked lower than item $a'$ if and only if item $a$ is ranked lower than item $a'$ in the in-group ranking of group $j$.
	\end{axiom}
    Since the intra-group merit comparisons are reliable, their in-group ranking must remain consistent, which is what the axiom asks for.
	Many post-processing algorithms for group-fairness ranking satisfy this axiom \cite{GDL2021,CSV2018,ZEHLIKE2022102707,ZBCHMB2017}.
	Once we assign the ranks to groups, \Cref{axm:intra_group} determines the corresponding items to be placed there consistent with in-group ranking.
	Hence, for the next axioms, we look at the group assignments instead of rankings.
	A \textit{group assignment} assigns each rank in the top $k$ ranking to exactly one of the $\ell$ groups.
	Let $Y_i$ be a random variable representing the group $i$th rank is assigned to.
	Therefore $\bY = \vecY$ is a random vector representing a group assignment.
	Let $y = \vecy$ represent an instance of a group assignment.
	A \textit{group-fair assignment} is a group assignment that satisfies the representation constraints.
	Therefore the set of group-fair assignments is $\set{y \in [\ell]^k : L_j \le \sum_{i \in [k]} \I[y_i = j] \le U_j, \forall \jinl}$, 
	where $\I[\cdot]$ is an indicator function.
	The ranking can then be obtained by assigning the items within the same group, according to their in-group ranking, to the ranks assigned to the group.
	We use $Y_0$ to represent a dummy group assignment of length $0$ for notational convenience when no group assignment is made to any group (e.g.~in \Cref{axm:inter_group}).
	
	Let $X_j$ be a random variable representing the number of ranks assigned to group $j$ in a group assignment for all $\jinl$.
	Therefore $\bX = \vecX$ represents a random vector for a \textit{group representation}.
	Let $x = \vecx$ represent an instance of a group representation.
	Then the set of group-fair representations is $\set{x \in \N^{\ell} :\sum_{j \in [\ell]} x_j = k~\text{and}~L_j \le x_j \le U_j, \forall \jinl}$.
	
	Since the inter-group comparisons are unreliable, any feasible group-fair representation is equally likely to be the best. That is, the distribution should be maximally non-committal distribution over the group-fair representations, which is nothing but a uniform distribution over all feasible group-fair representations.
	This is captured by our next axiom as follows,
	
	\begin{axiom}[Representation Fairness]
		\label{axm:rep_fairness}
		All the non-group-fair representations should be sampled with probability zero, and all the group-fair representations should be sampled uniformly at random. 
	\end{axiom}
	
	\begin{remark}
		Any distribution for top $k$ ranking that satisfies \Cref{axm:rep_fairness} is \textit{ex-post} group fair since the support of the distribution consists only of rankings that satisfy representation constraints.
		This is important when the fairness constraints are legal or strict requirements.
	\end{remark}
	
	Many distributions over rankings could satisfy \Cref{axm:intra_group} and \Cref{axm:rep_fairness}. 
	Consider a distribution that samples a group representation $x$ uniformly at random.
	Let $x_1\in [L_1,U_1]$ be the representation corresponding to group $1$.
	Let us assume that this distribution always assigns ranks $k-x_1+1$ to $k$ to group $1$.
	Due to in-group consistency, the best $x_1$ items in group $1$ get assigned to these ranks.
	However, always being at the bottom of the ranking is not fair to group $1$, since it gets low visibility.
	Therefore, we introduce a third axiom that asks for fairness in the second step of ranking -- assigning the top $k$ ranks to the groups in a \textit{rank-aware} manner.
	
	\begin{axiom}[Ranking Fairness]
		\label{axm:inter_group}
		For any two groups $j,j'\in [\ell]$, for all $i \in \set{0,\ldots, k-2}$, conditioned on the top $i$ ranks and a group representation $x$, the $(i+1)$-th and the $(i+2)$-th ranks are assigned to $j$ and $j'$ interchangeably with equal probability. That is, $\forall j, j' \in [\ell],\forall i \in \set{0,\ldots,k-2}$,
		\begin{equation*}
		\Pr\sparen{Y_{i+1} = j, Y_{i+2} = j' \mid Y_0, Y_1, \ldots, Y_i, \bX} =\Pr\sparen{Y_{i+1} = j', Y_{i+2} = j \mid Y_0, Y_1, \ldots, Y_i, \bX}.
		\end{equation*}
	\end{axiom}
	
	Let $\cU$ represent a uniform distribution. 
	In the result below, we prove that there exists a unique distribution over the rankings that satisfies all three axioms. 
	\begin{theorem}
		\label{thm:uniqueness}
		Let $\cD$ be a distribution from which a ranking is sampled as follows,
		\begin{enumerate}[leftmargin=*]
			\itemsep0em 
			\item Sample an $x$ as follows,\\
			$\bX\sim\cU\set{x\in \N^{\ell} :\sum\limits_{j \in [\ell]} x_j = k \text{ and }
				L_j \le x_j \le U_j, \forall \jinl}.$
			\item Sample a $y$, given $x$, as follows,\\
			$\bY\mid x\sim\cU\set{y\in [\ell]^{k}: \sum\limits_{i \in k}\I[y_i = j] = x_j, \forall \jinl}.$
		
			\item Rank the items within the same group in the order consistent with their in-group ranking, in the ranks assigned to the groups in the group assignment $y$.
		\end{enumerate}
		Then $\cD$ is the unique distribution that satisfies all three axioms.
	\end{theorem}
	\begin{proof}
	Recall that $x = \vecx$ is defined as \textit{group representation} where $x_j$ is the number of ranks assigned to group $j$ for all $\jinl$, and $y = \vecy$ is defined as \textit{group assignment} where $y_i$ is the group assigned to rank $i$ for all $i \in [k]$.
	
	For \Cref{axm:intra_group} to be satisfied, the distribution should consist only of rankings where the items from the same group are ranked in the order of their merit.
	Clearly $\cD$ satisfies \Cref{axm:intra_group}.
	
	To satisfy \Cref{axm:rep_fairness} all the group-fair representations need to be sampled uniformly at random, and all the non-group-fair rankings need to be sampled with probability zero.
	Hence, $\cD$ also satisfies \Cref{axm:rep_fairness}.
	
	We now use strong induction on the prefix length $i$ to show that any distribution over group assignments that satisfies \Cref{axm:inter_group} has to sample each group assignment $y$, conditioned on a group representation $x$, with equal probability.
	We note that whenever we say a common prefix, we refer to the longest common prefix.
	
	\textbf{Induction hypothesis.}
	Any two rankings with a common prefix of length $i$, for some $0 \le i \le k-2$, have to be sampled with equal probability.
	
	\textbf{Base case ($i = k-2$).}
	Let $y$ and $y'$ represent a pair of group assignments with fixed group representation $x$ and common prefix till ranks $k-2$.
	Then there exist exactly two groups $j,j'\in [\ell]$ such that
	\begin{gather*}
	y_{k-1} = y'_{k} = j\quad\text{and}\quad
	y_{k} = y'_{k-1} = j'.
	\end{gather*}
	Therefore, to satisfy \Cref{axm:inter_group}, these two group assignments $y$ and $y'$ need to be sampled with equal probability.
	Therefore we can conclude that for a fixed $x$, any two group assignments with the same prefix of length $k-2$ have to be sampled with equal probability.
	We note here that there do not exist two or more group assignments with group representation $x$ and common prefix of length exactly $k-1$.
	
	\textbf{Induction step.}
	Assume that for some $i < k-2$, any two group assignments with group representation $x$ and common prefix of length $i' \in \set{i+1, i+2, \ldots, k-2}$ are equally likely.
	Then we want to show that any two group assignments with group representation $x$ and common prefix of length $i$ are also equally likely.
	Let $y^{(s)}$ and $y^{(t)}$ be two different group assignments with group representation $x$ and common prefix of length $i$.
	Let $w = \paren{w_1, w_2, \ldots, w_i}$ represent this common prefix of length $i$, that is,
	\[
	w_1 := y^{(s)}_1 = y^{(t)}_1, w_2 := y^{(s)}_2 = y^{(t)}_2, \cdots, w_i := y^{(s)}_i = y^{(t)}_i.
	\]
	
	Observe that if $x_j'$ represents the number of ranks assigned to group $j$ in ranks $\paren{i+1, i+2, \ldots, k}$ in $y^{(s)}$, then the number of ranks assigned to group $j$ in ranks $\paren{i+1, i+2, \ldots, k}$ in $y^{(t)}$ is also $x_j'$ for all $\jinl$, since $y^{(s)}$ and $y^{(t)}$ have common prefix of length $i$, and both have group representation $x$.

	Since $w$ is of length exactly $i$ we also have that $y^{(s)}_{i+1}\neq y^{(t)}_{i+1}$.
	But the observation above give us that the group assigned to rank $i+1$ in $y^{(t)}$ appears in one of the ranks between $i+2$ and $k$ in $y^{(s)}$.
	Let $\cP$ be the set of all permutations of the elements in the multi-set 
	\[
	\set{y^{(s)}_{i+2},y^{(s)}_{i+3}, \ldots, y^{(s)}_{k}}\Big\backslash \set{y^{(t)}_{i+1}},
	\]
	that is, we remove one occurrence of the group assigned to rank $i+1$ in the group assignment $y^{(t)}$ from the multi-set $\set{y^{(s)}_{i+2},y^{(s)}_{i+3}, \ldots, y^{(s)}_{k}}$.
	We then have that each element of $\cP$ is a tuple of length $k-i-2$.
	We now construct two sets of group assignments $M^{(s)}$ and $M^{(t)}$ as follows,
	\begin{align*}
	M^{(s)} := \Bigg\{ \Big\{\underbrace{w_1, w_2, \ldots, w_i}_{\text{first}~i},\underbrace{y^{(s)}_{i+1}}_{i+1},\underbrace{y^{(t)}_{i+1}}_{i+2}, \underbrace{\hat{w}_1, \hat{w}_2, \ldots, \hat{w}_{k-i-2}}_{\text{last}~k-i-2}\Big\},\forall \hat{w} \in \cP\Bigg\},
    \end{align*}
    \begin{align*}
	M^{(t)} := \Bigg\{ \Big\{\underbrace{w_1, w_2, \ldots, w_i}_{\text{first}~i},\underbrace{y^{(t)}_{i+1}}_{i+1},\underbrace{y^{(s)}_{i+1}}_{i+2}, \underbrace{\hat{w}_1, \hat{w}_2, \ldots,  \hat{w}_{k-i-2}}_{\text{last}~k-i-2}\Big\}, \forall \hat{w} \in \cP\Bigg\}.
	\end{align*}
	For a fixed $\hat{w}\in \cP$ there is exactly one group assignment in  $M^{(s)}$ and one group assignment in $M^{(t)}$ such that their $i+1$st and $i+2$nd coordinates are interchanged, and their first $i$ and last $k-i-2$ coordinates are same.
	Therefore, $\abs{M^{(s)}} = \abs{M^{(t)}}$.
	
	We also have from the induction hypothesis that all the group assignments in $M^{(s)}$ are equally likely since they have a common prefix of length $i+2$.
	Similarly all the group assignments in $M^{(t)}$ are equally likely.
	For any group assignment in $M^{(s)}$ let $\delta^{(s)}$ be the probability of sampling it.
	Similarly, for any group assignment in $M^{(t)}$ let $\delta^{(t)}$ be the probability of sampling it.
	Then,
	\begin{multline}
	\Pr\sparen{Y_{i+1} = y^{(s)}_{i+1}, Y_{i+2} = y^{(t)}_{i+1} \mid Y_0, \paren{Y_1, \ldots, Y_i} = w, \bX=x} \\
	= \Pr\sparen{\text{sampling a group assignment from}~M^{(s)}} = \abs{M^{(s)}}\delta^{(s)},\label{eq:ms}
	\end{multline}
	\begin{multline}
	\Pr\sparen{Y_{i+1} = y^{(t)}_{i+1}, Y_{i+2} = y^{(s)}_{i+1} \mid Y_0, \paren{Y_1, \ldots, Y_i} = w,  \bX=x} \\= \Pr\sparen{\text{sampling a group assignment from}~M^{(t)}} = \abs{M^{(t)}}\delta^{(t)}.\label{eq:mt}
	\end{multline}
	Fix two group assignments  $y^{(s')}\in M^{(s)}$ and  $y^{(t')}\in M^{(t)}$.
	By the induction hypothesis $y^{(s)}$ and $y^{(s')}$ are equally likely since they have a common prefix of length $i+1$.
	Similarly $y^{(t)}$ and $y^{(t')}$ are also equally likely.
	Therefore, for $y^{(s)}$ and $y^{(t)}$ to be equally likely we need $y^{(s')}$ and $y^{(t')}$ to be equally likely.
	
	\paragraph{Comparing $y^{(s')}$ and $y^{(t')}$ instead of $y^{(s)}$ and $y^{(t)}$.}
	We know from above that $y^{(s')}$ and $y^{(t')}$ are sampled with probability $\delta^{(s)}$ and $\delta^{(t)}$ respectively.
	Therefore for any distribution satisfying \Cref{axm:inter_group} we have,
	\begin{multline*}
	\Pr\sparen{Y_{i+1} = y^{(s)}_{i+1}, Y_{i+2} = y^{(t)}_{i+1} \mid Y_0, Y_{1:i} = w, \bX=x} \\
	= \Pr\sparen{Y_{i+1} = y^{(t)}_{i+1}, Y_{i+2} = y^{(s)}_{i+1} \mid Y_0, Y_{1:i} = w,  \bX=x}\\
	\implies \abs{M^{(s)}}\delta^{(s)} = \abs{M^{(t)}}\delta^{(t)},~~\text{from Equations}~(\ref{eq:ms})~\text{and}~(\ref{eq:mt}) \\
	\implies \delta^{(s)} = \delta^{(t)},\qquad\qquad\because \abs{M^{(s)}}=\abs{M^{(t)}}.
	\end{multline*}
	Note that the converse is also easy to show, which means that \Cref{axm:inter_group} is satisfied if and only if $y^{(s')}$ and $y^{(t')}$ are equally likely.
	Therefore, \Cref{axm:inter_group} is satisfied if and only if $y^{(s)}$ and $y^{(t)}$ are equally likely.
	
	For a fixed group representation $x$, for any two group assignments with corresponding group representation $x$, there exists an $i \in \set{0,1,\ldots, k-2}$ such that they have a common prefix of length $i$.
	Therefore, any two group assignments, for a fixed group representation $x$, have to be equally likely.
	Therefore $\cD$ is the unique distribution that satisfies all three axioms.
\end{proof}

	We also have the following additional characteristic of the distribution in \Cref{thm:uniqueness}.
	It guarantees that every rank in a randomly sampled group assignment is assigned to group $j$ with probability at least $\frac{L_j}{k}$ and at most $\frac{U_j}{k}$.
	Hence, every rank gets a sufficient representation of each group.
	Note that no deterministic group-fair ranking can achieve this.
	
	Let $\cD_{\delta}$ be a distribution that differs from $\cD$ as follows: $\bX$ is sampled from a distribution $\delta$-close to a uniform distribution in Step 1 of $\cD$, in the total-variation distance, $\bY|x$ is sampled as in Step 2 of $\cD$. The items are also assigned as in Step 3 of $\cD$.
	Then it is easy to show that $\cD_{\delta}$ is $\delta$-close to $\cD$ in total-variation distance.
	We then prove the following theorem and its corollary.
	
	\begin{theorem}
		\label{thm:rep_at_i}
		For any $\delta > 0$ and group assignment $\bY$ sampled from $\cD_{\delta}$, for every group $\jinl$ and for every rank $i \in [k]$, $
		\frac{L_j}{k} \le \Pr_{\cD_{\delta}}\sparen{Y_i = j}\le\frac{U_j}{k}$.
	\end{theorem}

\begin{proof}
	Given an $\delta > 0$ and a distribution $\cD_{\delta}$ that is at total-variation distance of $\delta$ from $\cD$ defined in \Cref{thm:uniqueness}, when sampling group representation.
	Therefore,
	\begin{equation}
	\label{eq:tv_for_cX}
	\sup_{A \subseteq \cX}\abs{\Pr_{\cD}(A) - \Pr_{\cD_{\delta}}(A)} = \delta.    
	\end{equation}
	Now, fix a group $\jinl$ and a rank $i \in [k]$.
	Let $\cX$ be the set of all group-fair representations for given constraints, $L_j, U_j, \forall \jinl$.
	Then,
	\begin{align*}
	\Pr_{\cD_{\delta}}\sparen{Y_i = j} &= \sum_{x\in\cX}\Pr_{\cD_{\delta}}\sparen{X = x}\Pr_{\cD_{\delta}}\sparen{Y_i = j | X} \\&~~\text{(by the law of total probability)}\\
	&= \sum_{x\in\cX}\Pr_{\cD_{\delta}}\sparen{X = x}\frac{x_j}{k}\\
	&\ge \frac{L_j}{k}\sum_{x\in\cX}\Pr_{\cD_{\delta}}\sparen{X = x}\\
	&= \frac{L_j}{k}.
	\end{align*}
	Similarly we get $\Pr_{\cD_{\delta}}\sparen{Y_i = j} \le \frac{U_j}{k}$. 
\end{proof}

	\begin{corollary}
		\label{cor:rep_at_prefix}
		For any $\delta > 0 $, let $i, i' \in [k]$ be such that $i \le i'$ and let $Z_{i,i'}^j$ be a random variable representing the number of ranks assigned to group $j$ in ranks $i$ to $i'$,
		for a ranking sampled from $\cD_{\delta}$.
		Then, for every group $\jinl$ and for every rank $i \in [k]$, $
		\paren{\frac{i'-i+1}{k}}\cdot L_j \le \E_{\cD_{\delta}}\sparen{Z_{i,i'}^j}\le\paren{\frac{i'-i+1}{k}}\cdot U_j$.
	\end{corollary}
	
\begin{proof}[Proof of \Cref{cor:rep_at_prefix}]
	Given an $\delta > 0$ and a distribution $\cD_{\delta}$ that is at total-variation distance of $\delta$ from $\cD$ defined in \Cref{thm:uniqueness}, when sampling group representation.
	Fix a group $\jinl$ and rank $i,i' \in [k]$ such that $i \le i'$.
	Let $\cX$ be the set of all group-fair representations for given constraints, $L_j, U_j, \forall \jinl$.
	\begin{align*}
	\E_{\cD_{\delta}}\sparen{Z_{i,i'}^j} &= \E_{\cD_{\delta}}\sparen{\sum_{\hat{i} = i}^{i'}\I\sparen{Y_{\hat{i}} = j}} \\ 
	&= \sum_{\hat{i} = i}^{i'}\E_{\cD_{\delta}}\sparen{\I\sparen{Y_{\hat{i}} = j}} &\text{by linearity of expectation}\\
	&=  \sum_{\hat{i} = i}^{i'}\Pr_{\cD_{\delta}}\sparen{Y_{\hat{i}} = j}\\
	&\ge \frac{i'-i+1}{k}\cdot L_j. &\text{from \Cref{thm:rep_at_i}} 
	\end{align*}
	Similarly $\E_{\cD_{\delta}}\sparen{Z_{i,i'}^j} \le \frac{i'-i+1}{k}\cdot U_j$.
\end{proof}
	Two comments are in order. 
	First, fixing $i = 0$ in \Cref{cor:rep_at_prefix} gives us that every prefix of the ranking sampled from $\cD_{\delta}$ will have sufficient representation from the groups, in expectation.
	Such fairness requirements are consistent with those studied in the ranking literature \cite{CSV2018}.
	Second, let $k' := i' - i$ for some $i, i' \in [k]$ such that $i \le i'$.
	Then \Cref{cor:rep_at_prefix} also gives us that any consecutive $k'$ ranks of the ranking sampled from $\cD_{\delta}$ also satisfy representation constraints.
	Such fairness requirements are consistent with those studied in \cite{GDL2021}.

    In the ranking, one might ask for different representation requirements for different prefixes of the ranking. We extend our algorithms to handle prefix fairness constraints in the next section (see \Cref{subsec:prefix}).

	\paragraph{Time taken to sample from $\cD$ (in \Cref{thm:uniqueness}).}
	Given a group representation $x$, one can find a uniform random $y$, for a given $x$, by sampling a random binary string of length at most $\log k!$. 
	Since $x$ is fixed, this gives us a uniform random sample of $y$ conditioned on $x$.
	This takes time $\cO(k \log k)$.
	This sampling takes care of Step 2.
	Step 3 simply takes $\cO(k)$ time, given in-group rankings of all the groups.
	The main challenge is to provide an efficient algorithm to perform Step 1.
	Therefore, in the next section, we focus on sampling a uniform random group-fair representation in Step 1.
	
	\section{Sampling a Uniform Random Group-Fair Representation}
	\label{sec:algorithms}
	We first note that each group-fair representation corresponds to a unique integral point in the convex polytope $K$ defined below,
	\begin{equation}
	K = \Big\{x\in\R^{\ell}\Big|\sum_{j \in [\ell]} x_j = k,~L_j \le x_j \le U_j, \forall j \in [\ell] \Big\}.\label{eq:polytopeK}
	\end{equation}
	Therefore, sampling a uniform random group-fair representation is equivalent to sampling an integral or a lattice point uniformly at random from the convex set $K$.
	
	\subsection{Dynamic Programming for Exact Sampling}
	
	\begin{algorithm}[t]
		\KwIn{Fairness constraints $L_j, U_j$ for all the groups $j \in [\ell]$, a number $k \in \N$}
		\nonl \hrulefill
		
		\textbf{Initialize:} Set $D[k', i] := 0, \forall k' = \set{0,1,\ldots, k}$ and $\forall i \in \set{0,1,\ldots, \ell}$, and $D[0,0] := 1$\label{step:init}
		
		\tcp{Counting}
		
		\For{$ k' = 0 \text{ to } k$}{\label{step:counting_start}
			\For{$ i = 1 \text{ to } \ell$}{
				$D[k', i] = \sum_{L_i \le x_i \le U_i} D[k'-x_i, i-1]$\label{step:update}
			}
		}\label{step:counting_end}
		
		\tcp{Sampling}
		
		Set $k' := k$ and  $i := \ell$
		
		\While{$i \neq 0$}
		{\label{step:sampling_start}
			Sample $x_i$ from the categorical distribution on ${L_i, L_{i+1}, \ldots, U_i}$ with corresponding probabilities $\frac{D[k' - L_i, i-1]}{D[k',i]}, \frac{D[k' - L_{i+1}, i-1]}{D[k',i]}, \ldots, \frac{D[k' - U_i, i-1]}{D[k',i]} $\label{step:cat}
			
			\tcp{by convention $D[t, i] = 0$ whenever $t<0$.}
			
			Update $k' := k' - x_i$ and $i := i - 1$
		}\label{step:sampling_end}
		\caption{Sampling a uniform random group-fair representation}
		\label{alg:dp}
	\end{algorithm}

	In this section, we give a dynamic programming-based algorithm (see \Cref{alg:dp}) for uniform random sampling of integer points from the polytope $K$.
	Each entry $D[k', i], \forall k' = \set{0,1,\ldots, k}$ and $\forall i \in \set{0,1,\ldots, \ell}$ in \Cref{alg:dp} corresponds to the number of integer points in $K_{i,k'} = \set{x \in \R^i \mid \sum_{h \in [i]}x_h = k', L_h \le x_h \le U_h, \forall h \in [i]}$.
	That is, the DP table keeps track of the number of feasible solutions that sum to $k'$ with the first $i$ groups.	
	Therefore, $D[k, \ell]$ contains all feasible integer points of $K_{\ell, k}$, which is nothing but $K$ defined in \Cref{eq:polytopeK}.
	The reader should note that the entry $D[0,i] = 1$ if assigning $0$ to the first $i$ groups is feasible with respect to the fairness constraints and $0$ otherwise.
	However, the entry $D[k', 0]$ is always $0$ for $k' > 0$, since we can not construct a ranking of non-zero length without assigning the ranks to any of the groups.
	In Step~\ref{step:init}, we initialize all the entries of the DP to $0$ except for the entry $D[0,0]$, which is set to $1$. 
	Steps~\ref{step:counting_start} to \ref{step:counting_end} then count the number of feasible solutions for $D[k', i]$ by recursively summing over all feasible values for $x_i$. We note that this DP is similar to the DP, given by \citet{SVV2012}, for counting $0/1$ knapsack solutions where the feasible values of an item are $0$ (including it) and $1$ (not including it).
	Now, let us assume that we have sampled the value of $x_i$ for all $i+1, i+2, \ldots, \ell$ for some $0 < i < \ell$, and let $k' := k - x_{i+1} - x_{i+2} - \cdots - x_\ell$.
	Then for any $x_i \in [L_i, U_i]$ the probability that we sample $x_i$ is given by the number of feasible solutions after fixing $x_i$, divided by the total number of solutions for $x_i \in [L_i, U_i]$, which is nothing but $\frac{D[k' - x_i, i-1]}{D[k', i]}$ (see Step~\ref{step:cat}).
	Therefore, expanding the probability of sampling a feasible solution $(x_1, x_2, \ldots, x_\ell)$ gives us a telescoping product that evaluates to $1/D[k,\ell]$.

	\begin{theorem}
		\label{thm:dp}
		\Cref{alg:dp} samples a uniform random group-fair representation in time $\mathcal{O}(k^2\ell)$.
	\end{theorem}
	\begin{proof}
	We first show by mathematical induction on $i$ that for any $k' \in \set{0, 1, \ldots, k}$, 
	\begin{align}
	D[k', i] = |\big\{(x_1, x_2, \ldots, x_i) \mid L_j \le x_j \le U_j, \forall j \in [i] \text{ and } x_1+x_2+\ldots+x_i = k'\big\}| \label{eq:dp_polytope}	\end{align}
	In the base case, $D[k', 1] = 1$ if $L_1 \le k' \le U_1$ because choosing $x_1 = k'$ gives us exactly one feasible integer solution.
	Let us assume that the hypothesis is true for every $k'$ and for every $i'$ such that $i' \le i < \ell$.
	Then for $i+1$, and for any $k'$, the feasible values of $x_{i+1}$ are in $[L_{i+1}, U_{i+1}]$.
	For each of these values of $x_{i+1}$, all the feasible solutions with the first $i$ groups that sum to $k' - x_{i+1}$ are feasible solutions for that value of $x_{i+1}$.
	By the induction hypothesis, this is exactly what $D[k'-x_{i+1}, i]$ stores.
	Therefore, for any $k'\in \set{0,1,\ldots, k}$, $D[k', i+1] = \sum_{L_i \le x_i \le U_i} D[k'-x_i, i-1]$ is the number of feasible integer solutions with the first $i+1$ groups that sum to $k'$, which is exactly what Step~\ref{step:update} is counting.
	Therefore, $D[k,\ell]$ counts the number of integer solutions in the polytope $K$.
	
	Now let $X$ be an integer random vector $\paren{X_1, X_2, \ldots, X_{\ell}} \in [k]^{\ell}$ representing the group representation.
	\begin{align*}
	\Pr&\sparen{\text{DP outputs}~X= (x_1,\ldots,x_\ell)} \\&= \Pr\Big[\text{DP outputs}~x_1 \land \text{DP outputs}~x_2
	\land \cdots \land \text{DP outputs}~x_\ell\Big]
	\\
	&= \prod_{i = 1}^\ell \Pr\Big[\text{DP outputs}~x_i \mid \text{DP output}~x_{i+1}, \cdots, \text{DP output}~x_\ell\Big]\\
	&= \prod_{i = 1}^\ell \frac{D[k-x_i-x_{i+1} - \cdots - x_\ell,~~i-1]}{D[k-x_{i+1} - \cdots - x_\ell,~~i]}.\numberthis{}\label{eq:product}
	\end{align*}

	We first show that the DP never samples infeasible solutions.
	To see this, observe that any tuple $(x_1, x_2, \ldots, x_{\ell})$ can be infeasible in two cases. One when there exists a group $j \in [\ell]$ such that the condition $L_j \le x_j \le U_j$ is not satisfied.
	Other case is when the summation constraint $x_1 + x_2 +\cdots + x_\ell = k$ is not satisfied.
	The former does not occur in the DP because for every $j \in [\ell]$ it only samples the values of $x_j \in [L_j, U_j]$. 
	For the latter, the product term in \Cref{eq:product} will have the count of the entry $D[k-\sum_{j \in [\ell]}x_j, 0]$, which is $0$ due to our initialization. Hence, such an $x$ is sampled with probability $0$.
	
	When $x$ is feasible, for any $k'\in \set{0,1, \ldots, k}$ and for each sampling step $i$, the DP samples $x_i \in [L_i, U_i]$ from a valid probability distribution because $\sum_{L_i \le x_i \le U_i} D[k'-x_i, i-1]/D[k',i] = 1$.
	Moreover, we have $x_1 + x_2 + \ldots + x_\ell = k$.
	Therefore the telescopic product in \Cref{eq:product} always gives $D[0, 0]/D[k, \ell]$. Due to our initialization, $D[0,0]=1$. Hence, the probability of sampling any feasible $x$ is $1/D[k, \ell]$.
	Therefore, this DP gives uniform random samples.
	
	Since the DP table is of size $k \ell$ and computing each entry takes time $\mathcal{O}(k)$, the \textit{counting} step takes time $\mathcal{O}\paren{k^2\ell}$.
	Sampling from categorical distribution of size at most $k$ in \Cref{step:cat} takes time $\mathcal{O}(k)$ and this step is run $\ell$ times. Hence, \textit{sampling} takes $\mathcal{O}(k\ell)$ amount of time.
\end{proof}

	\subsection{Approximate Uniform Sampling}

	Our second algorithm outputs an integral point from $K$, defined in \Cref{eq:polytopeK}, from a density that is close to the uniform distribution over the set of integral points in $K$, with respect to the total variation distance (see \Cref{alg:random_walk}).

	There is a long line of work on polynomial-time algorithms to sample a point approximately uniformly from a given convex polytope or a convex body \cite{DFK1991,LV2006,CV2018}. We use the algorithm by \cite{CV2018} as \unif~in \Cref{alg:random_walk}. 
	We get an algorithm with expected running time $\mathcal{O}^*(k^2\ell^2)$ to sample a close to uniform random group-fair representation (\Cref{thm:main}).

	\begin{theorem}
		\label{thm:main}
		Let $L_j, U_j\in \N, \forall \jinl$ be the fairness constraints and $k\in \N$ be the size of the ranking. Let $\Delta$ be as defined in \Cref{alg:random_walk}.
		Then for any non-negative number $\delta < e^{-2\frac{\ell\sqrt{\ell}}{\Delta}}$, \Cref{alg:random_walk} samples a random point from a density that is within total variation distance $\delta$ from the uniform distribution on the integral points in $K$ by making $1/\paren{e^{-2\frac{\ell\sqrt{\ell}}{\Delta}} - \delta}$ calls to the oracle in expectation.
		When $\delta$ is a non-negative constant, such that $\delta < e^{-2}$ and $\Delta = \Omega\paren{\ell^{1.5}}$, \Cref{alg:random_walk} calls the oracle only a constant number of times in expectation, and each oracle call takes time $\mathcal{O}^*\paren{k^2\ell^2}$.
	\end{theorem}
	
	\begin{algorithm}[t]
		\setcounter{AlgoLine}{0}
		\SetAlgoLined
		\KwIn{Fairness constraints $L_j, U_j$ for all the groups $j \in [\ell]$, numbers $k \in \N$ and $\delta$}
		
		\nonl \hrulefill
		
		$H := \set{\paren{x_1, x_2, \ldots, x_{\ell}}\in\R^{\ell}~~\middle\vert~~\sum_{\jinl} x_j = k}$\label{step:def_H}
		
		$P := \set{\paren{x_1, x_2, \ldots, x_{\ell}}\in\R^{\ell}~~\middle\vert~~L_j \le x_j \le U_j, \forall \jinl }$\label{step:def_P}
		
		$\Delta := \min\Bigg\{\floor{\frac{k- \paren{\sum_{\jinl}L_j}}{\ell}}, \floor{\frac{ \paren{\sum_{\jinl}U_j }-k}{\ell}},\min_{\jinl}\floor{\frac{U_j- L_j}{2}}\Bigg\}$\label{step:def_delta}

		$x^*_j := L_j + \Delta, \forall \jinl$\label{step:find_center_start}
		
		\For{$j := 1,2,\ldots, \ell$}
		{\label{step:for}
			
			\If{$\sum_{j' \in [\ell]}x_{j'}^* < k$}
			{$x^*_{j} := \min\set{k - \sum_{j' \neq j}x^*_{j'},~~U_{j} - \Delta}$\label{step:find_center_end}}}
		
		$K' := K - x^*$\label{step:def_K_prime}

		$z$ := \unif$\paren{ \paren{1+\frac{\sqrt{\ell}}{\Delta}}K',\delta}$\label{step:sample}
		\nonl

		\nl$\textbf{if }j \in \sparen{\abs{\sum_j \floor{z_j}}}, x_j := \ceil{z_j}; \textbf{ else }x_j := \floor{z_j}$\label{step:rounding}
		
		\textbf{if } $x\in K'$, return $x+x^*$, \textbf{ else } reject $x$ and go to Step \ref{step:sample}
		\label{step:reject}

		\caption{Sampling an approximately uniform random group-fair representation}
		\label{alg:random_walk}
	\end{algorithm}
	\subsubsection{Overview of \Cref{alg:random_walk} and the proof of \Cref{thm:main}}
	\label{subsec:overview}
	Let $H, P,$ and $\Delta$ be as defined in Steps \ref{step:def_H}, \ref{step:def_P} and \ref{step:def_delta} respectively. 
	Clearly, $K = H \cap P$.
	We first find an integral center in $x^* \in H \cap P$ (Steps \ref{step:find_center_start} to \ref{step:find_center_end}) such that there is a ball of radius $\Delta$ in $P$ (see \Cref{lem:ball_contained}) and translate the origin to this point $x^*$ (Step \ref{step:def_K_prime}).
	This ensures that there exists a bijection between the set of integral points in the translated polytope $K'$ and the original polytope $K$ (see proof of \Cref{thm:main}).
	
	We now sample a rational point $z$ uniformly at random from the expanded polytope $\paren{1+\frac{\sqrt{\ell}}{\Delta}}K'$, using \unif~ (Step~\ref{step:sample}).
	We then round the point $z$ to an integer point on $H'$ (Step~\ref{step:rounding}).
	We prove that our deterministic rounding algorithm ensures that the set of points in the expanded polytope that get rounded to an integral point on $H'$ is contained inside a cube of side length $2$ around this point (\Cref{lem:rounding}) and that this cube is fully contained in this expanded polytope (\Cref{lem:cube_contained}).
	\Cref{lem:same_size} gives us that for any two integral points $x$ and $x'$, there is a bijection between the set of points that get rounded to these points.
	Therefore, every integral point is sampled from a distribution close to uniform, given the \unif~samples any rational point in the expanded polytope from a distribution $\delta = 0.1$ close to uniform.
	If the rounded point belongs to $K'$, we accept, else we reject and go to \ref{step:sample}.
        We then lower bound the probability of acceptance.
        The algorithm is run until a point is accepted. Hence, the expected running time polynomial is inversely proportional to the probability of acceptance, which is exponential in $\ell$ in expectation.
	However, if $\Delta = \Omega\paren{\ell^{1.5}}$ and $\delta < e^{-2}$ the probability of acceptance is at least a constant.
	Note that the value of $R^2$ in \Cref{thm:vempala} for the polytope $K'$ is $k^2$.
	Therefore, the algorithm by \citet{CV2018} gives a rational point from $K'$, from a distribution close to uniform, in time $\cO^*\paren{k^2\ell^2}$.
	Therefore each oracle call in \Cref{thm:main} takes time $\cO^*(k^2\ell^2)$ when we use \Cref{thm:vempala} as the \unif.

    \noindent\textbf{Comparison with \citet{KV1997}.}
    \Cref{alg:random_walk} is inspired by the algorithm by \citet{KV1997} to sample integral points from a convex polytope from a distribution close to uniform.
	On a high level, their algorithm on polytope $K'$ works slightly differently when compared to our algorithm on $K'$.
	They first expand $K'$ by $\cO\paren{\sqrt{\ell\log\ell}}$, sample a rational point from a distribution close to uniform over this expanded polytope (similar to Step \ref{step:sample}), and use a probabilistic rounding method to round it to an integral point.
	If the integral point is in $K'$, accept; otherwise, reject and repeat the sampling step.
	Their algorithm requires that a ball of radius $\Omega\paren{\ell^{1.5}\sqrt{\log\ell}}$ lies entirely inside $K'$.
	We expand the polytope $K'$ by $\cO(\sqrt{\ell})$, sample a rational point from this polytope from a distribution close to uniform, and then deterministically round it to an integral point.
	If the integral point is in $K'$, accept; otherwise, reject and repeat the sampling step.
	Our algorithm only requires that a ball of radius $\Omega\paren{\ell^{1.5}}$ lies inside $P-x^*$ with center on $H - x^*$, where $P$, $H$ and $x^*$ are as defined in \Cref{alg:random_walk}.
	As a result, we get an expected polynomial time algorithm for a larger set of fairness constraints.
	We also note here that the analysis of the success probability of \Cref{alg:random_walk} is the same as that of the algorithm by \citet{KV1997}.

\subsubsection{Proof of \Cref{thm:main}}
\label{subsec:proof_main}
For completeness, we restate the Theorem 1.2 in \cite{CV2018} that states the running time and success probability of their uniform sampler, in \Cref{subsec:proof_main}.
\begin{theorem}[Theorem 1.2, \cite{CV2018}]
	\label{thm:vempala}
	There is an algorithm that, for any $\delta > 0$, $p > 0$, and any convex body $C \in \R^{d}$ that contains the unit ball and has $\E_C(\|X\|^2) = R^2$, with probability $1 - p$, generates random points from a density $\nu$ that is within total variation distance $\delta$ from the uniform distribution on $C$. In the membership oracle model, the complexity of each random point, including the first, is $\cO^*\paren{\max\set{R^2{d}^2, {d}^3}}$\footnote{The $\cO^*$ notation suppresses error terms and logarithmic
		factors.}.
\end{theorem}
\begin{lemma}
	\label{lem:ball_contained}
	$B(0,\Delta) \subseteq P'$.
\end{lemma}
\begin{proof}
	From the definition of $\Delta$ we have the following inequalities.
	\begin{align}
	\Delta &\le \floor{\frac{k- \paren{\sum_{\jinl}L_j}}{\ell}}\implies \Delta \le \frac{k- \paren{\sum_{\jinl}L_j}}{\ell}\implies \ell\cdot \Delta+\sum_{\jinl}L_j \le k \notag\\
	&\implies \sum_{\jinl}\paren{L_j+\Delta} \le k\label{eq:suml_delta},
	\\
	&\Delta \le \floor{\frac{\paren{\sum_{\jinl}U_j}-k}{\ell}}\implies \Delta \le \frac{\paren{\sum_{\jinl}U_j}-k}{\ell}\implies -\ell \cdot \Delta+\sum_{\jinl}U_j  \ge k\notag\\
	&\implies \sum_{\jinl}\paren{U_j-\Delta} \ge k\label{eq:sumu_delta},\\
	\text{and}\notag\\
	&\Delta \le \floor{\frac{U_j - L_j}{2}}\implies \Delta \le \frac{U_j - L_j}{2}\implies 2\Delta \le U_j - L_j\notag\\
 &\implies L_j + \Delta \le U_j - \Delta.\label{eq:l_lessthan_u}
	\end{align}
	
	\noindent To show that Steps \ref{step:find_center_start} to \ref{step:find_center_end} find the correct center we use the following loop invariant.
	\paragraph{Loop invariant.} At the start of every iteration of the \textbf{for} loop $x^*$ is an integral point such that $L_j + \Delta \le x_j^* \le U_j - \Delta, \forall \jinl$. 
	
	\begin{itemize}
		\item[] \textbf{Initialization:} In Step \ref{step:find_center_start} each $x_j^*$ is initialized to $L_j + \Delta$.
		From \Cref{eq:l_lessthan_u} we know that $L_j + \Delta \le U_j - \Delta$.
		Moreoever, $L_j, U_j$, and $\Delta$ are all integers.
		Therefore, $x^*$ is integral and satisfies $L_j + \Delta \le x_j^* \le U_j - \Delta, \forall \jinl$.
		\item[] \textbf{Maintenance:}
		If the condition in Step \ref{step:find_center_end} fails, the value of $x^*$ is not updated. Therefore the invariant is maintained.
		If the condition succeeds we have that,
		\begin{equation}
		\sum_{j'\in[\ell]} x_{j'}^* < k\label{eq:if_cond_succeeds}
		\end{equation}
		The value $x_j^*$ is set to $\min\set{k - \sum_{j' \in [\ell] : j' \neq j} x_{j'}^*~,~~~U_{j} - \Delta}$ in Step \ref{step:find_center_end}.
		The following two cases arise based on the minimum of the two quantities.
		
		\begin{itemize}
			\item \textbf{Case 1: $ k - \sum_{j' \in [\ell] : j' \neq j} x_{j'}^*  \le U_{j}-\Delta $.}\\ 
			In this case $x^*_{j}$ is set to $k - \sum_{j' \in [\ell] : j' \neq j} x_{j'}^* \le U_j - \Delta $, which is an an integer value since both $x_{j}^*$ and $k - \sum_{j' \in [\ell] : j' \neq j} x_{j'}^*  $ are integers before the iteration.
			From (\ref{eq:if_cond_succeeds}) we have that
			\begin{equation}
			\label{eq:x_j_increases}
			k - \sum_{j' \in [\ell] : j' \neq j} x_{j'}^*  = x_{j}^* + k - \sum_{j' \in [\ell]} x_{j'}^* > x_{j}^*.
			\end{equation}
			Since $x^*_{j} \ge L_{j} + \Delta$ before the iteration, (\ref{eq:x_j_increases}) gives us that $x^*_{j}$ is greater than $L_{j} + \Delta$ even after the update.
			
			\item \textbf{Case 2: $ k - \sum_{j' \in [\ell] : j' \neq j} x_{j'}^*  > U_{j}-\Delta $.}\\
			Since $U_{j} - \Delta \ge L_{j} + \Delta$ from \Cref{eq:l_lessthan_u} and since $U_{j} - \Delta$ is an integer, the value of $x^*_{j}$ after the update is an integer such that $U_{j}-\Delta \ge x^*_{j} \ge L_{j} + \Delta$.
			
		\end{itemize}
		Therefore in both the cases the invariant is maintained.
		\item[] \textbf{Termination:} At termination $j = \ell$. The invariant gives us that $x^*$ is an integral point such that $L_j + \Delta \le x_j^* \le U_j - \Delta, \forall \jinl$. 
	\end{itemize}
	
	From \Cref{eq:suml_delta} we have that before the start of the \textbf{for} loop $\sum_{\jinl}x_j^* = \sum_{\jinl}L_j +\Delta \le k$.
	After the termination of the \textbf{for} loop we have that $x_j^* = U_j - \Delta$, forall $\jinl$, when the \textbf{if} condition in Step \ref{step:find_center_end} fails for all $\jinl$, or the \textbf{if} condition in Step \ref{step:find_center_end} succeeds for some $j$, in which case $\sum_{\jinl}x_j^*=k$, and the value of $x^*$ does not change after this iteration.
	Therefore, after the \textbf{for} loop we get $\sum_{\jinl}x_j^* = \min\set{\sum_{\jinl}U_j -\Delta, k}$.
	But \Cref{eq:sumu_delta} gives us that $\sum_{\jinl}U_j -\Delta \ge k$.
	Therefore, the \textbf{for} loop finds an integral point $x^*$ such that $L_j + \Delta \le x_j^* \le U_j - \Delta, \forall \jinl$, and $\sum_{\jinl}x_j^* = k$.

	Therefore there is an $l_1$ ball of radius $\Delta$ in $P$ centered at the integral point $x^* \in  H$ (that is, $\sum_{\jinl}x_j^* = k$).
	Consequently, there exists an $l_1$ ball of radius $\Delta$ centered at the origin in the polytope $P'$.
	Since an $l_1$ ball of radius $\Delta$ centered at origin encloses an $l_2$ ball of radius $\Delta$ centered at origin we get that an $l_2$ ball of radius $\Delta$ centered at the origin, $B(0,\Delta)$, is in the polytope $P'$.

\end{proof}

\noindent Let $C(x,\beta)\subseteq \R^{\ell}$ represent a cube of side length $\beta$ centered at $x$. For any integral point $x \in K'$ let $F_x$ represent the set of points in $\paren{1+\frac{\sqrt{\ell}}{\Delta}}K'$ that are rounded to $x$.

\begin{lemma}
	\label{lem:rounding}
	For any integral point $x \in K'$, $F_x \subseteq H' \cap C(x,2)$.
\end{lemma}
\begin{proof}
	Let $z$ be the point sampled in Step \ref{step:sample}.
	Since $z \in \paren{1+\frac{\sqrt{\ell}}{\Delta}}K'$ we have that $\sum_{\jinl}z_j = 0$.
	Therefore,
	\[
	\sum_{\jinl} \floor{z_j} \le 0 \qquad \text{and} \qquad \sum_{\jinl} \ceil{z_j} \ge 0.
	\]
	Then,
	\begin{multline*}
	m = \abs{\sum_{\jinl} \floor{z_j}}
	= \abs{\sum_{\jinl} \floor{z_j} - \sum_{\jinl} z_j}=\abs{\sum_{\jinl} \paren{\floor{z_j} - z_j}}
	\le \sum_{\jinl} \abs{\floor{z_j} - z_j}
	\le \ell,
	\end{multline*}
	where the second equality is because $ \sum_{\jinl} z_j = 0$.
	Hence, starting from $x_j = \floor{z_j}, \forall \jinl$, the algorithm has to round at most $\ell$ coordinates to $x_j = \ceil{z_j}$.
	Since $j \in [\ell]$ this is always possible.
	Therefore, the rounding in Step \ref{step:rounding} always finds an integral point $x$ that satisfies the following,
	\begin{equation}
	\sum_{\jinl} x_j = 0\qquad \text{and} \qquad \paren{\forall \jinl, x_j = \floor{z_j}~~\text{or}~~x_j = \ceil{z_j}}.\label{eq:rounded_x}
	\end{equation}
	Therefore, the set of points $z \in \paren{1+\frac{\sqrt{\ell}}{\Delta}}K'$ that are rounded to the integral point $x\in K'$ satisfying (\ref{eq:rounded_x}) is a strict subset of
	\[
	\set{z :\paren{\forall \jinl, x_j = \floor{z_j} \lor x_j = \ceil{z_j}} \land \sum_{\jinl} z_j = 0 },
	\]
	which is contained in $H' \cap C(x,2)$ since $\abs{z_j - \floor{z_j}} \le 1$ and  $\abs{ \ceil{z_j} - z_j} \le 1, \forall \jinl$.
\end{proof}

\begin{lemma} For any $x \in K'$, $ H' \cap C(x,2) \subseteq \paren{1+\frac{\sqrt{\ell}}{\Delta}}K'$.
	\label{lem:cube_contained}
\end{lemma}
\begin{proof}
	Fix a point $x \in P'$.
	Then for any $x' \in C(x,2)$, $\|x'-x\|_2 \le \sqrt{\ell}$.
	\Cref{lem:ball_contained} gives us that the translated polytope $P'$ contains a ball of radius $\Delta$ centered at the origin.
	Then the polytope $\frac{\sqrt{\ell}}{\Delta}P'$ contains a ball of radius $\sqrt{\ell}$ centered at the origin, which implies that the polytope $\frac{\sqrt{\ell}}{\Delta}P'$ contains every vector of length at most $\sqrt{\ell}$.
	Therefore, $x'-x\in \frac{\sqrt{\ell}}{\Delta}P'$.
	Now since $x\in P'$ we get that $x'\in\paren{1+\frac{\sqrt{\ell}}{\Delta}}P'$.
	Therefore, $C(x,2) \subseteq \paren{1+\frac{\sqrt{\ell}}{\Delta}}P'$.
	Consequently, $H' \cap C(x,2) \subseteq H' \cap \paren{1+\frac{\sqrt{\ell}}{\Delta}}P' = \paren{1+\frac{\sqrt{\ell}}{\Delta}}(H' \cap P')$ since $\alpha H' = H'$ for any scalar $\alpha \neq 0$.
	Hence, $H' \cap C(x,2) \subseteq \paren{1+\frac{\sqrt{\ell}}{\Delta}}K'$
\end{proof}

\begin{lemma}
	\label{lem:corollary}
	For any point $z \in \frac{1}{\paren{1+\frac{\sqrt{\ell}}{\Delta}}}K'$ the integral point it is rounded to belongs to the polytope $K'$.
\end{lemma}
\begin{proof}
	From \Cref{lem:rounding} we know that for any integral point $x \in K'$, $F_x \subseteq H' \cap C(x,2)$.
	Due to convexity of $K'$, $\frac{1}{\paren{1+\frac{\sqrt{\ell}}{\Delta}}}K'$ is contained entirely inside $K'$.
	Therefore, \Cref{lem:rounding} is true for all the points in $\frac{1}{\paren{1+\frac{\sqrt{\ell}}{\Delta}}}K'$.
	By arguing similarly as in the proof of \Cref{lem:cube_contained}, we can show that for any any $x \in \frac{1}{\paren{1+\frac{\sqrt{\ell}}{\Delta}}}K'$, $ H' \cap C(x,2) \subseteq K'$.
	This proves the lemma.
\end{proof}

Let $\mu$ be a uniform probability measure on the convex rational polytope $\paren{1+\frac{\sqrt{\ell}}{\Delta}}K'$.
\begin{lemma}
	\label{lem:same_size}
	Fix any two distinct integral points $x,x' \in K'$, $\mu\paren{F_x} = \mu\paren{F_{x'}}$.
\end{lemma}
\begin{proof}
	Given two distinct integral points $x,x' \in K'$,
	let $c = x'-x$. 
	Clearly $c$ is an integral point and $\sum_{\jinl} c_j = \sum_{\jinl} x'_j - \sum_{\jinl} x_j = 0-0 = 0$.
	Let $z\in F_x$ and $z' = z+c$.
	Then 
	\begin{multline*}
	\abs{\sum_{\jinl}\floor{z_j'}} = \abs{\sum_{\jinl}\floor{z_j}+c_j}= \abs{\sum_{\jinl}\floor{z_j}+\sum_{\jinl}c_j}
	=\abs{\sum_{\jinl}\floor{z_j}} = m.
	\end{multline*}
	Therefore, for both $z$ and $z'$ the first $m$ coordinates are rounded up, and the remaining are rounded down, in Step \ref{step:rounding}.
	Since $\floor{z_j'} = \floor{z_j}+c_j$ and $\ceil{z_j'} = \ceil{z_j}+c_j$, the point $z'$ is rounded to is nothing but $x'$.
	Therefore, for every point $z\in F_x$ there is a unique point $z' \in F_{x'}$ such that they are rounded to $x$ and $x'$ respectively. 
	This gives us a bijection between the sets $F_x$ and $F_{x'}$.
	Therefore, $\mu\paren{F_x} = \mu\paren{F_{x'}}$.
\end{proof}

\begin{proof}[Proof of \Cref{thm:main}]
	Let $K'' = \paren{1+\frac{\sqrt{\ell}}{\Delta}}K'$.
	
	Let $\nu$ be the distribution from which the \unif~samples a point from $K''$.
	That is for a given $\delta > 0$,
	\begin{equation}
	\sup_{A \subseteq K''}\abs{\nu(A) - \mu(A)} \le \delta.\label{eq:tv}
	\end{equation}
	\paragraph{Close to uniform sample.}
	From \Cref{lem:rounding,lem:cube_contained} we know that for any point $x\in K'$, $F_x \subseteq K''$.
	Therefore, from (\ref{eq:tv}) we have that
	\begin{align*}
	\abs{\nu(F_x) - \mu(F_x)} \le \delta. 
	\end{align*}
	
	Let $\nu'$ be the density from which \Cref{alg:random_walk} samples an integral point from $K$.
	We know that $\forall x, ~x \in K'\iff x+x^*\in  K$. 
	Since $x^*$ is an integral point, for any integral point $x$ we also have that $x \in K'\iff x + x^* \in K$.
	This gives us a bijection between the integral points in $K'$ and $K$.
	
	Therefore for any integral point $x \in K'$, $\nu'(x) = \nu(F_x)$.
	Let $\mu'(x) := \mu(F_x)$ for any integral point $x$ in $K'$.
	
	For any two integral points $x, x'\in K'$, \Cref{lem:same_size} gives us that $\mu\paren{F_x} = \mu\paren{F_{x'}}$.
	Moreover, since $F_x$ and $F_{x'}$ are both fully contained in $K''$, we get that $\mu'(F_x) = \mu(F_x) = \mu(F_{x'}) = \mu'(F_{x'})$.
	Therefore, $\mu'$ is a uniform measure over all the integral points in $K'$.
	
	Moreover, for any subset of integral points in $K$, say $I$, we have that 
	\[
	\nu'(I)=\nu\paren{\cup_{x \in I}F_x}.
	\]
	From \Cref{eq:tv} we have that
	\[
	\abs{\nu\paren{\cup_{x \in I}F_x}-\mu\paren{\cup_{x \in I}F_x}} \le \delta.
	\]
	Consequently,
	\[
	\abs{\nu'(I) - \mu'(I)} \le \delta.
	\]
	Therefore $\nu'$ over the integral points in $K'$ is at a total variation distance of at most $\delta$ from the uniform probability measure $\mu'$ over the integral points in $K'$.
	
	\paragraph{Probability of acceptance.}\Cref{alg:random_walk} samples points from $\paren{1+\frac{\sqrt{\ell}}{\Delta}}K'$ in each iteration.
	Due to \Cref{lem:corollary} we know for sure that whenever the algorithm samples a point from $\frac{1}{\paren{1+\frac{\sqrt{\ell}}{\Delta}}}K'$, it will be rounded to an integral point in $K'$.
	Therefore, the probability of sampling an integral point in $K'$ is
	\begin{align*}
	&\ge\nu\paren{ \frac{1}{\paren{1+\frac{\sqrt{\ell}}{\Delta}}}K'}\\
	&\ge \mu\paren{ \frac{1}{\paren{1+\frac{\sqrt{\ell}}{\Delta}}}K'}-\delta 
 &\text{from}~\Cref{eq:tv}\\
	&=\frac{\vol{\ell-1}{ \frac{1}{\paren{1+\frac{\sqrt{\ell}}{\Delta}}}K'}}{\vol{\ell-1}{K''}}-\delta 
 &\because \mu~\text{is a uniform distribution over}~K''\\
	&=\frac{\vol{\ell-1}{ \frac{1}{\paren{1+\frac{\sqrt{\ell}}{\Delta}}}K'}}{\vol{\ell-1}{ \paren{1+\frac{\sqrt{\ell}}{\Delta}}K'}}-\delta  &\text{by the definition of }~K''\\
	&= \paren{1+\frac{\sqrt{\ell}}{\Delta}}^{-2(\ell-1)}-\delta &\text{{\sf Vol}}_{\ell-1}~\text{is volume in}~\ell-1~\text{dimensions}\\
	&\ge e^{-\frac{\sqrt{\ell}}{\Delta}2(\ell-1)} - \delta &\text{using}~(1+x) \le e^x, \forall x
	\end{align*}
	Therefore the expected running time of the algorithm before it outputs an acceptable point is inversely proportional to the probability of acceptance, that is, $1/\paren{e^{-2\frac{\ell\sqrt{\ell}}{\Delta}} - \delta}$.
	When $\delta < e^{-2}$ and is a non-negative constant, and if $\Delta = \Omega\paren{\ell^{1.5}}$, this probability is at least a constant.
	Hence, repeating the whole process a polynomial number of times in expectation guarantees we sample an integral point from $K'$.
\end{proof}
\begin{remark}
The polytope $\paren{1+\frac{\sqrt{\ell}}{\Delta}}K'$ is an $\ell-1$ dimensional polytope given to us by the $\cH$ description in $\ell$ dimensions.
The random walk-based algorithms used as \unif~in Step \ref{step:sample} require the polytope they sample from to be full-dimensional.
Below we describe a rotation operation such that the rotated polytope, that is, the polytope formed after applying the rotation on $\paren{1+\frac{\sqrt{\ell}}{\Delta}}K'$, is full-dimensional in $\ell-1$ dimensions.
This is a well-known transformation used as a pre-processing step to make a polytope full-dimensional.
Let $u_1, u_2, \ldots, u_{\ell}$ be orthonormal basis of $\R^{\ell}$ such that $u_{\ell} := (1,1,\ldots,1)^T$.
We now construct a matrix $R$ such that $Ru_j = e_j, \forall \jinl$, where $e_1, e_2, \ldots, e_{\ell}$ are the standard basis vectors in $\ell$ dimensions.
Fix $\jinl$. 
We know that $e_j = \sum_{j'\in [\ell]} \alpha^{(j)}_{j'} u_{j'}$, where $\forall j' \in [\ell], \alpha^{(j)}_{j'} = e_{j}^Tu_{j'}$.
Thus, we get that $Re_j = \sum_{j'\in [\ell]} \alpha^{(j)}_{j'} Ru_{j'} = \sum_{j'\in [\ell]} \alpha^{(j)}_{j'} e_{j'}$, which implies that the vector $\paren{\alpha_1^{(j)}, \alpha_2^{(j)}, \ldots, \alpha_{\ell}^{(j)}}^T$ forms the $j$th column of $R$.
It is also easy to verify that $R$ is orthogonal.
Therefore, the rotation matrix $R$ can be computed efficiently.
This rotation maps the hyperplane $\sum_{\jinl}x_j = 0$ into the $\ell-1$ dimensional space spanned by $e_1, e_2, \ldots, e_{\ell-1}$.
Therefore, the rotated polytope is an $\ell-1$ dimensional polytope in $\ell-1$ dimensions.
For any point $(x_1, x_2, \ldots, x_{\ell-1}) \in \R^{\ell-1}$ we check the membership of $R^{-1}(x_1, x_2, \ldots, x_{\ell-1},0)$ in $\paren{1+\frac{\sqrt{\ell}}{\Delta}}K'$.
This gives us the membership oracle for the rotated polytope. 
We can then sample a rational point from this rotated polytope in Step \ref{step:sample}, apply $R^{-1}$ on the point sampled to get a point in $\paren{1+\frac{\sqrt{\ell}}{\Delta}}K'$, and proceed with our algorithm.  
\end{remark}
    \subsection{Exact Uniform Sampling for small $\ell$}
\label{subsec:alg2}
In this section, we give another exact sampling algorithm to sample a uniform random group-fair representation for small $\ell$.
In \Cref{eq:polytopeK} the convex polytope $K$ is described using an $\cH$ description defined as follows,
\begin{definition}[\textbf{$\cH$-description of a polytope}]
A representation of the polytope
as the set of solutions of finitely many linear inequalities.
\end{definition}
We can also have a representation of the polytope as described by its vertices, defined as follows,
\begin{definition}[\textbf{$\cV$-description of a polytope}]
The representation of the polytope by the set of its vertices.
\end{definition}

\cite{Barvinok} gave an algorithm to count exactly, the number of integral points in $K$, as re-stated below.

\begin{theorem}[Theorem 7.3.3 in \cite{Barvinok}]
\label{thm:barvinok}
Let us fix the dimension ${\ell}$. Then there exists a polynomial time algorithm that, for any given rational $\cV$-polytope $P \subset \R^{\ell}$, computes the number $|P \cap \Z^{\ell}|$. The complexity of the algorithm in terms of the dimension ${\ell}$ is ${\ell}^{\cO({\ell})}$.
\end{theorem}

We also have the algorithm by \cite{Pak2000OnSI} gives us an exact uniform random sampler for the integral points in $K$.

\begin{theorem}[Theorem 1 in \cite{Pak2000OnSI}]
\label{thm:pak}
Let $P \subset \R^{\ell}$ be a rational polytope, and let $B = P \cap \Z^{\ell}$. Assume
an oracle can compute $|B|$ for any $P$ as above. Then there exists a polynomial-time algorithm for sampling uniformly from $B$, which calls this oracle $\cO\paren{{\ell}^2L^2}$ times where $L$ is the bit complexity of the input.
\end{theorem}

Using the counting algorithm given by \Cref{thm:barvinok} as the counting oracle in \Cref{thm:pak} gives us our second algorithm that samples a uniform random group representation exactly.
\begin{theorem}
\label{thm:alg1}
For given fairness parameters $L_j, U_j \in \N$ and an integer $k > 0$, there is an algorithm that samples an exact uniform random integral point in $K$ and runs in time $\ell^{\cO(\ell)}\cO\paren{\log^2k}$.
\end{theorem}
\begin{proof}
The proof essentially follows from the proof of Theorem 1 in \cite{Pak2000OnSI}.
They assume access to an oracle that counts the number of integral points in any convex polytope that their algorithm constructs.
We show that Barvinok's algorithm can be used as an oracle for all the polytopes that are constructed in the algorithm to sample a uniform random integral point from our convex rational polytope $K$.

The algorithm in \Cref{thm:pak} intersects the polytope by an axis-aligned hyperplane and recurses on one of the smaller polytopes (to be specified below).
In the deepest level of recursion where the polytope in that level contains only one integral point, the algorithm terminates the halving process and outputs that point.
The proof of their theorem shows that this gives us a uniform random integral point from the polytope we started with.

Let us consider the dimension $1$ w.l.o.g.
The algorithm finds a value $c$ such that $L_1 < c < U_1$, $|H_+ \cap B|/|B| \le 1/2$, $|H_- \cap B|/|B| \le 1/2$, where $H_+$ and $H_-$ are two halves of the space separated by the hyperplane $H$ defined by $x_1 = c$.
That is, $H_+$ is the halfspace $x_1 \ge c$ and $H_-$ is the halfspace $x_1 \le c$.
Therefore, there are three possible polytopes for the algorithm to recurse on, $H_+ \cap B, H_- \cap B$, and $H \cap B$.
Here $|H \cap B|/|B|$ can be $\ge 1/2$.
Let 
\[
f_{+} = \frac{|H_+ \cap B|}{|B|},~f_{-} = \frac{|H_- \cap B|}{|B|},~\text{and}~f = \frac{|H \cap B|}{|B|}.
\]
Then the algorithm recurses on the polytope $H_+ \cap B$ with probability $f_+$, on $H_- \cap B$ with probability $f_-$, and on $H \cap B$ with probability $f$.

Observe that $K$ is also defined by the axis-aligned hyperplanes, $x_1=L_1$ and $x_1 = U_1$, amongst others.
Therefore, $x_1 \ge L_1$ will become a redundant constraint if the algorithm recurses on $H_+\cap B$.
Else $x_1 \le U_1$ will become a redundant constraint if the algorithm recurses on $H_- \cap B$.
In both these cases, the number of integral points reduces by more than $1/2$.
If the algorithm recurses on $H\cap B$, it fixes the value of $x_1$ to $c$, and the dimension of the problem reduces by $1$.
Since the number of integral points is $\exp\paren{dL}$, the number of halving steps performed by the algorithm is at most $\cO\paren{dL}$.

Observe that in all levels of recursion, the polytopes constructed are of $d$ dimensions ($1 \le d \le \ell$)
and are of the following form,
\begin{align*}
    \Big\{\paren{x_1, x_2, \ldots, x_{d}}\in\R^{d}~~\Big|~~\sum_{j \in [d]} x_j = k'~~\text{and}~~c'_j \le x_j \le c''_j, \forall j \in [\ell] \Big\},\label{eq:polytope}
\end{align*}
where $k', c', c''$ are some constants.
This gives us the $\cH$-description of each of the polytopes the algorithm constructs in each level of recursion.
The vertices of such a polytope are formed by the intersection of $d$ hyperplanes.
Therefore, there could be at most ${2d+2\choose d} = 2^{\cO\paren{{d}}}$ number of vertices for such a polytope in $d$ dimensions, which gives us that the $\cV$-description can be computed from the $\cH$-description in time $2^{\cO\paren{{d}}}$.
Therefore, for all these intermediate polytopes, we can use the counting algorithm given by \Cref{thm:barvinok} whose run time depends on $d^{(\cO(d)}$.
Using $ d \le \ell$, we get that the counting algorithm given by \Cref{thm:barvinok} takes time $\ell^{(\cO(\ell)}$ for all the polytopes constructed by the algorithm.
Further, the algorithm makes at most $\cO\paren{dL}$ calls to this counting algorithm in each recursion step.

Since the input to the algorithm consists of a number $k > 0$ and fairness parameters $0 \le L_j \le U_j \le k, \forall \jinl$, where all these parameters are integers, the bit complexity of the input is $\log k$.

Therefore, the total running time of our algorithm is $\cO\paren{d^2L^2}\ell^{\cO(\ell)} = \cO\paren{\ell^2\log^2 k}\ell^{\cO(\ell)} = \cO\paren{\log^2 k}\ell^{\cO(\ell)}$.
\end{proof}

    \subsection{Prefix fairness constraints}
    \label{subsec:prefix}
    Prefix fairness constraints are represented by the numbers $L_{ij}$ and $U_{ij}$, for all $j \in [\ell]$ and $i \in M$, where $M \subseteq [k]$, which give lower and upper bounds on the representation of group $j$ in the top $i$ ranks, i.e., the top $i$ prefix of the top $k$ ranking. 
    When $M = \set{k}$, this gives us the setup we started with.
    This model has been first studied by \cite{CSV2018,YS2017}, who give a deterministic algorithm to get a utility-maximizing ranking while satisfying prefix fairness constraints. 
    To overcome the limitations of a deterministic ranking, we propose to use our algorithm\footnote{Both \Cref{alg:dp} and \Cref{alg:random_walk} can be used for the heuristic. We use \Cref{alg:random_walk} as it is faster in practice.} as a heuristic to output randomized rankings under prefix fairness constraints. It inductively finds a group-fair assignment in \textit{blocks} between two ranks (ranks $i+1$ to $i'$ such that $i,i' \in M$ and $i'' \not\in M, \forall i'' \in [i+1, i'-1]$), as follows:
    \begin{enumerate}[leftmargin=*]
        \item Let us assume we have a random sample of the top $i$ ranking for some $i \in M$. Let $w_{ij}$ be the counts of groups $j \in [\ell]$ in this top-$i$ ranking. \item Let $i' \in M$ be the smallest rank larger than $i$ in the set $M$. Use \Cref{alg:random_walk} to find an integer solution $x^{(i')}_j$ in $K = \{x\in\R^{\ell}|\sum_{j \in [\ell]} x_{j} = i'-i+1,~\max\{0,L_{i'j}-w_{ij}\} \le x_j \le \min\{i'-i+1,U_{i'j} - w_{ij}\}, \forall j \in [\ell] \}$ to get a group-fair representation for ranks $i+1$ to $i'$. 
        \item Find a uniform random permutation of $x^{(i')}_j$ similar to Step 2 in \Cref{thm:uniqueness} to get a group-fair assignment for ranks $i+1$ to $i'$, and go to Step 1 with $i = i'$. 
    \end{enumerate}
    We call this algorithm `Prefix Random Walk' in our experiments.

    In the next section, we validate the theoretical guarantees and efficiency of our algorithms on real-world data sets.
	
	\section{Experimental Results}
	\label{sec:experiments}
	In this section, we validate our algorithms on various real-world datasets.
	We implement \Cref{alg:random_walk} (called `Random walk' in plots) using the tool called \textit{PolytopeSampler}\footnote{\href{https://github.com/ConstrainedSampler/PolytopeSamplerMatlab}{github.com/ConstrainedSampler/PolytopeSamplerMatlab} (License: GNU GPL v3.0)} to sample a point from a distribution close to uniform, on the convex rational polytope $K$.
	This tool implements a constrained Riemannian Hamiltonian Monte Carlo for sampling from high dimensional distributions on polytopes \cite{kook2022sampling}.
	
	\begin{figure*}[t]
		\centering
		\begin{subfigure}[b]{\linewidth}
			\centering
			\includegraphics[scale=0.12]{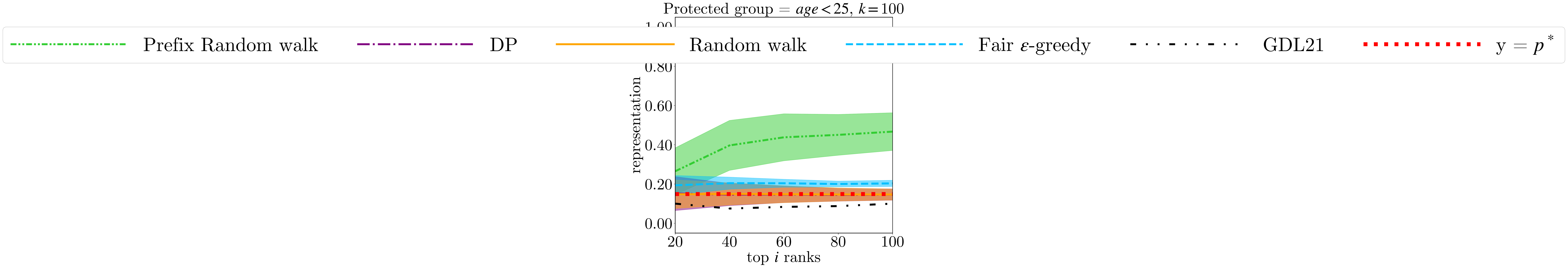} 
		\end{subfigure}
		
		\begin{subfigure}[b]{0.25\linewidth}
			\centering
			\includegraphics[scale=0.13]{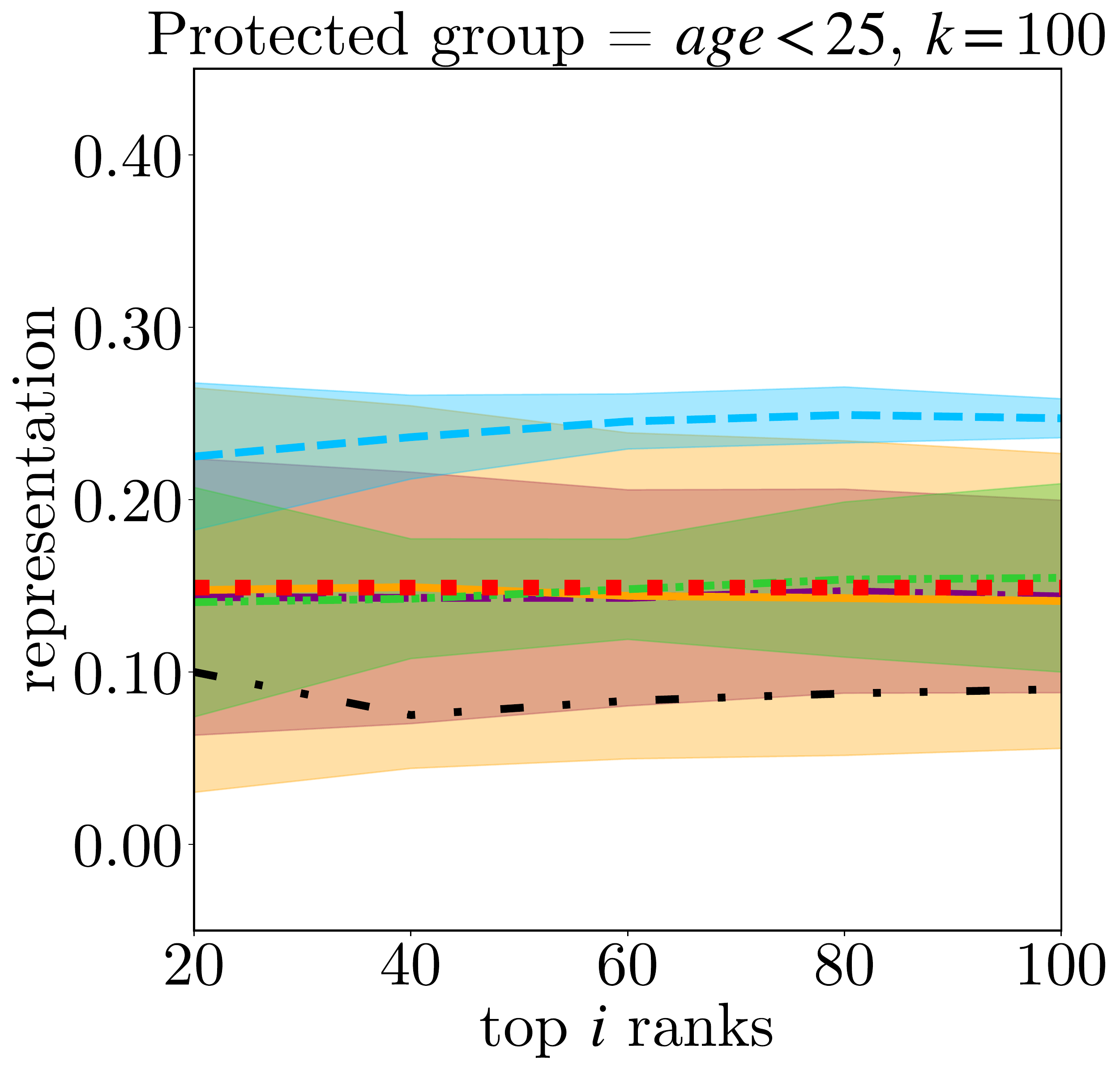} 
			\label{fig:german_25_rep}
		\end{subfigure}
		\begin{subfigure}[b]{0.48\linewidth}
			\centering
			\includegraphics[scale=0.13]{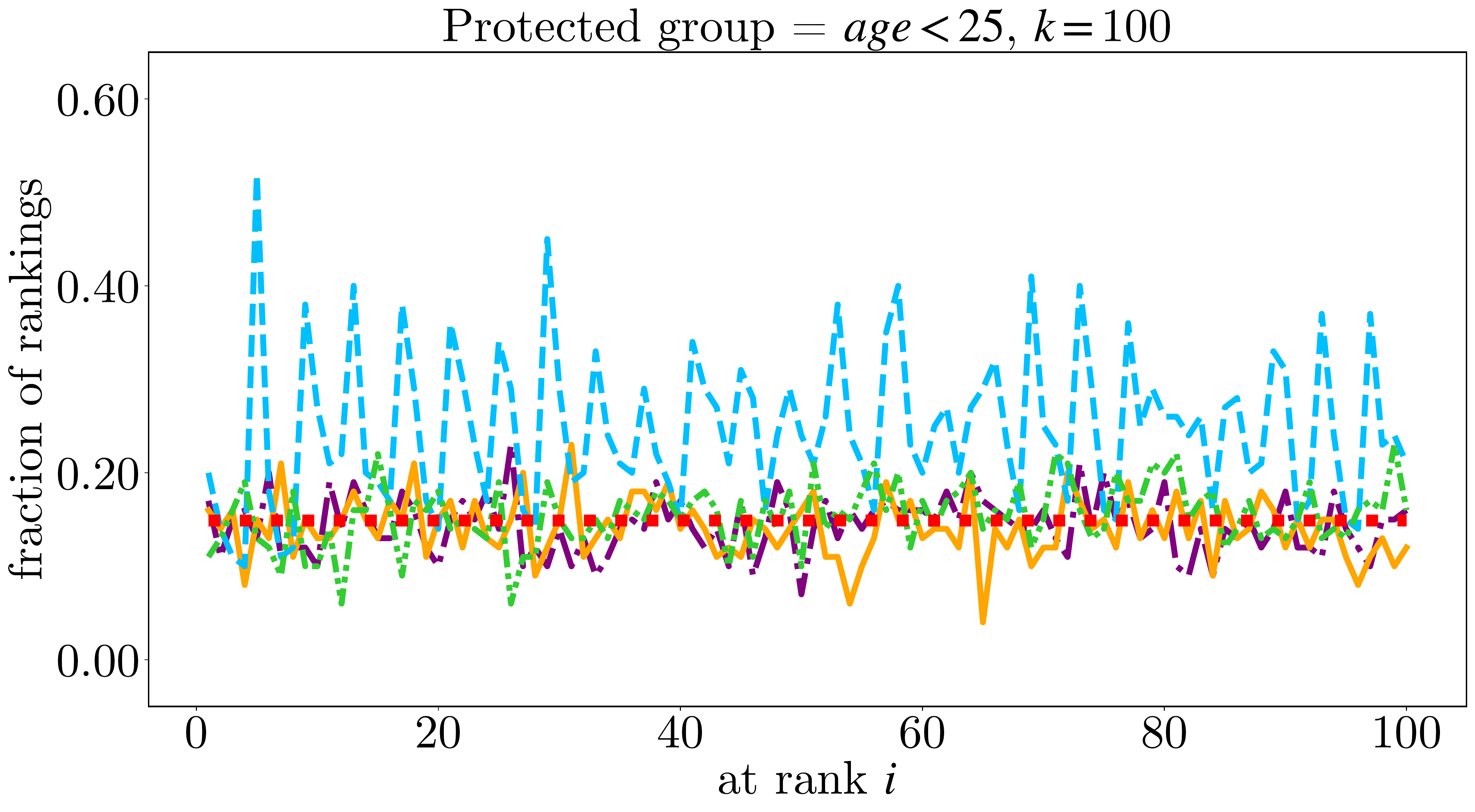} 
			\label{fig:german_35_rep}
		\end{subfigure}
		\begin{subfigure}[b]{0.25\linewidth}
			\centering
			\includegraphics[scale=0.125]{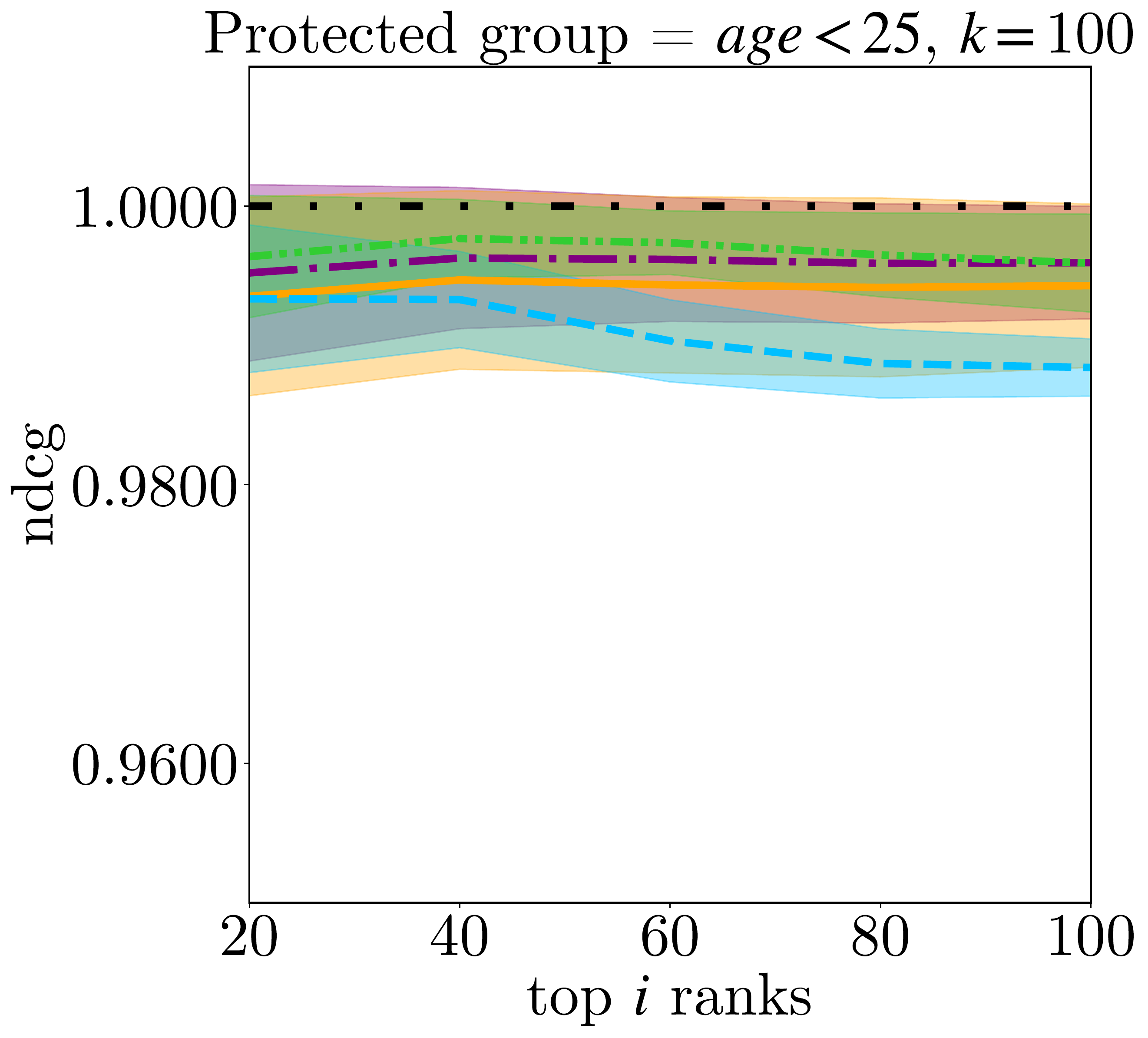} 
			\label{fig:german_35_rep}
		\end{subfigure}

		\caption{Results on the German Credit Risk dataset with \textit{age} $< 25$ as the protected group. For Fair $\epsilon$-greedy we use $\epsilon =0.3$ (see \Cref{fig:german_binary_eps015,fig:german_binary_eps05} for other values of $\epsilon$). In the plots, the means of the DP and the random walk algorithms are almost coinciding.}
		\label{fig:german_binary}
	\end{figure*}
	
	\begin{figure*}[t]
		\centering
		\begin{subfigure}[b]{\linewidth}
			\centering
			\includegraphics[scale=0.12]{legend.pdf} 
		\end{subfigure}
		
		\begin{subfigure}[b]{0.25\linewidth}
			\centering
			\includegraphics[scale=0.13]{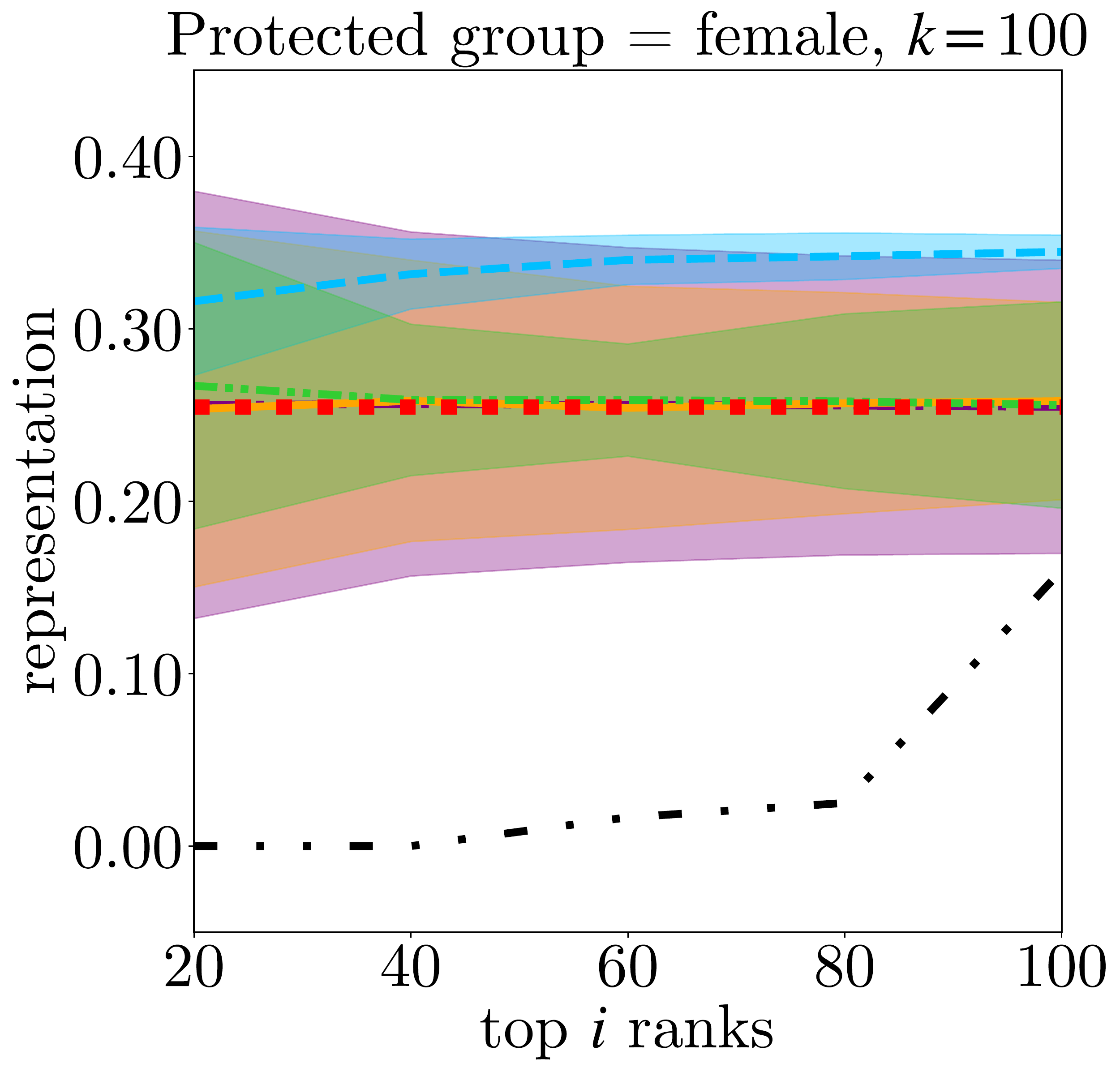} 
			\label{fig:german_25_rep}
		\end{subfigure}
		\begin{subfigure}[b]{0.48\linewidth}
			\centering
			\includegraphics[scale=0.13]{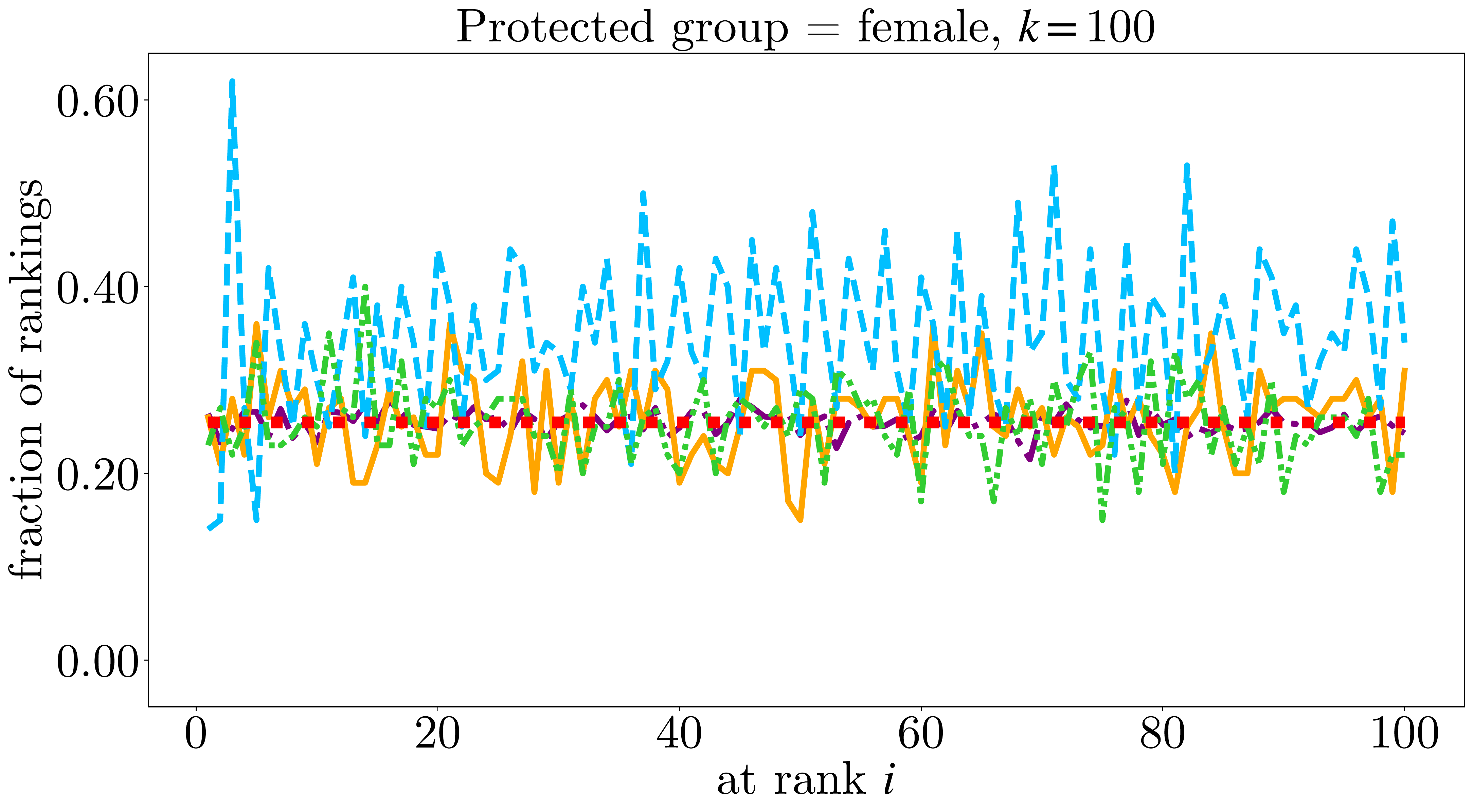} 
			\label{fig:german_35_rep}
		\end{subfigure}
		\begin{subfigure}[b]{0.25\linewidth}
			\centering
			\includegraphics[scale=0.13]{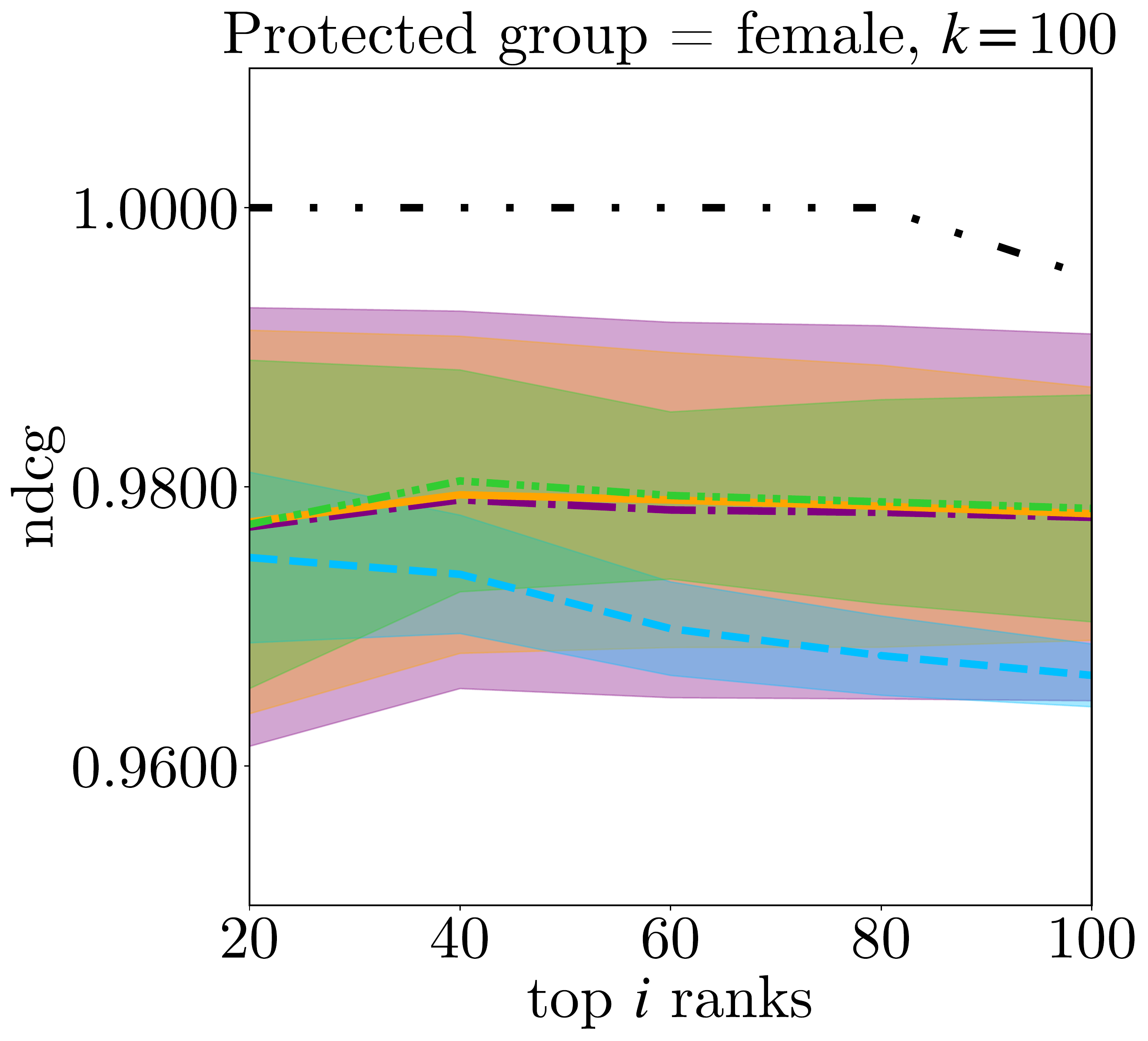} 
			\label{fig:german_35_rep}
		\end{subfigure}
		
		\caption{Results on the JEE 2009 dataset with \textit{gender} as the protected group. For Fair $\epsilon$-greedy we use $\epsilon =0.3$ (see \Cref{fig:jee_gender_eps015,fig:jee_gender_eps05} for other values of $\epsilon$). In the plots, the means of the DP and the random walk algorithms are almost coinciding.}
		\label{fig:jee_gender}
	\end{figure*}
	\paragraph{Datasets.}
	We evaluate our results on the German Credit Risk dataset\footnote{taken from \href{https://github.com/sruthigorantla/Underranking_and_group_fairness/tree/master/data/GermanCredit}{Gorantla et al. repository}} comprising credit risk scoring of $1000$ adult German residents \cite{Dua2019}, along with their demographic information (e.g., gender, age, etc.).
	We use the Schufa scores of these individuals to get the in-group rankings.
	We evaluate our algorithm on the grouping based on $age < 25$ (see \Cref{fig:german_binary}) similar to \cite{ZBCHMB2017,GDL2021,Castillo2019}, who observed that Schufa scores are biased against the adults of $age < 25$. Their representation in the top $100$ ranks is $10\%$ even though their true representation in the whole dataset is $15\%$.
	
	We also evaluate our algorithm on the IIT-JEE 2009 dataset, also used in \citet{CMV2020}.
	The dataset\footnote{taken from  \href{https://github.com/AnayMehrotra/Ranking-with-Implicit-Bias}{Celis et al. 2020b repository}} consists of the student test scores of the joint entrance examination (JEE) conducted for undergrad admissions at the Indian Institutes of Technology (IITs).
    Information about the students includes gender details (25\% women and 75\% men)\footnote{Only binary gender information was annotated in the dataset.}. 
	Students' test scores give score-based in-group rankings.
	We evaluate our algorithm with \textit{female} as the protected group, as they are consistently underrepresented ($0.04\%$ in top $100$ \cite{CMV2020}), in a score-based ranking on the entire dataset, despite $25\%$ female representation in the dataset. 
	
	\paragraph{Baselines.} $(i)$ We compare our experimental results with \textit{fair $\epsilon$-greedy} \cite{GS2020}, which is a greedy algorithm with $\epsilon$ as a parameter (explained in detail in \Cref{subsec:observations}).
	To the best of our knowledge, this algorithm is the closest state-of-the-art baseline to our setting, as it does not rely on comparing the scores of two candidates from different groups. 
	$(ii)$ We also compare our results with a recent deterministic re-ranking algorithm (GDL21) given by \citet{GDL2021}, which achieves the best balance of both group fairness and underranking of individual items compared to their original ranks in top-$k$.
	
	\paragraph{Plots.} We plot our results for the protected groups in each dataset (see \Cref{fig:german_binary,fig:jee_gender}).
	We use the representation constraints $L_j = \ceil{(p_j^*-\eta)k}$ and $U_j = \floor{(p_j^*+\eta)k}$ for group $j$ where $p_j^*$ is the total fraction of items from group $j$ in the dataset and $\eta = 0.1$. For ``prefix random walk" we put constraints at every $b$ ranks, i.e., $L_{ij} = \ceil{\paren{p_j^*-\frac{\eta b}{\max\{b,k-i\}}}k}$ and $U_{ij} = \floor{\paren{p_j^*+\frac{\eta b}{\max\{b,k-i\}}}k}$ with $i \in \set{b, 2b, \ldots}$. In the experiments, we use $k = 100$ and $b = 50$.
    With these, the representation constraints are stronger in the top $50$ ranks than in the top $100$ ranks.
    The ``representation" (on the y-axis) plot shows the fraction of ranks assigned to the protected group in the top $i$ ranks (on the x-axis).
	For randomized algorithms, we sample $1000$ rankings and output the mean and the standard deviation.
	The dashed red line is the true representation of the protected group in the dataset, which we call $p^*$, dropping the subscript.
	The ``fraction of rankings" (on the y-axis) plot for randomized ranking algorithms represents the fraction of $1000$ rankings that assign rank $i$ (on the x-axis) to the protected group.
	For completeness, we plot the results for the ranking utility metric, normalized discounted cumulative gain, defined as 
	$
	\text{nDCG}@i = \paren{\sum_{i'\in[i]} \frac{(2^{\hat{s}_{i'}}-1)}{\log_2(i'+1)}}\big/\paren{\sum_{i'\in[i]} \frac{(2^{{s}_{i'}}-1)}{\log_2(i'+1)}},
	$
	where $\hat{s}_{i'}$ and $s_{i'}$ are the scores of the items assigned to rank $i'$ in the group-fair ranking and the score-based ranking, respectively.

	\subsection{Observations} 
	\label{subsec:observations}
        The rankings sampled by our algorithms have the following property: for any rank $i$, rank $i$ is assigned to the protected group in a sufficient fraction of rankings (see plots with ``fraction of rankings'' on the y-axis).
	This experimentally validates our \Cref{thm:rep_at_i}.
	Moreover, this fraction is stable across the ranks.
	Hence the line looks almost flat.
	Whereas fair $\epsilon$-greedy fluctuates a lot, which can be explained as follows. For each rank $k' = 1$ to $k$, with $\epsilon$ probability, it assigns a group uniformly at random, and with $1-\epsilon$ probability, it assigns group $G_1 :=$ (age$\ge 25$) if the number of ranks assigned to $G_1$ is less than $(\frac{L_1 k'}{k})$ in the top $k'$ ranks, and to $G_2 :=$ (age $<25$) otherwise.
	Consider \Cref{fig:german_binary} top row where $L_1 = 80, L_2 = 10$, and $k=100$, and the plot on the right shows the fraction of rankings (y-axis) assigning rank $i$ (x-axis) to $G_1$.
	Note that if $\epsilon = 0$ (no randomization), this algorithm gives a deterministic ranking where the first four ranks are assigned to $G_2$ and the fifth to $G_1$, and this pattern repeats after every five ranks. Hence, there would be a peak in the plot at ranks $k'=5,10,15,20,\ldots$.
	Now, when $\epsilon=0.3$, fair-$\epsilon$-greedy introduces  randomness in group assignment at each rank and, as a result, smoothens out the peaks as $k'$ increases, which is exactly what is observed.
	Therefore, the first four ranks will have very low representation, even in expectation.
	Similarly the ranks $6$ to $9$.
	Clearly, fair-$\epsilon$-greedy does not satisfy fairness for any $k' < k$ consecutive ranks.
	But our algorithm satisfies this property, as is also confirmed by \Cref{cor:rep_at_prefix}.
	
	Our algorithms satisfy representation constraints for the protected group in the top $k'$ ranks for $k' = 20,40,60,80,100$, in expectation (see plots with ``representation'' on the y-axis).
	Fair $\epsilon$-greedy overcompensates for representing the protected group.
	The deterministic algorithm GDL21 achieves very high nDCG but very low representation for smaller values of $k'$, although all run with similar representation constraints.
	This is because the deterministic algorithm uses comparisons based on the scores, hence putting most of the protected group items in higher ranks (towards $k$).
	With a larger value of $\epsilon$, fair $\epsilon$-greedy gets a much higher ``representation'' of the protected group than necessary (see \Cref{fig:german_binary_eps05,fig:jee_gender_eps05}), whereas, with a smaller value of $\epsilon$, it fluctuates a lot in the ``fraction of rankings'' (see \Cref{fig:german_binary_eps015,fig:jee_gender_eps015}).
    Our ``Prefix Random Walk" algorithm is run with stronger fairness requirements than ``Random Walk" and ``DP" in the top $50$ ranks, which can be observed by its smaller deviation from the line $y = p^*$ in the left-most plots. 
	
 In conclusion, our experimental results validate that our algorithms in fact provide sufficient opportunities for all the groups in each rank without losing out much on utility. Moreover, our algorithms, especially the random walk, run very fast even for a large number of groups (\Cref{fig:running_time}). 
	We also run experiments on the JEE 2009 dataset with birth category defining $5$ groups (see \Cref{fig:jee_category}).
	The experiments were run on a Quad-Core Intel Core i5 processor consisting of 4 cores, with a clock speed of 2.3 GHz and DRAM of 8GB. 
    Implementation of our algorithms and the baselines has been made available for reproducibility\footnote{ \href{https://github.com/sruthigorantla/SamplingExPostGroupFairRankings}{github.com/sruthigorantla/SamplingExPostGroupFairRankings}.}.

	\begin{figure}[t]
		\centering
		\begin{subfigure}[b]{0.48\linewidth}
			\centering
			\includegraphics[scale=0.20]{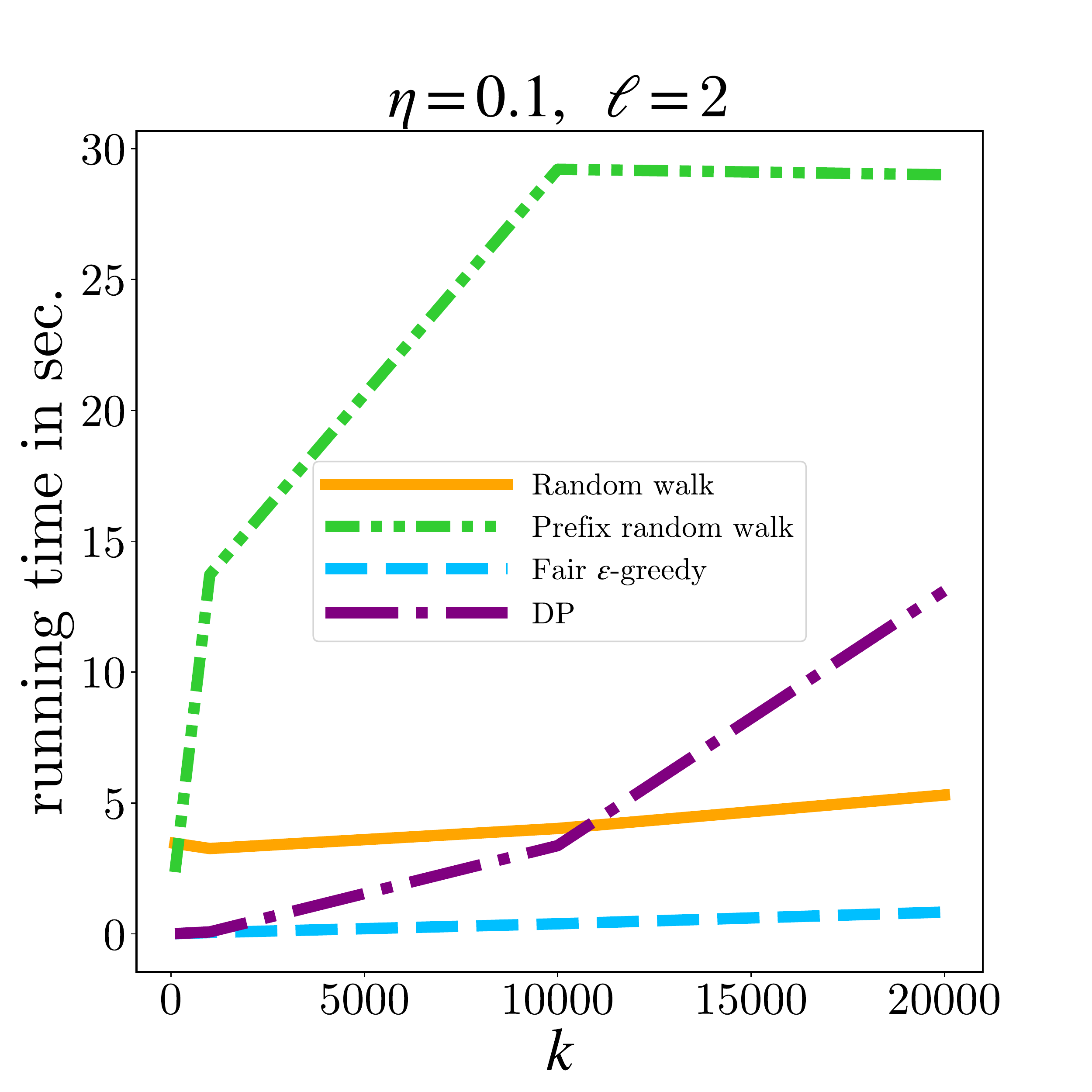} 
			\label{fig:time_k2}
		\end{subfigure}
		\begin{subfigure}[b]{0.48\linewidth}
			\centering
			\includegraphics[scale=0.20]{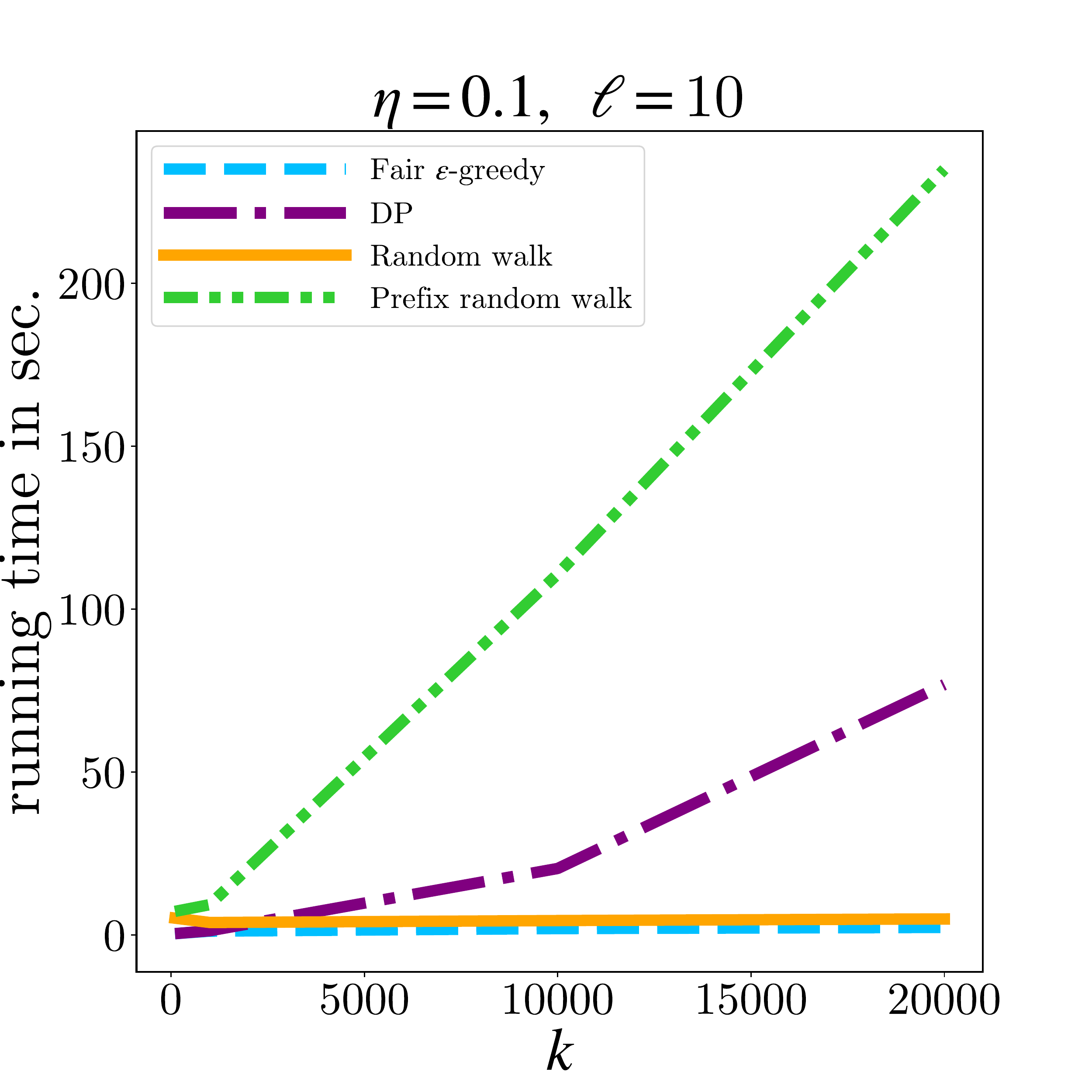} 
			\label{fig:time_k10}
		\end{subfigure}
		
		\caption{Average running time in seconds of the algorithms, over $5$ runs, to sample a ranking. For prefix random walk, we add prefix constraints with $b = 50,200,400,400$ for $k = 100, 1000,10000,20000$ respectively.}
		\label{fig:running_time}
	\end{figure}

	\section{Conclusion}
	\label{sec:conclusion}
	We take an axiomatic approach to define randomized group-fair rankings and show that it leads to a unique distribution over all feasible rankings that satisfy lower and upper bounds on the group-wise representation in the top ranks. We propose practical and efficient algorithms to exactly and approximately sample a random group-fair ranking from this distribution. Our approach requires merging a given set of ranked lists, one for each group, and can help circumvent implicit bias or incomplete comparison data across groups.
	
	The natural open problem is to extend our method to work even for noisy, uncertain inputs about rankings within each group. Even though our heuristic algorithm does output ex-post group-fair rankings under prefix constraints, it is important to investigate the possibility of  polynomial-time algorithms to sample from the distribution that satisfies natural extensions of our axioms for prefix group-fair rankings.

\section*{Ethical Statement}
A limitation of our work as a post-processing method is that it cannot fix all sources of bias, e.g., bias in data collection and labeling. Randomized rankings can be risky and opaque in high-risk, one-time ranking applications. Our guarantees for group fairness may not necessarily reflect the right fairness metrics for all downstream applications for reasons including biased, noisy, incomplete data and legal or ethical considerations in quantifying the eventual adverse impact on individuals and groups.

\section*{Acknowledgements}
SG was supported by a Google PhD Fellowship. The proof of Theorem 4.1 was suggested to us by Santosh Vempala. The authors would like to thank him for allowing us to include it in our paper. AL was supported in part by the SERB Award ECR/2017/003296 and a Pratiksha Trust Young Investigator Award. AL is also grateful to Microsoft Research for supporting this collaboration.

\bibliographystyle{named}
\bibliography{references}

\begin{thebibliography}{}

\bibitem[\protect\citeauthoryear{AG}{2022}]{schufa}
Schufa~Holding AG.
\newblock Schufa.
\newblock 2022.

\bibitem[\protect\citeauthoryear{Barvinok}{2017}]{Barvinok}
Barvinok.
\newblock Handbook of discrete and computational geometry (3rd ed.).
\newblock 2017.

\bibitem[\protect\citeauthoryear{Baswana \bgroup \em et al.\egroup
  }{2019}]{BCCKP2019}
Surender Baswana, Partha~Pratim Chakrabarti, Sharat Chandran, Yashodhan
  Kanoria, and Utkarsh Patange.
\newblock Centralized admissions for engineering colleges in india.
\newblock {\em Interfaces}, 49(5):338--354, 2019.

\bibitem[\protect\citeauthoryear{Beutel \bgroup \em et al.\egroup
  }{2019}]{BCDQ2019}
Alex Beutel, Jilin Chen, Tulsee Doshi, Hai Qian, Li~Wei, Yi~Wu, Lukasz Heldt,
  Zhe Zhao, Lichan Hong, Ed~H. Chi, and Cristos Goodrow.
\newblock Fairness in recommendation ranking through pairwise comparisons.
\newblock In {\em Proceedings of the 25th ACM SIGKDD International Conference
  on Knowledge Discovery and Data Mining}, KDD '19, page 2212–2220, New York,
  NY, USA, 2019. Association for Computing Machinery.

\bibitem[\protect\citeauthoryear{Biega \bgroup \em et al.\egroup
  }{2018}]{BGW2018}
Asia~J. Biega, Krishna~P. Gummadi, and Gerhard Weikum.
\newblock Equity of attention: Amortizing individual fairness in rankings.
\newblock In {\em The 41st International ACM SIGIR Conference on Research and
  Development in Information Retrieval}, SIGIR '18, page 405–414, New York,
  NY, USA, 2018. Association for Computing Machinery.

\bibitem[\protect\citeauthoryear{Bobko and Roth}{2004}]{fourfifths}
P.~Bobko and P.L. Roth.
\newblock The four-fifths rule for assessing adverse impact: An arithmetic,
  intuitive, and logical analysis of the rule and implications for future
  research and practice.
\newblock {\em Research in Personnel and Human Resources Management}, 23, 2004.

\bibitem[\protect\citeauthoryear{Bogen and Rieke}{2018}]{upturn}
Miranda Bogen and Aaron Rieke.
\newblock Help wanted: An examination of hiring algorithms, equity, and bias.
\newblock 2018.

\bibitem[\protect\citeauthoryear{Castillo}{2019}]{Castillo2019}
Carlos Castillo.
\newblock Fairness and transparency in ranking.
\newblock {\em SIGIR Forum}, 52(2):64–71, jan 2019.

\bibitem[\protect\citeauthoryear{Celis \bgroup \em et al.\egroup
  }{2018a}]{CKSDKV2018fair}
L.~Elisa Celis, Vijay Keswani, Damian Straszak, Amit Deshpande, Tarun Kathuria,
  and Nisheeth~K. Vishnoi.
\newblock Fair and diverse dpp-based data summarization.
\newblock In {\em Proceedings of the 35th International Conference on Machine
  Learning, {ICML} 2018, Stockholmsm{\"{a}}ssan, Stockholm, Sweden, July 10-15,
  2018}, volume~80 of {\em Proceedings of Machine Learning Research}, pages
  715--724. {PMLR}, 2018.

\bibitem[\protect\citeauthoryear{Celis \bgroup \em et al.\egroup
  }{2018b}]{CSV2018}
L.~Elisa Celis, Damian Straszak, and Nisheeth~K. Vishnoi.
\newblock Ranking with fairness constraints.
\newblock In {\em ICALP}, 2018.

\bibitem[\protect\citeauthoryear{Celis \bgroup \em et al.\egroup
  }{2019}]{CKSV2019controlling}
L.~Elisa Celis, Sayash Kapoor, Farnood Salehi, and Nisheeth Vishnoi.
\newblock Controlling polarization in personalization: An algorithmic
  framework.
\newblock In {\em Proceedings of the Conference on Fairness, Accountability,
  and Transparency}, FAT* '19, page 160–169. Association for Computing
  Machinery, 2019.

\bibitem[\protect\citeauthoryear{Celis \bgroup \em et al.\egroup
  }{2020a}]{CKV2020}
L.~Elisa Celis, Vijay Keswani, and Nisheeth~K. Vishnoi.
\newblock Data preprocessing to mitigate bias: A maximum entropy based
  approach.
\newblock In {\em ICML}, 2020.

\bibitem[\protect\citeauthoryear{Celis \bgroup \em et al.\egroup
  }{2020b}]{CMV2020}
L.~Elisa Celis, Anay Mehrotra, and Nisheeth~K. Vishnoi.
\newblock Interventions for ranking in the presence of implicit bias.
\newblock In {\em Proceedings of the 2020 Conference on Fairness,
  Accountability, and Transparency}, page 369–380. Association for Computing
  Machinery, 2020.

\bibitem[\protect\citeauthoryear{Chierichetti \bgroup \em et al.\egroup
  }{2017}]{CKLV2017}
Flavio Chierichetti, Ravi Kumar, Silvio Lattanzi, and Sergei Vassilvitskii.
\newblock Fair clustering through fairlets.
\newblock In I.~Guyon, U.~V. Luxburg, S.~Bengio, H.~Wallach, R.~Fergus,
  S.~Vishwanathan, and R.~Garnett, editors, {\em Advances in Neural Information
  Processing Systems}, volume~30. Curran Associates, Inc., 2017.

\bibitem[\protect\citeauthoryear{Collins}{2007}]{rooney}
Brian~W. Collins.
\newblock Tackling unconscious bias in hiring practices: The plight of the
  rooney rule.
\newblock {\em New York University Law Review}, 82, 2007.

\bibitem[\protect\citeauthoryear{Cousins and Vempala}{2018}]{CV2018}
Benjamin~R. Cousins and Santosh~S. Vempala.
\newblock Gaussian cooling and $\mathcal{O}^*(n^3)$ algorithms for volume and
  gaussian volume.
\newblock {\em SIAM J. Comput.}, 47:1237--1273, 2018.

\bibitem[\protect\citeauthoryear{Dastin}{2018}]{amazon}
Jeffrey Dastin.
\newblock Amazon scraps secret ai recruiting tool that showed bias against
  women.
\newblock 2018.

\bibitem[\protect\citeauthoryear{Diaz \bgroup \em et al.\egroup
  }{2020}]{DME+2020}
Fernando Diaz, Bhaskar Mitra, Michael~D. Ekstrand, Asia~J. Biega, and Ben
  Carterette.
\newblock {\em Evaluating Stochastic Rankings with Expected Exposure}, page
  275–284.
\newblock 2020.

\bibitem[\protect\citeauthoryear{Dua and Graff}{2017}]{Dua2019}
Dheeru Dua and Casey Graff.
\newblock {UCI} machine learning repository, 2017.

\bibitem[\protect\citeauthoryear{Dyer \bgroup \em et al.\egroup
  }{1991}]{DFK1991}
Martin Dyer, Alan Frieze, and Ravi Kannan.
\newblock A random polynomial-time algorithm for approximating the volume of
  convex bodies.
\newblock {\em J. ACM}, 38(1):1–17, jan 1991.

\bibitem[\protect\citeauthoryear{Gao and Shah}{2020}]{GS2020}
Ruoyuan Gao and Chirag Shah.
\newblock Toward creating a fairer ranking in search engine results.
\newblock {\em Information Processing and Management}, 57(1):102138, 2020.

\bibitem[\protect\citeauthoryear{Geyik \bgroup \em et al.\egroup
  }{2019}]{GAK2019}
Sahin~Cem Geyik, Stuart Ambler, and Krishnaram Kenthapadi.
\newblock Fairness-aware ranking in search and recommendation systems with
  application to linkedin talent search.
\newblock In {\em Proceedings of the 25th ACM SIGKDD International Conference
  on Knowledge Discovery and Data Mining}, KDD '19, page 2221–2231, New York,
  NY, USA, 2019. Association for Computing Machinery.

\bibitem[\protect\citeauthoryear{Gorantla \bgroup \em et al.\egroup
  }{2021}]{GDL2021}
Sruthi Gorantla, Amit Deshpande, and Anand Louis.
\newblock On the problem of underranking in group-fair ranking.
\newblock In Marina Meila and Tong Zhang, editors, {\em Proceedings of the 38th
  International Conference on Machine Learning}, volume 139 of {\em Proceedings
  of Machine Learning Research}, pages 3777--3787. PMLR, 18--24 Jul 2021.

\bibitem[\protect\citeauthoryear{Goto \bgroup \em et al.\egroup
  }{2016}]{GOTO201640}
Masahiro Goto, Atsushi Iwasaki, Yujiro Kawasaki, Ryoji Kurata, Yosuke Yasuda,
  and Makoto Yokoo.
\newblock Strategyproof matching with regional minimum and maximum quotas.
\newblock {\em Artificial Intelligence}, 235:40--57, 2016.

\bibitem[\protect\citeauthoryear{Hassani}{2021}]{bias3}
Bertrand~K. Hassani.
\newblock Societal bias reinforcement through machine learning: a credit
  scoring perspective.
\newblock {\em AI and Ethics}, 1:239--247, 2021.

\bibitem[\protect\citeauthoryear{Heuss \bgroup \em et al.\egroup
  }{2022}]{MSdR2022}
Maria Heuss, Fatemeh Sarvi, and Maarten de~Rijke.
\newblock Fairness of exposure in light of incomplete exposure estimation.
\newblock SIGIR '22, page 759–769, 2022.

\bibitem[\protect\citeauthoryear{Kannan and Vempala}{1997}]{KV1997}
Ravi Kannan and Santosh Vempala.
\newblock Sampling lattice points.
\newblock In {\em Proceedings of the Twenty-Ninth Annual ACM Symposium on
  Theory of Computing}, STOC '97, page 696–700, New York, NY, USA, 1997.
  Association for Computing Machinery.

\bibitem[\protect\citeauthoryear{Kletti \bgroup \em et al.\egroup
  }{2022}]{expohedron}
Till Kletti, Jean-Michel Renders, and Patrick Loiseau.
\newblock Introducing the expohedron for efficient pareto-optimal
  fairness-utility amortizations in repeated rankings.
\newblock WSDM '22. Association for Computing Machinery, 2022.

\bibitem[\protect\citeauthoryear{Kook \bgroup \em et al.\egroup
  }{2022}]{kook2022sampling}
Yunbum Kook, Yin~Tat Lee, Ruoqi Shen, and Santosh~S. Vempala.
\newblock Sampling with riemannian hamiltonian monte carlo in a constrained
  space, 2022.

\bibitem[\protect\citeauthoryear{Kuhlman \bgroup \em et al.\egroup
  }{2019}]{KVR2019}
Caitlin Kuhlman, MaryAnn VanValkenburg, and Elke Rundensteiner.
\newblock Fare: Diagnostics for fair ranking using pairwise error metrics.
\newblock In {\em The World Wide Web Conference}, WWW '19, page 2936–2942,
  New York, NY, USA, 2019. Association for Computing Machinery.

\bibitem[\protect\citeauthoryear{Lov\'{a}sz and Vempala}{2006}]{LV2006}
L\'{a}szl\'{o} Lov\'{a}sz and Santosh Vempala.
\newblock Hit-and-run from a corner.
\newblock {\em SIAM J. Comput.}, 35(4):985–1005, apr 2006.

\bibitem[\protect\citeauthoryear{Memarrast \bgroup \em et al.\egroup
  }{2021}]{robustltr}
Omid Memarrast, Ashkan Rezaei, Rizal Fathony, and Brian~D. Ziebart.
\newblock Fairness for robust learning to rank.
\newblock {\em CoRR}, abs/2112.06288, 2021.

\bibitem[\protect\citeauthoryear{Narasimhan \bgroup \em et al.\egroup
  }{2020}]{NCGW2020}
Harikrishna Narasimhan, Andy Cotter, Maya Gupta, and Serena~Lutong Wang.
\newblock Pairwise fairness for ranking and regression.
\newblock In {\em 33rd AAAI Conference on Artificial Intelligence}, 2020.

\bibitem[\protect\citeauthoryear{Okonofua and Eberhardt}{2015}]{bias2}
Jason~A. Okonofua and Jennifer~L. Eberhardt.
\newblock Two strikes: Race and the disciplining of young students.
\newblock {\em Psychological Science}, 26(5):617--624, 2015.

\bibitem[\protect\citeauthoryear{Pak}{2000}]{Pak2000OnSI}
Igor Pak.
\newblock On sampling integer points in polyhedra.
\newblock {\em Foundations of Computational Mathematics}, 2000.

\bibitem[\protect\citeauthoryear{Singh and Joachims}{2018}]{SJ2018}
Ashudeep Singh and Thorsten Joachims.
\newblock Fairness of exposure in rankings.
\newblock In {\em Proceedings of the 24th ACM SIGKDD International Conference
  on Knowledge Discovery and Data Mining}, page 2219–2228, 2018.

\bibitem[\protect\citeauthoryear{Singh \bgroup \em et al.\egroup
  }{2021}]{SKJ2021}
Ashudeep Singh, David Kempe, and Thorsten Joachims.
\newblock Fairness in ranking under uncertainty, 2021.

\bibitem[\protect\citeauthoryear{Stoyanovich \bgroup \em et al.\egroup
  }{2018}]{SYV2018}
Julia Stoyanovich, Ke~Yang, and HV~Jagadish.
\newblock Online set selection with fairness and diversity constraints.
\newblock Advances in Database Technology - EDBT, pages 241--252, 2018.

\bibitem[\protect\citeauthoryear{Uhlmann and Cohen}{2005}]{bias1}
Eric~Luis Uhlmann and Geoffrey~L. Cohen.
\newblock Constructed criteria: Redefining merit to justify discrimination.
\newblock {\em Psychological Science}, 16(6):474--480, 2005.

\bibitem[\protect\citeauthoryear{\v{S}tefankovi\v{c} \bgroup \em et al.\egroup
  }{2012}]{SVV2012}
Daniel \v{S}tefankovi\v{c}, Santosh Vempala, and Eric Vigoda.
\newblock A deterministic polynomial-time approximation scheme for counting
  knapsack solutions.
\newblock {\em SIAM Journal on Computing}, 41(2):356--366, 2012.

\bibitem[\protect\citeauthoryear{Wu \bgroup \em et al.\egroup }{2018}]{WZW2018}
Yongkai Wu, Lu~Zhang, and Xintao Wu.
\newblock On discrimination discovery and removal in ranked data using causal
  graph.
\newblock In {\em Proceedings of the 24th ACM SIGKDD International Conference
  on Knowledge Discovery and Data Mining}, KDD '18, page 2536–2544, New York,
  NY, USA, 2018. Association for Computing Machinery.

\bibitem[\protect\citeauthoryear{Yang and Stoyanovich}{2017}]{YS2017}
Ke~Yang and Julia Stoyanovich.
\newblock Measuring fairness in ranked outputs.
\newblock In {\em Proceedings of the 29th International Conference on
  Scientific and Statistical Database Management}, SSDBM '17, New York, NY,
  USA, 2017. Association for Computing Machinery.

\bibitem[\protect\citeauthoryear{Zehlike \bgroup \em et al.\egroup
  }{2017}]{ZBCHMB2017}
Meike Zehlike, Francesco Bonchi, Carlos Castillo, Sara Hajian, Mohamed Megahed,
  and Ricardo Baeza-Yates.
\newblock Fa*ir: A fair top-k ranking algorithm.
\newblock In {\em Proceedings of the 2017 ACM on Conference on Information and
  Knowledge Management}, CIKM '17, page 1569–1578, New York, NY, USA, 2017.
  Association for Computing Machinery.

\bibitem[\protect\citeauthoryear{Zehlike \bgroup \em et al.\egroup
  }{2022a}]{ZEHLIKE2022102707}
Meike Zehlike, Tom Sühr, Ricardo Baeza-Yates, Francesco Bonchi, Carlos
  Castillo, and Sara Hajian.
\newblock Fair top-k ranking with multiple protected groups.
\newblock {\em Information Processing and Management}, 59(1):102707, 2022.

\bibitem[\protect\citeauthoryear{Zehlike \bgroup \em et al.\egroup
  }{2022b}]{Zehlike_part1}
Meike Zehlike, Ke~Yang, and Julia Stoyanovich.
\newblock Fairness in ranking, part i: Score-based ranking.
\newblock {\em ACM Comput. Surv.}, apr 2022.
\newblock Just Accepted.

\end{thebibliography}
\appendix

\section{Additional Experiments}

\begin{figure*}
	\centering
	\begin{subfigure}[b]{\linewidth}
		\centering
		\includegraphics[scale=0.125]{legend.pdf} 
	\end{subfigure}
	
	\begin{subfigure}[b]{0.33\linewidth}
		\centering
		\includegraphics[scale=0.17]{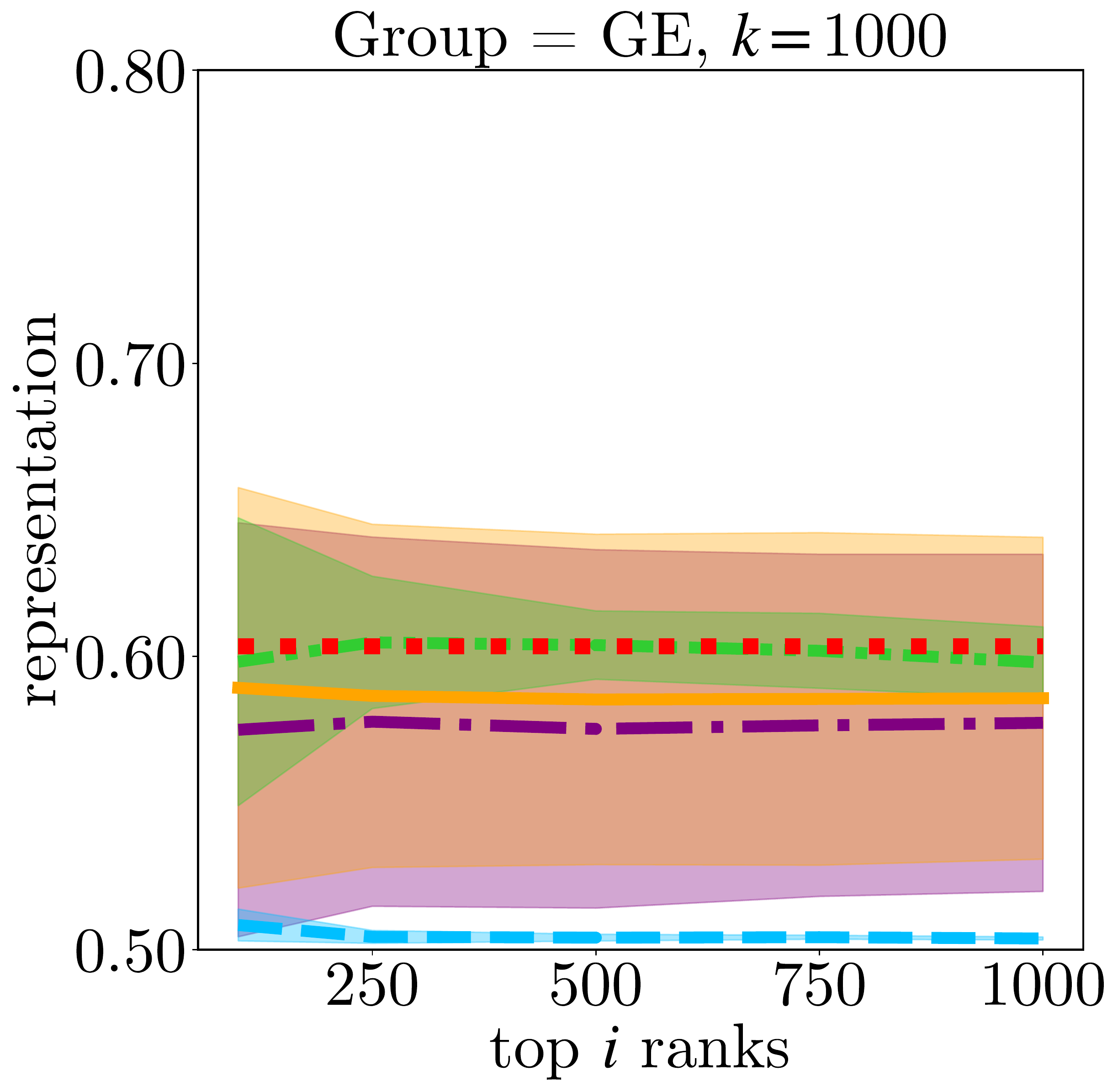} 
		\label{fig:german_25_rep}
	\end{subfigure}
	\begin{subfigure}[b]{0.66\linewidth}
		\centering
		\includegraphics[scale=0.17]{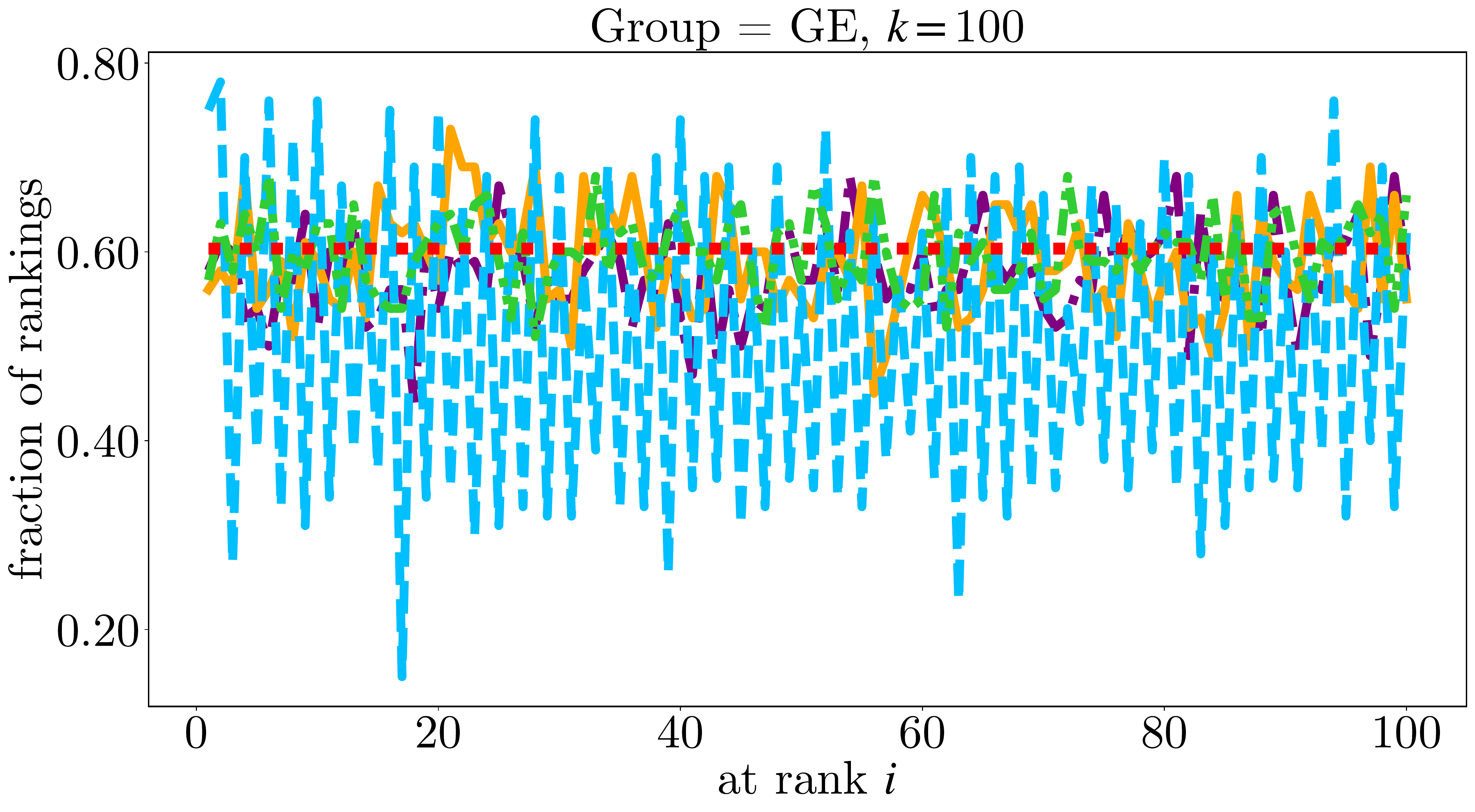} 
		\label{fig:german_35_rep}
	\end{subfigure}
	
	\begin{subfigure}[b]{0.33\linewidth}
		\centering
		\includegraphics[scale=0.17]{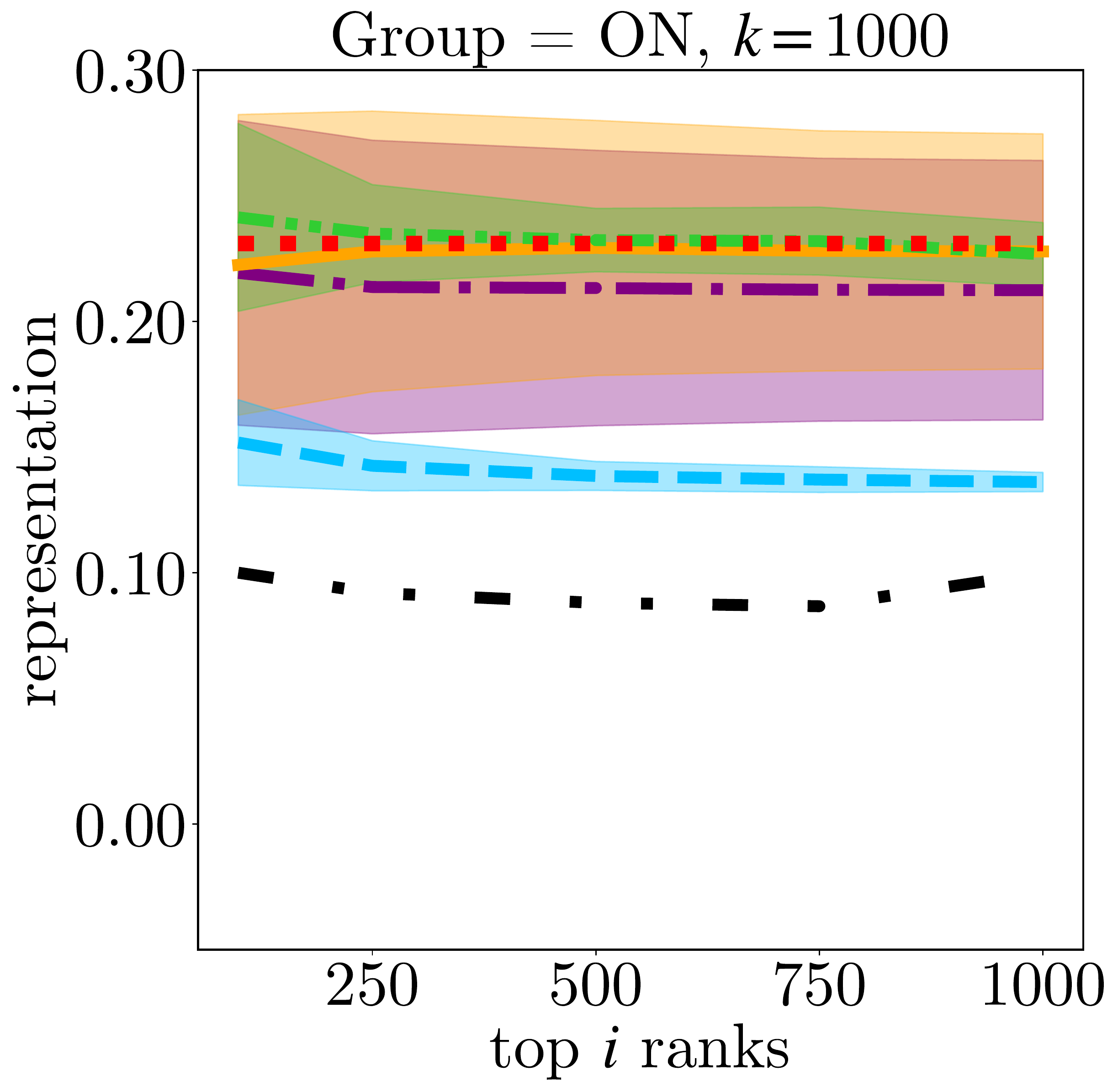} 
		\label{fig:german_25_rep}
	\end{subfigure}
	\begin{subfigure}[b]{0.66\linewidth}
		\centering
		\includegraphics[scale=0.17]{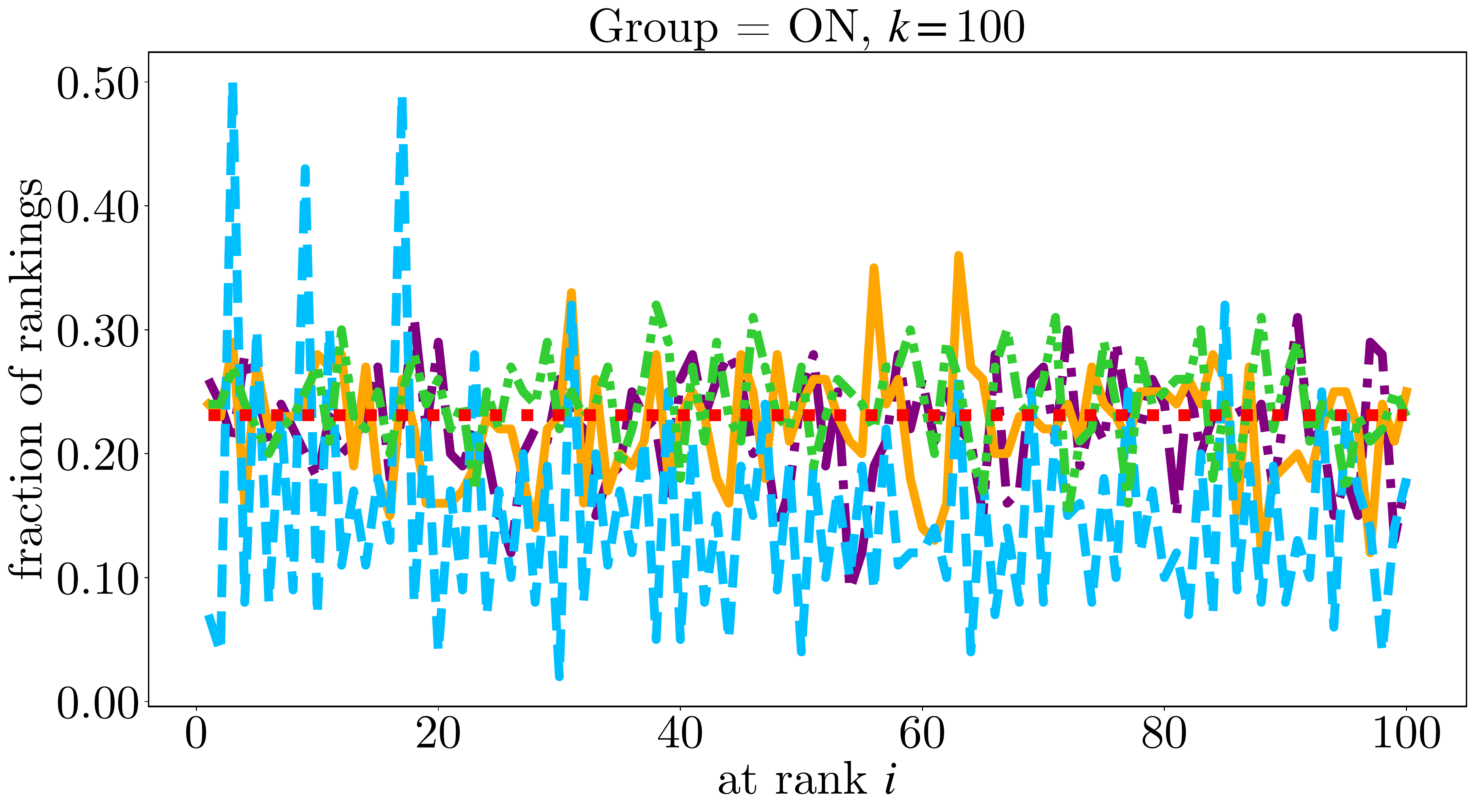} 
		\label{fig:german_35_rep}
	\end{subfigure}
	
	\begin{subfigure}[b]{0.33\linewidth}
		\centering
		\includegraphics[scale=0.17]{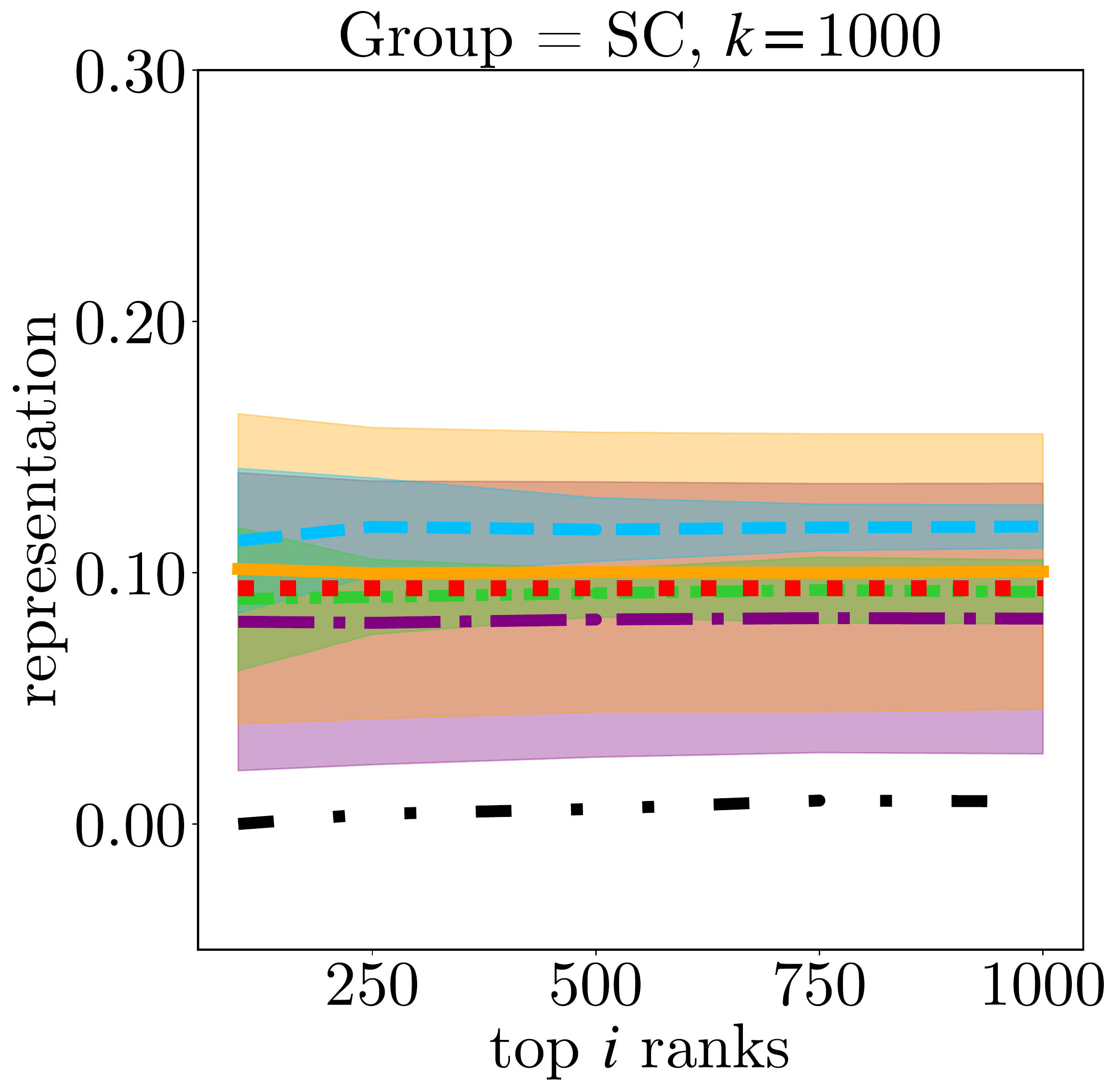} 
		\label{fig:german_25_rep}
	\end{subfigure}
	\begin{subfigure}[b]{0.66\linewidth}
		\centering
		\includegraphics[scale=0.17]{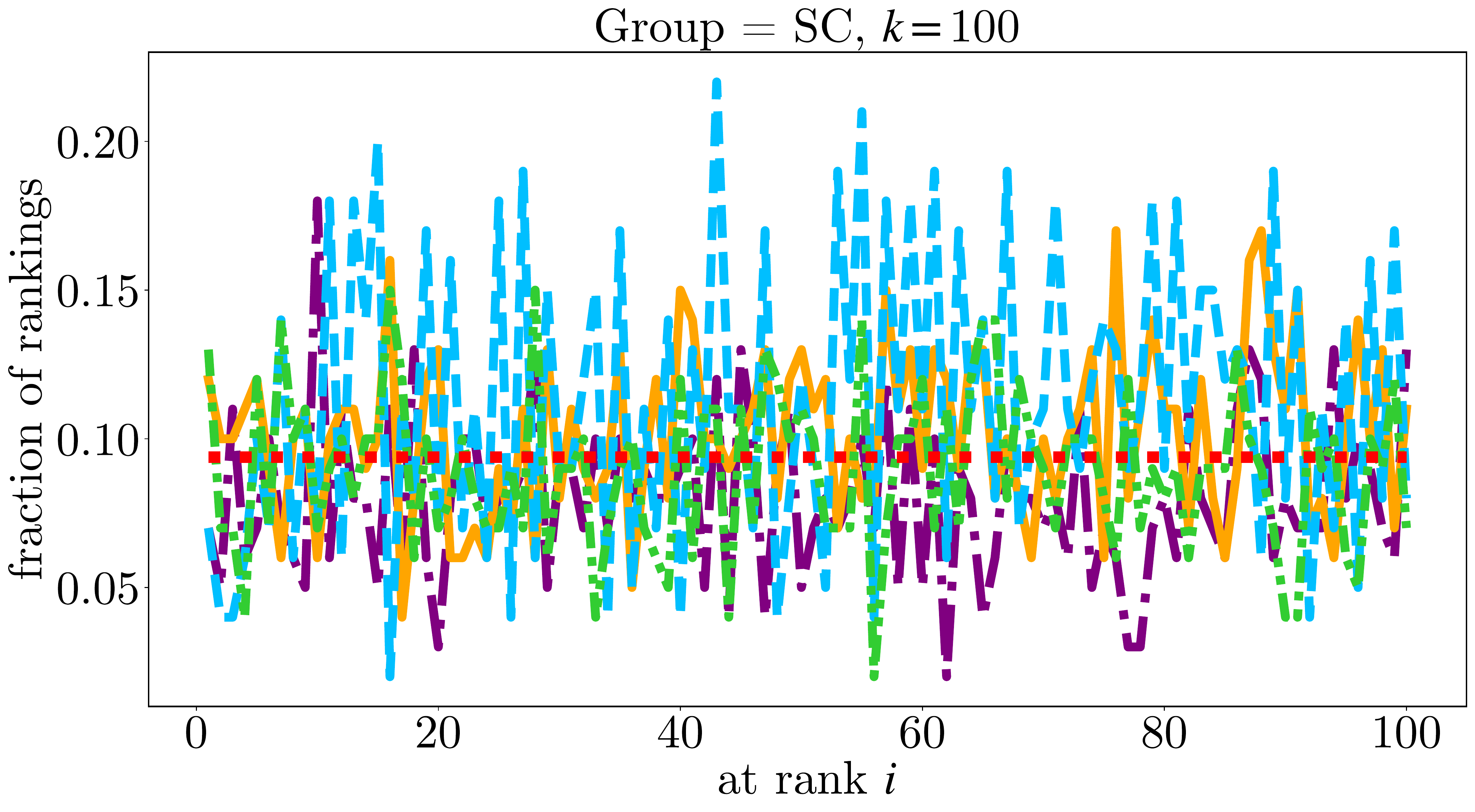} 
		\label{fig:german_35_rep}
	\end{subfigure}
	
	\begin{subfigure}[b]{0.33\linewidth}
		\centering
		\includegraphics[scale=0.17]{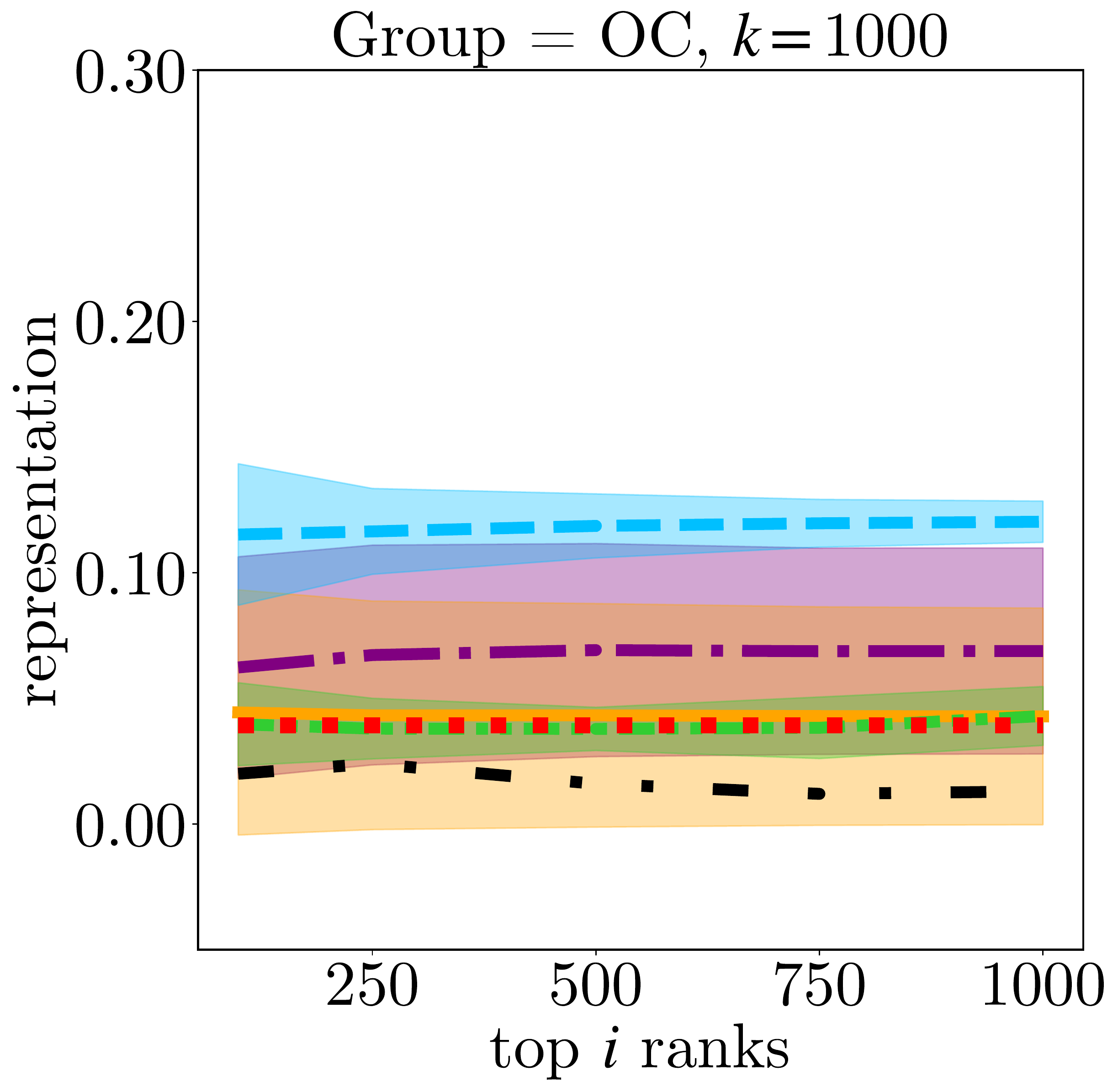} 
		\label{fig:german_25_rep}
	\end{subfigure}
	\begin{subfigure}[b]{0.66\linewidth}
		\centering
		\includegraphics[scale=0.17]{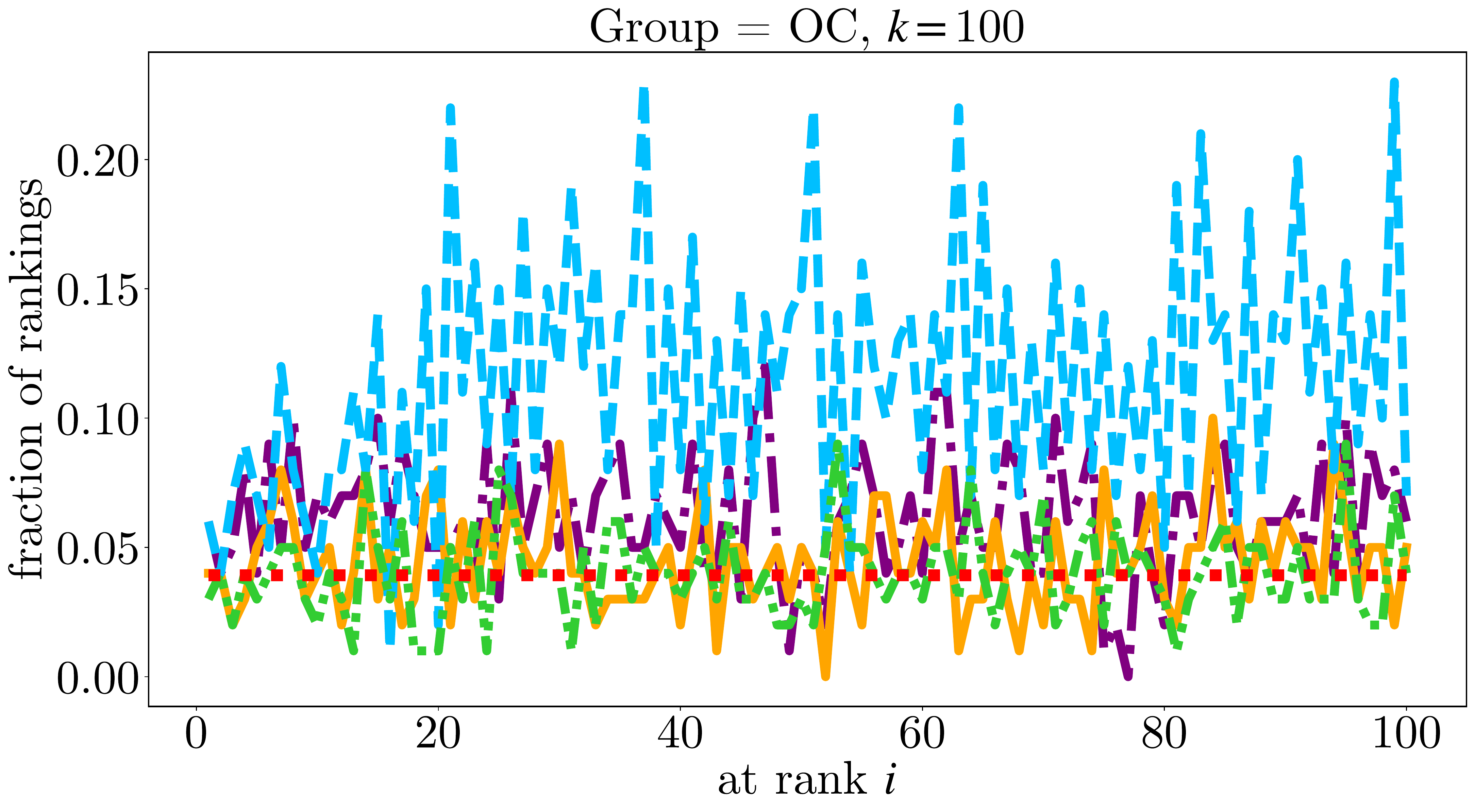} 
		\label{fig:german_35_rep}
	\end{subfigure}

	\begin{subfigure}[b]{0.25\linewidth}
		\centering
		\includegraphics[scale=0.13]{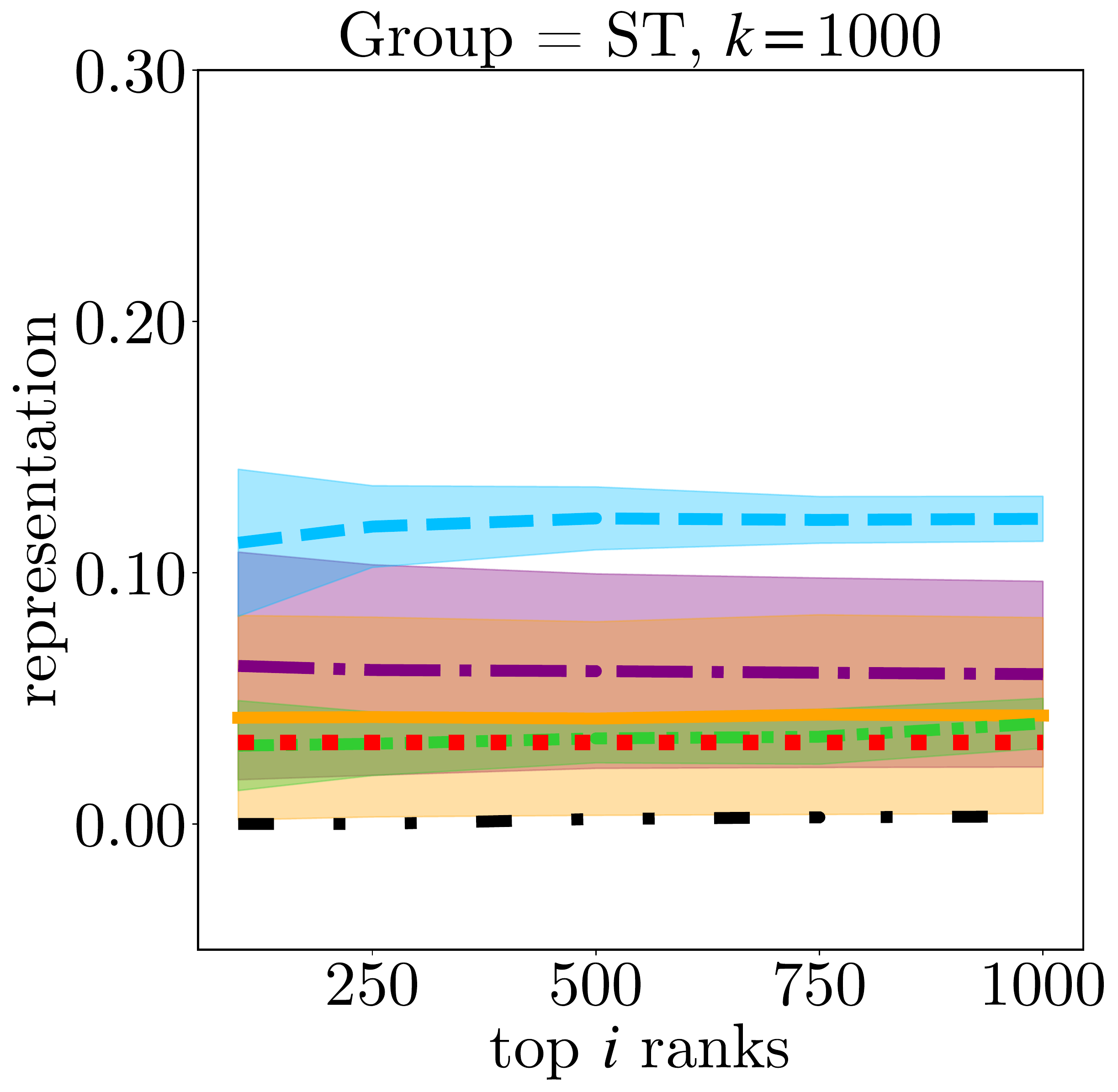} 
		\label{fig:german_25_rep}
	\end{subfigure}
	\begin{subfigure}[b]{0.48\linewidth}
		\centering
		\includegraphics[scale=0.13]{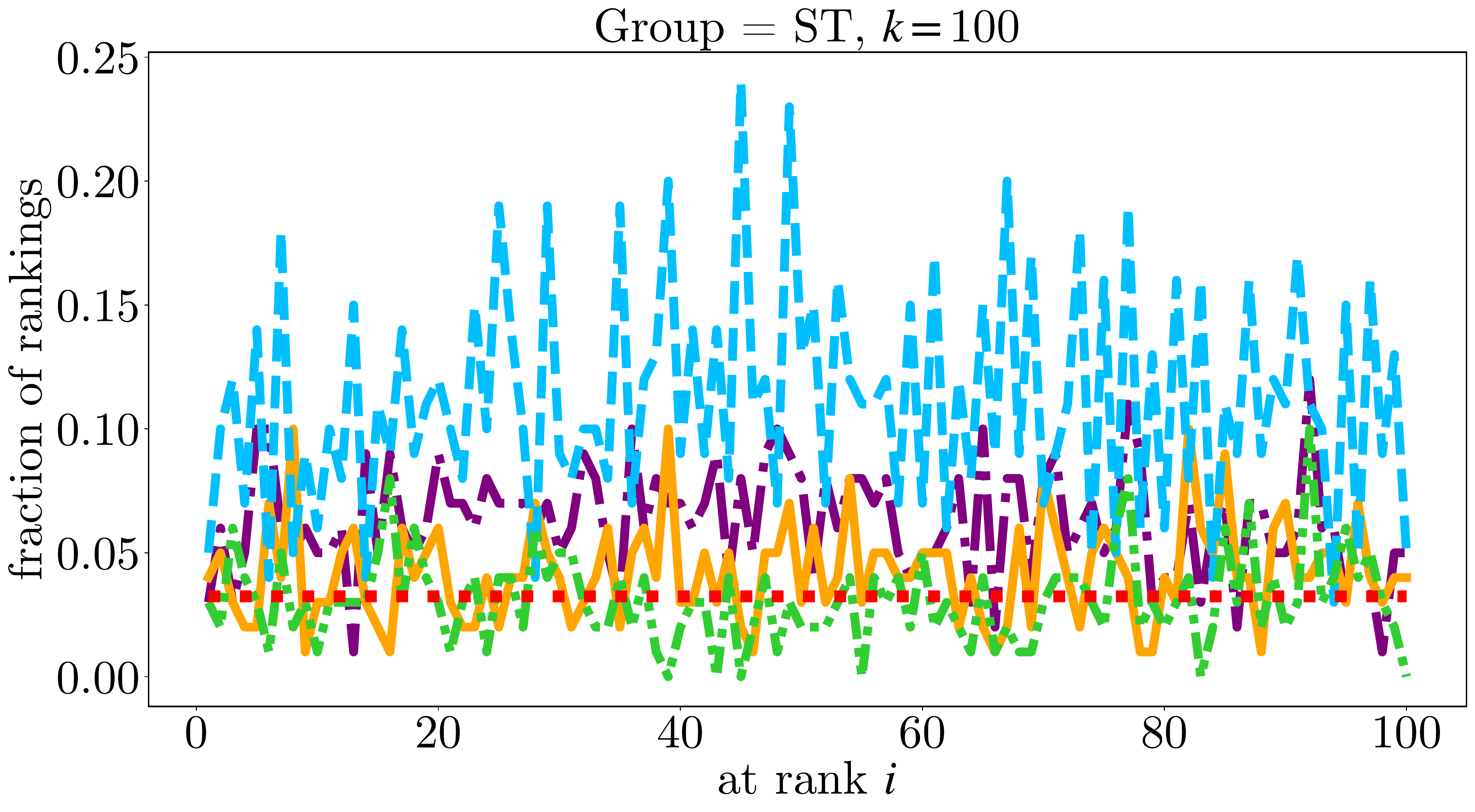} 
		\label{fig:german_35_rep}
	\end{subfigure}
	\begin{subfigure}[b]{0.25\linewidth}
		\centering
		\includegraphics[scale=0.13]{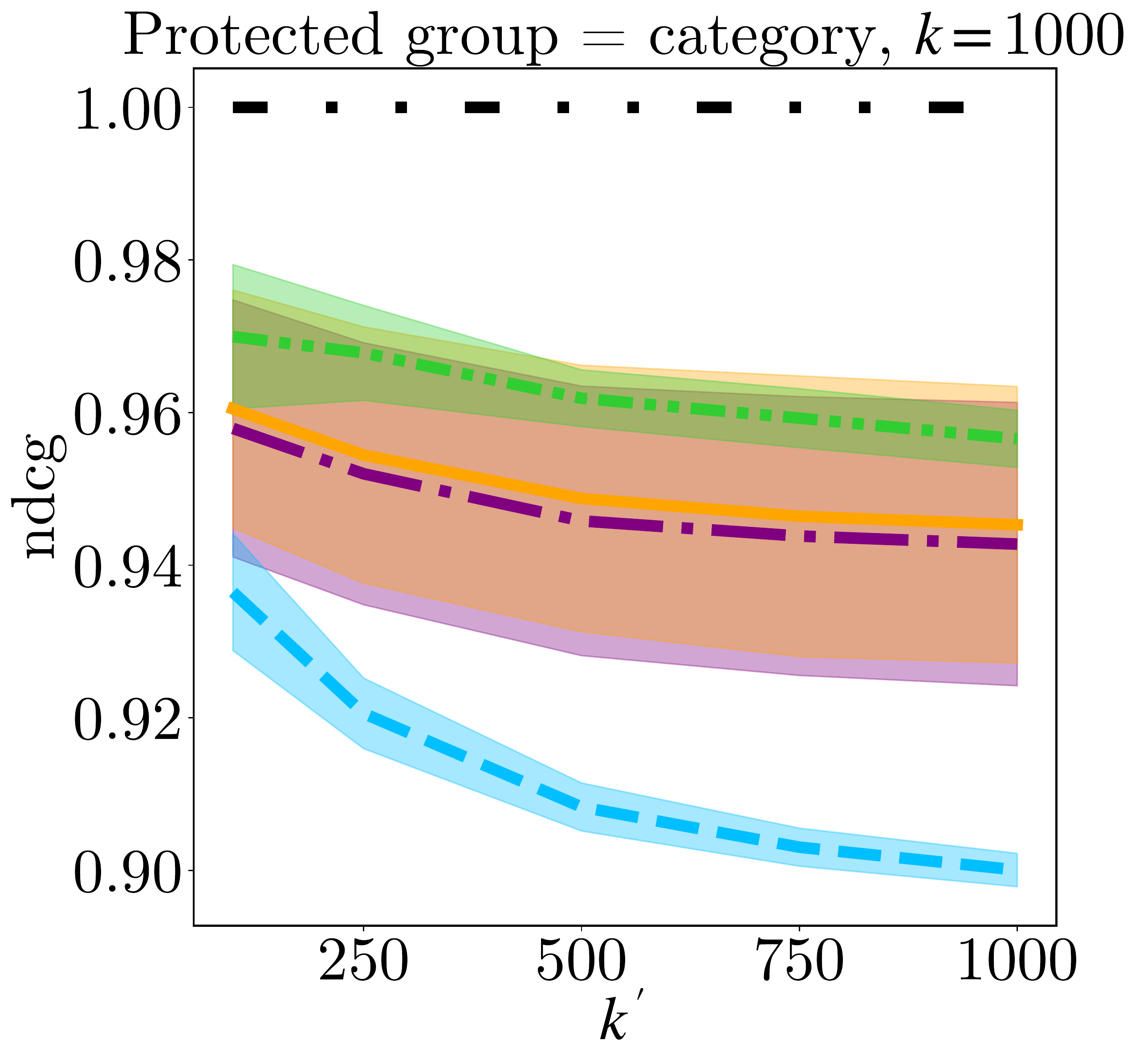} 
		\label{fig:german_35_rep}
	\end{subfigure}

	\caption{Results on the JEE 2009 dataset with \textit{birth category} as the protected group (with $5$ groups). For Fair $\epsilon$-greedy we use $\epsilon =0.3$.}
	\label{fig:jee_category}
\end{figure*}

\begin{figure*}[t]
	\centering
	\begin{subfigure}[b]{\linewidth}
		\centering
		\includegraphics[scale=0.12]{legend.pdf} 
	\end{subfigure}
	
	\begin{subfigure}[b]{0.25\linewidth}
		\centering
		\includegraphics[scale=0.13]{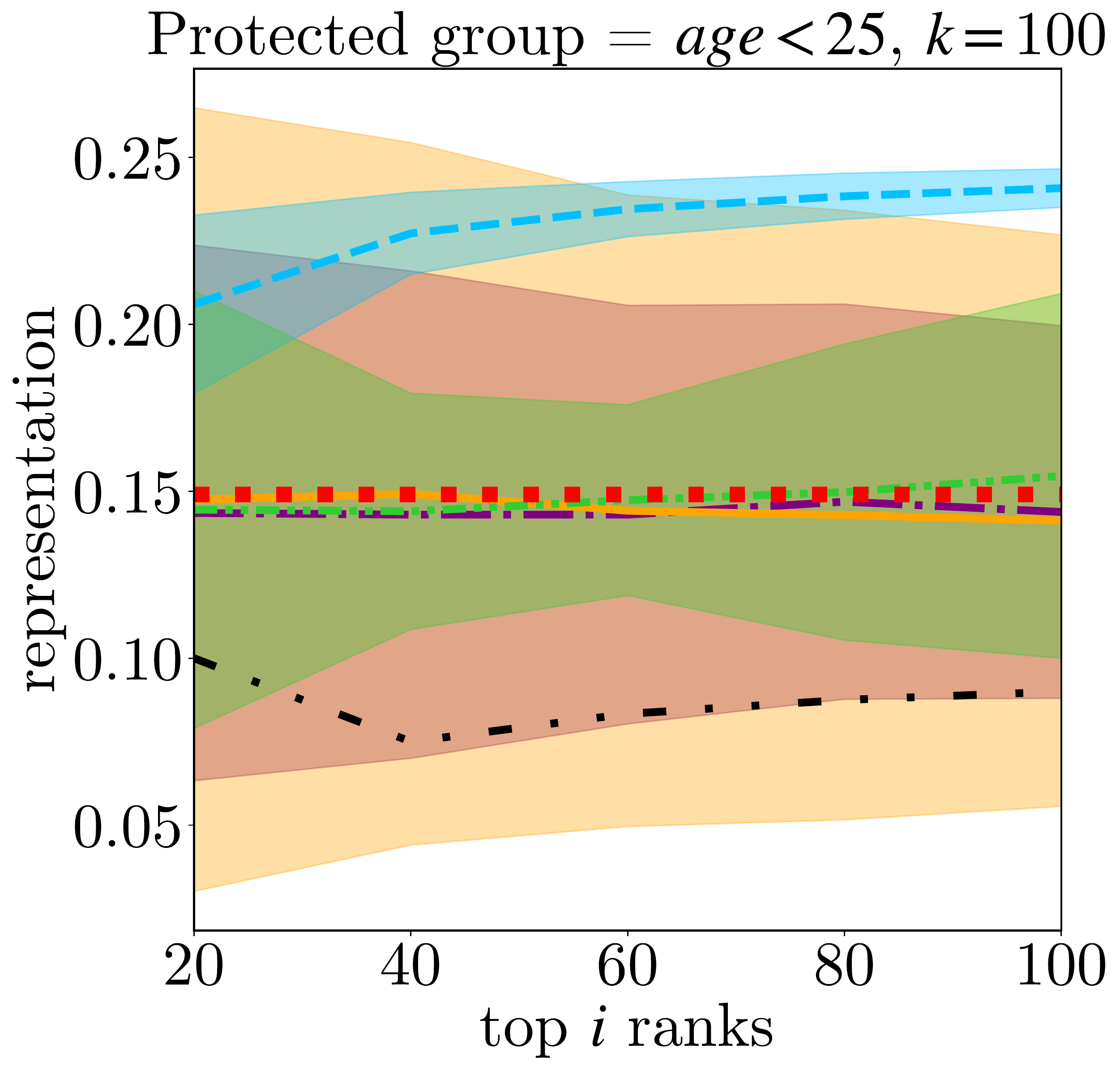} 
		\label{fig:german_25_rep}
	\end{subfigure}
	\begin{subfigure}[b]{0.48\linewidth}
		\centering
		\includegraphics[scale=0.13]{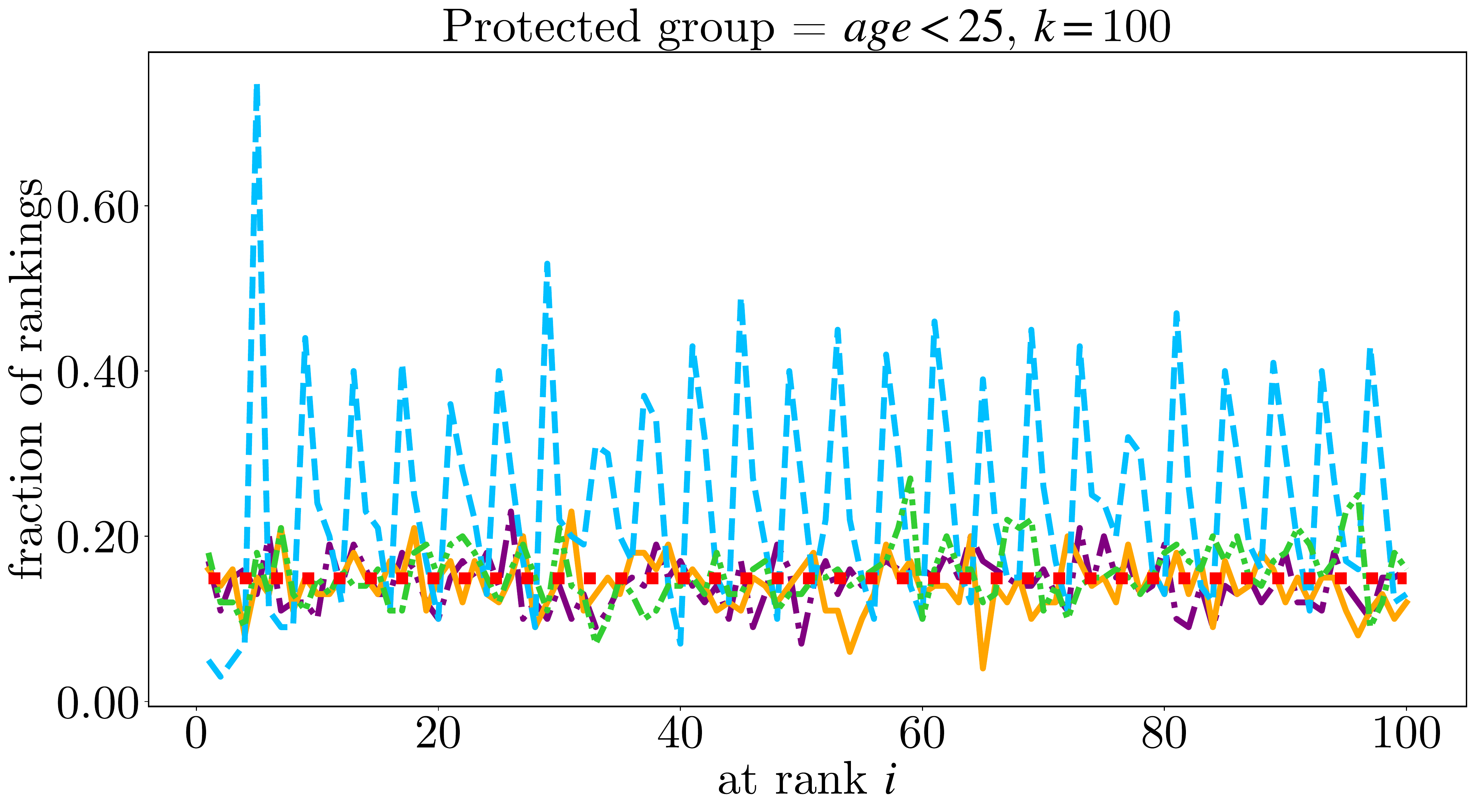} 
		\label{fig:german_35_rep}
	\end{subfigure}
	\begin{subfigure}[b]{0.25\linewidth}
		\centering
		\includegraphics[scale=0.13]{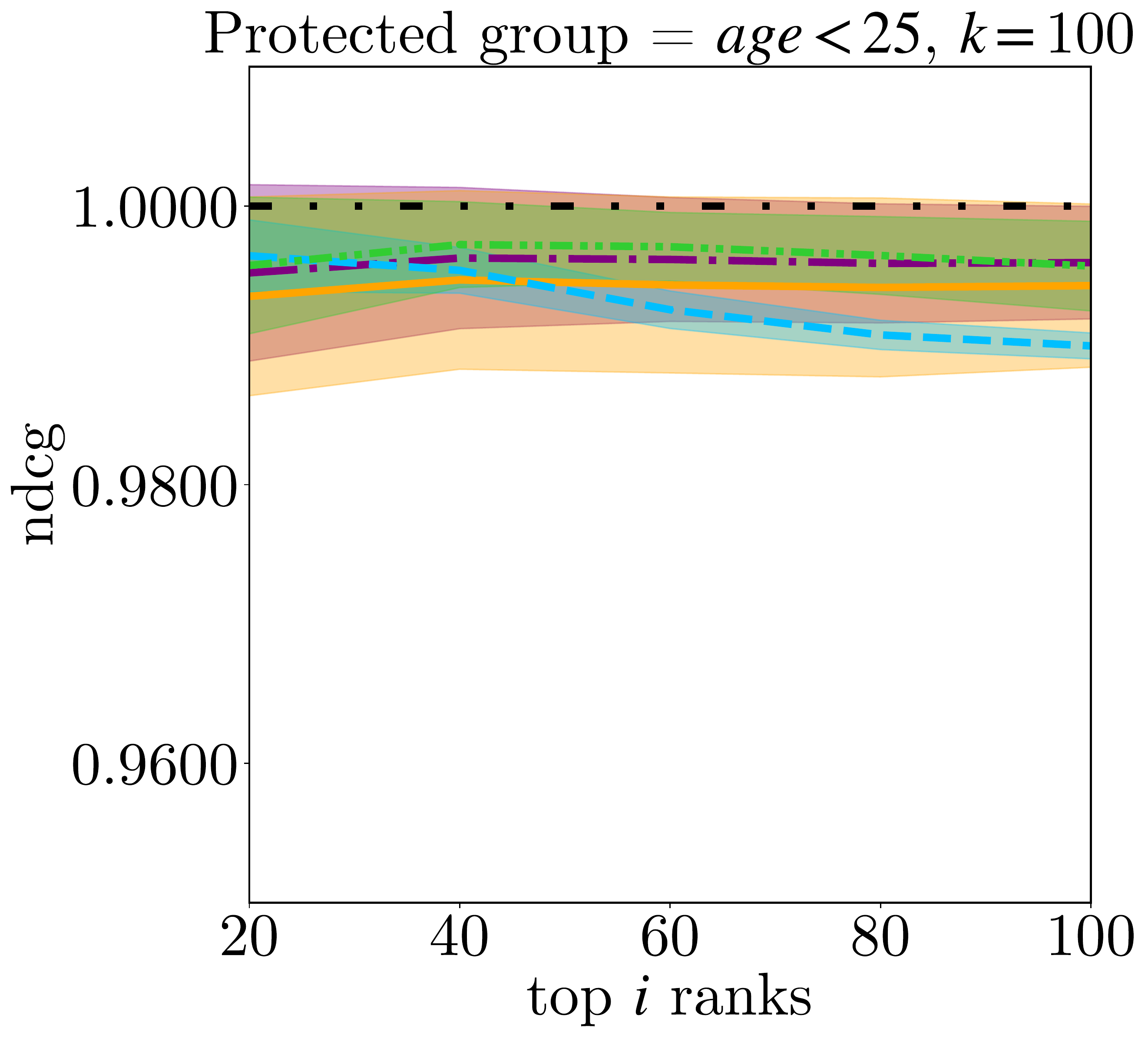} 
		\label{fig:german_35_rep}
	\end{subfigure}

	\begin{subfigure}[b]{0.25\linewidth}
		\centering
		\includegraphics[scale=0.13]{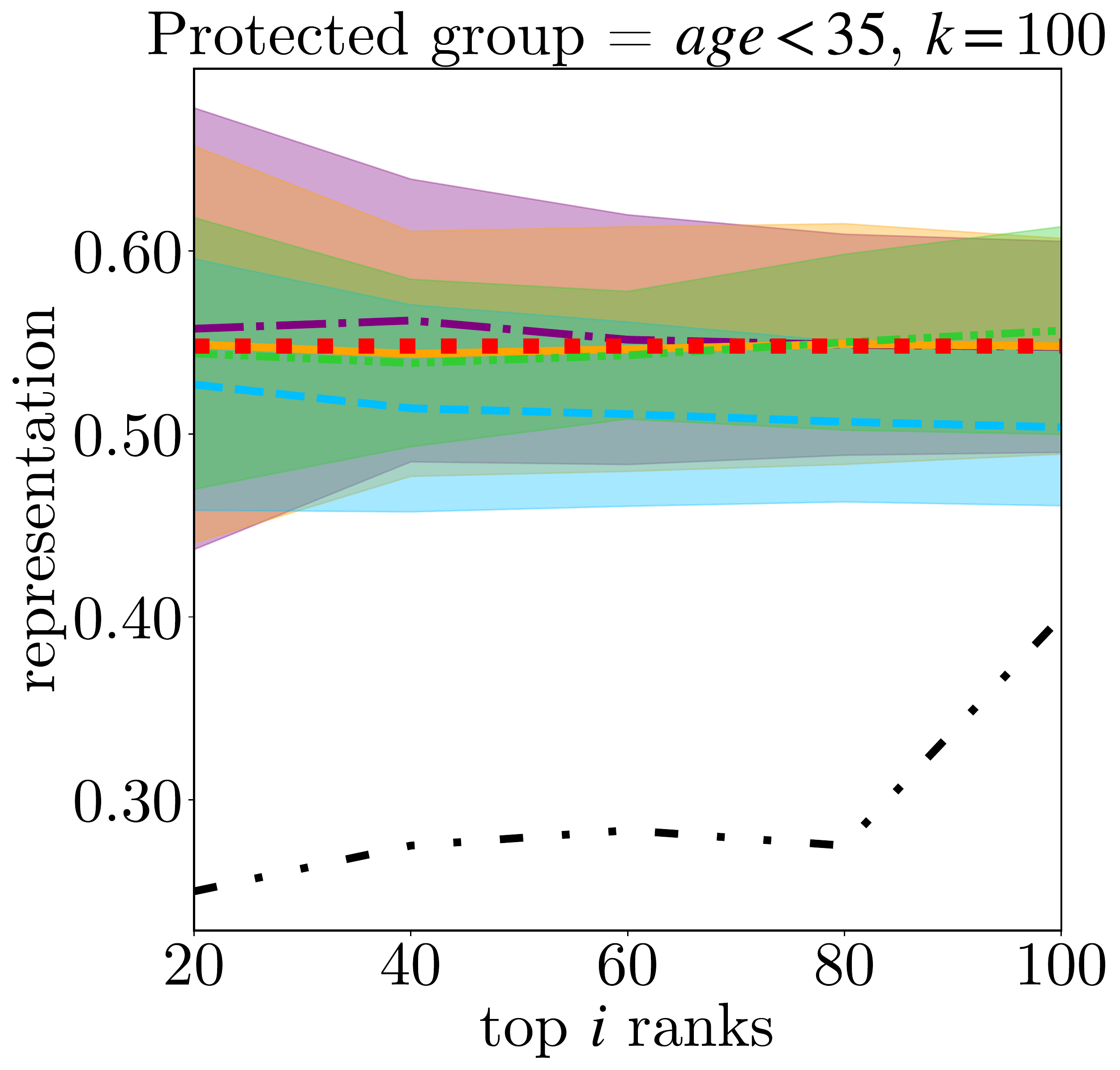} 
		\label{fig:german_25_prop}
	\end{subfigure}
	\begin{subfigure}[b]{0.48\linewidth}
		\centering
		\includegraphics[scale=0.13]{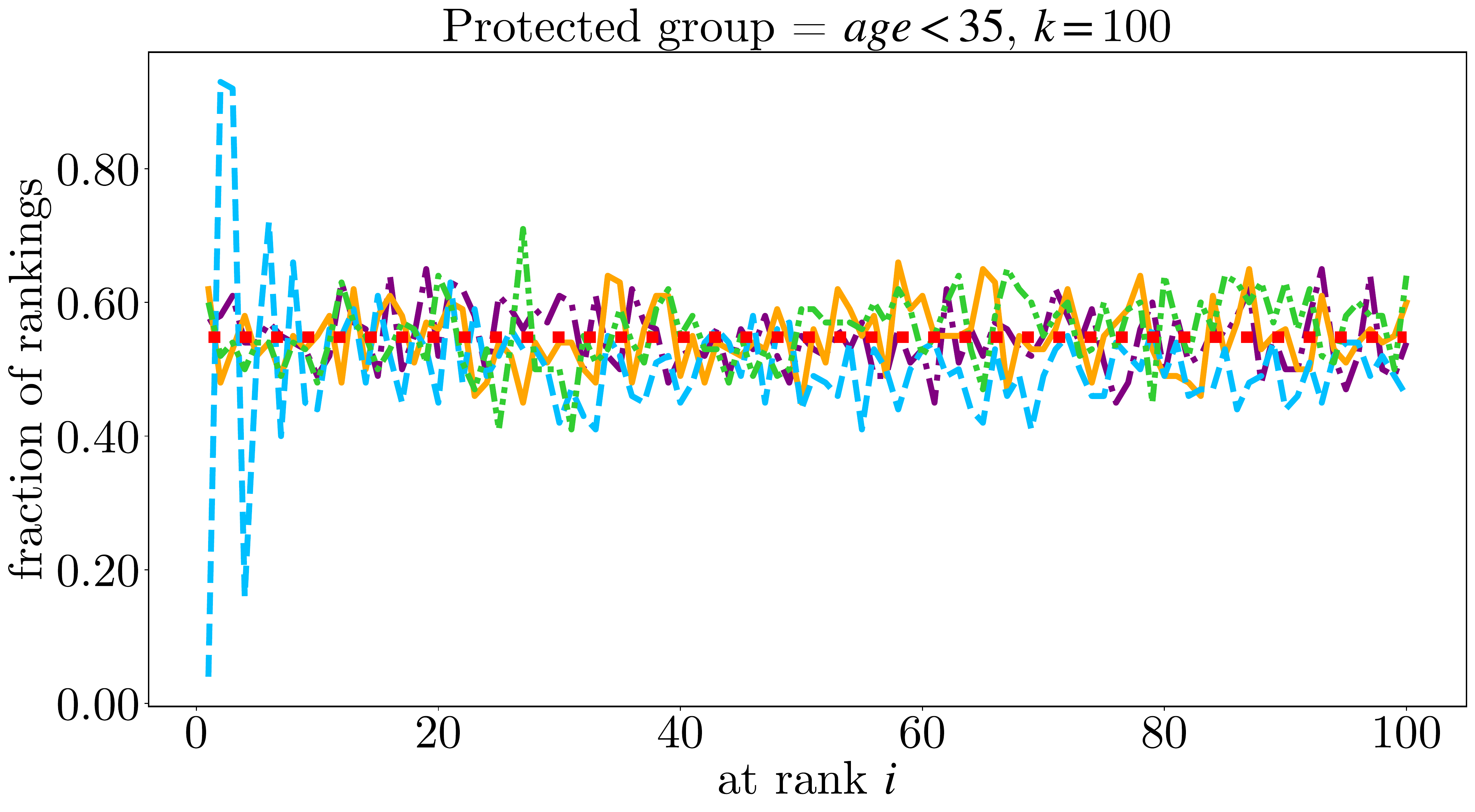} 
		\label{fig:german_35_prop}
	\end{subfigure}
	\begin{subfigure}[b]{0.25\linewidth}
		\centering
		\includegraphics[scale=0.13]{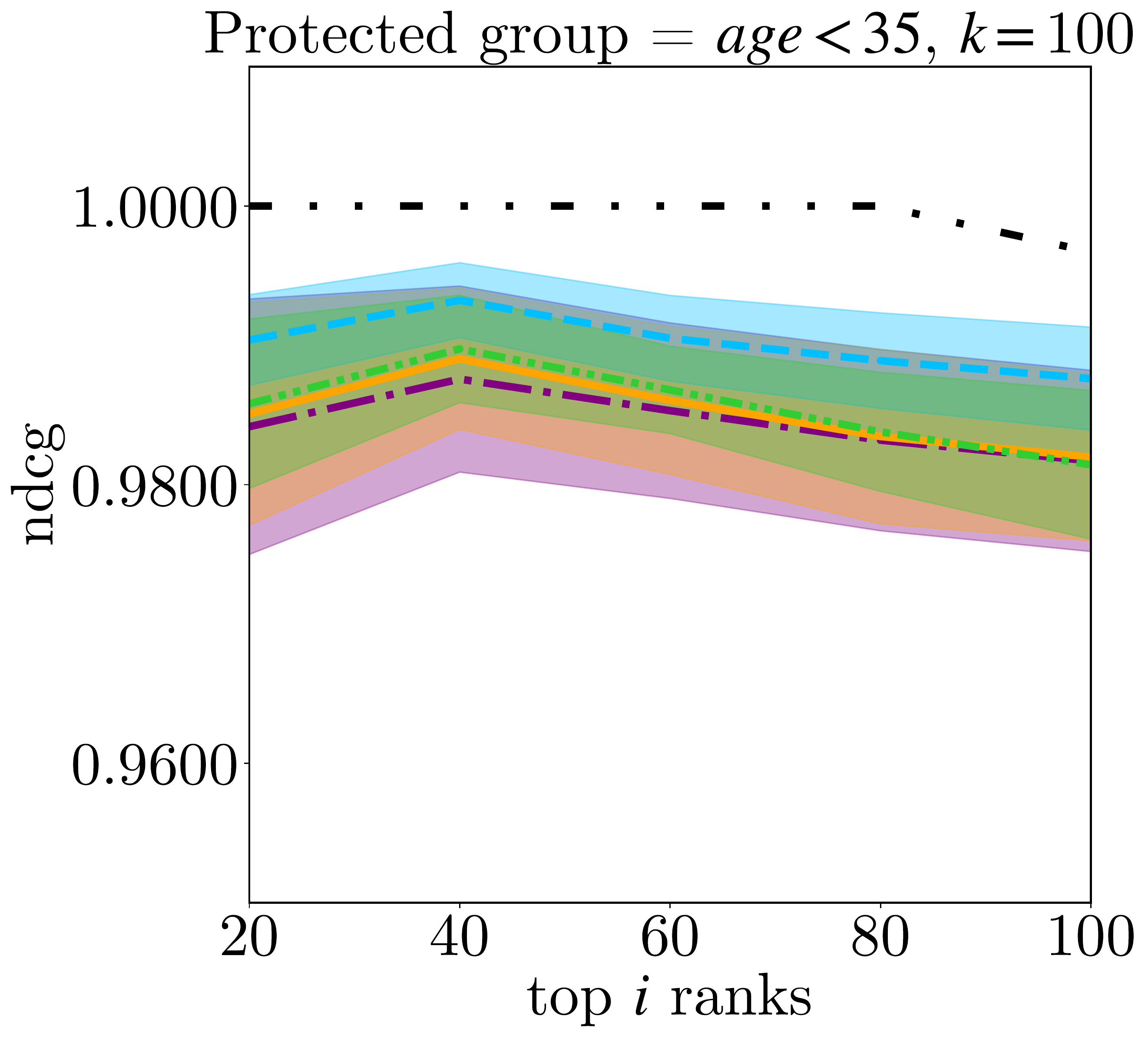} 
		\label{fig:german_35_rep}
	\end{subfigure}

	\caption{Results on the German Credit Risk dataset with \textit{age} $< 25$ as the protected group in the first row and \textit{age} $< 35$ as the protected group in the first row. For Fair $\epsilon$-greedy we use $\epsilon =0.15$.}
	\label{fig:german_binary_eps015}
\end{figure*}
\begin{figure*}[t]
	\centering
	\begin{subfigure}[b]{\linewidth}
		\centering
		\includegraphics[scale=0.12]{legend.pdf} 
	\end{subfigure}
	
	\begin{subfigure}[b]{0.25\linewidth}
		\centering
		\includegraphics[scale=0.13]{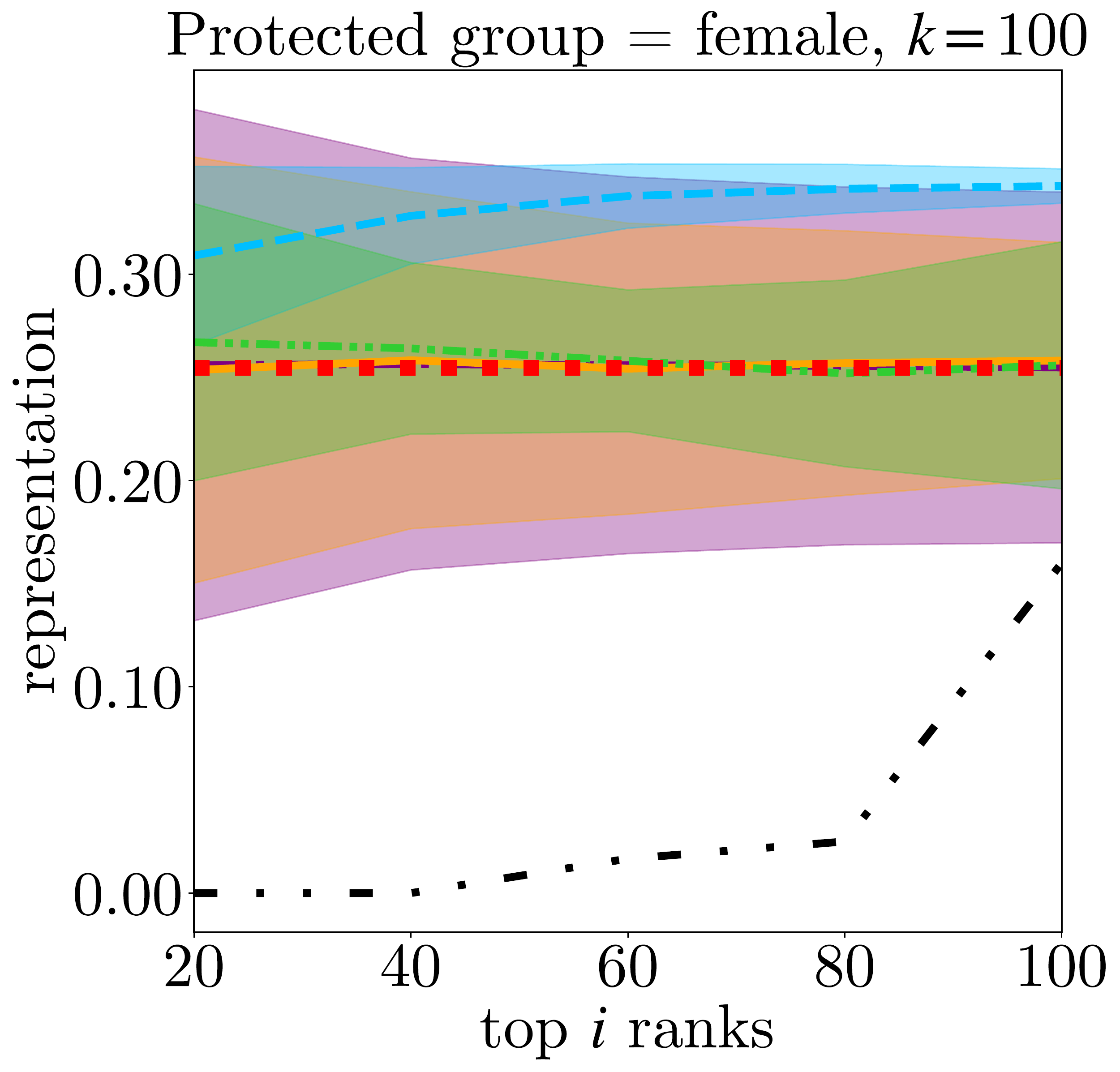} 
		\label{fig:german_25_rep}
	\end{subfigure}
	\begin{subfigure}[b]{0.48\linewidth}
		\centering
		\includegraphics[scale=0.13]{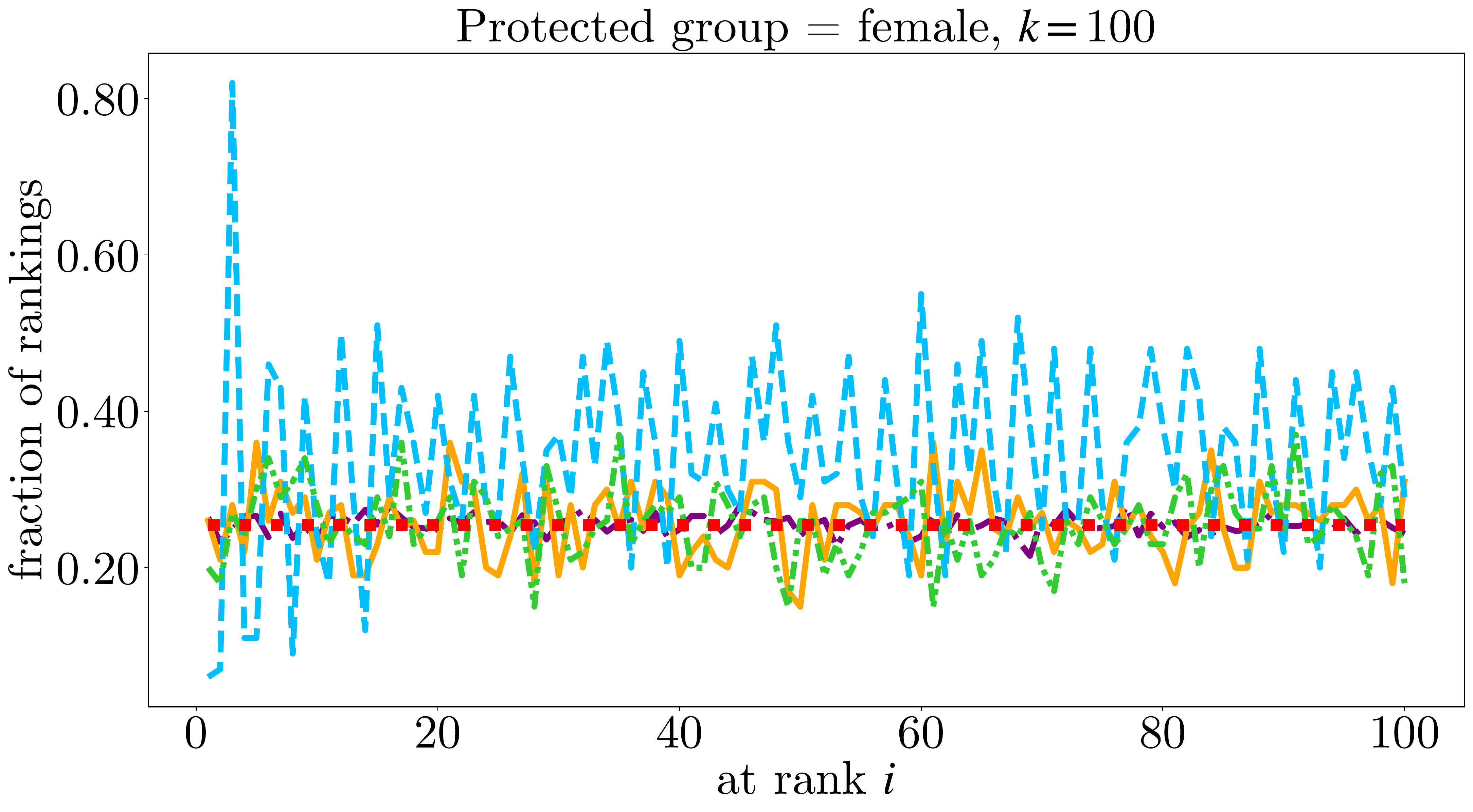} 
		\label{fig:german_35_rep}
	\end{subfigure}
	\begin{subfigure}[b]{0.25\linewidth}
		\centering
		\includegraphics[scale=0.13]{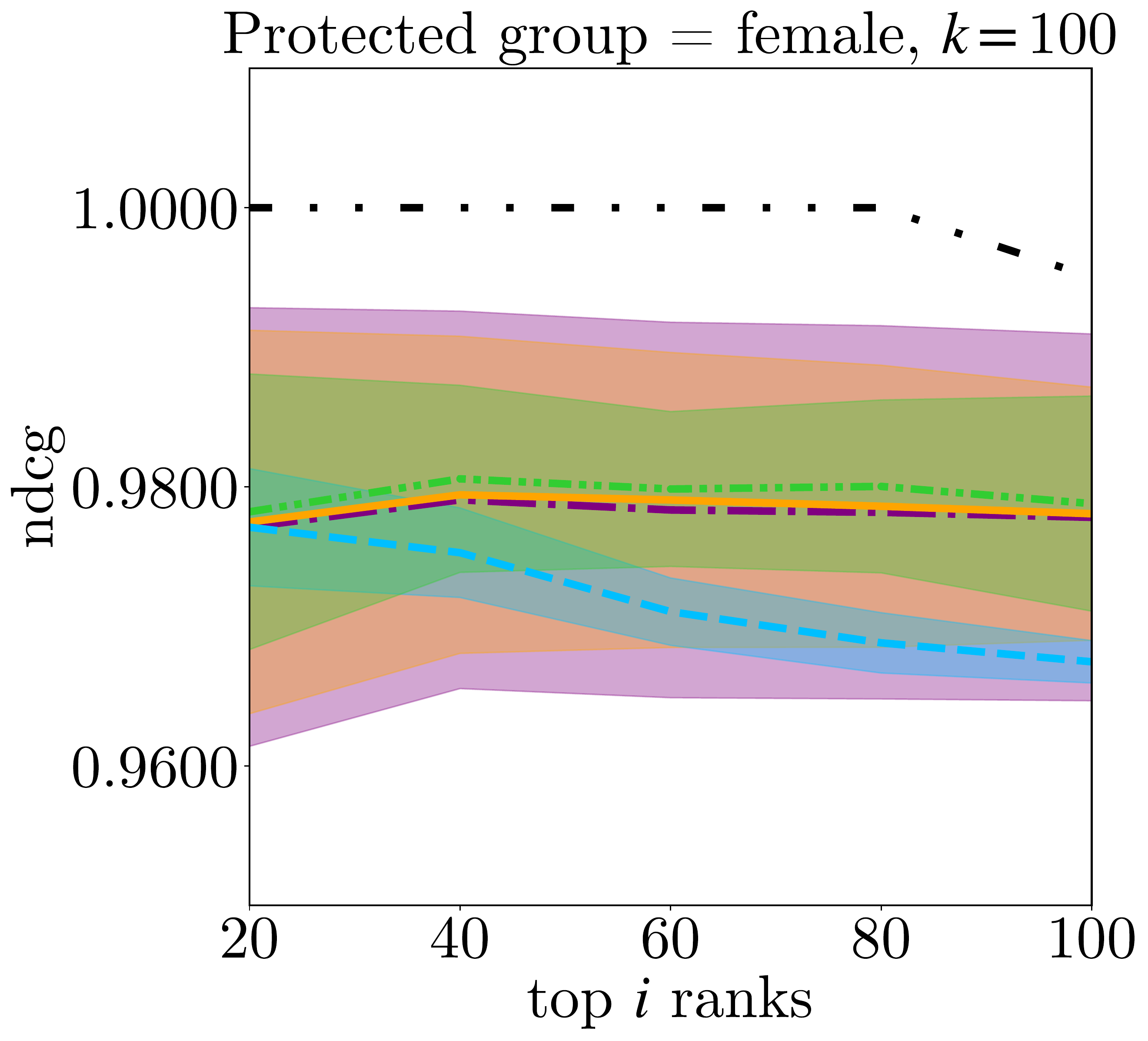} 
		\label{fig:german_35_rep}
	\end{subfigure}
	
	\caption{Results on the JEE 2009 dataset with \textit{gender} as the protected group. For Fair $\epsilon$-greedy we use $\epsilon =0.15$.}
	\label{fig:jee_gender_eps015}
\end{figure*}

\begin{figure*}[t]
	\centering
	\begin{subfigure}[b]{\linewidth}
		\centering
		\includegraphics[scale=0.12]{legend.pdf} 
	\end{subfigure}
	
	\begin{subfigure}[b]{0.25\linewidth}
		\centering
		\includegraphics[scale=0.13]{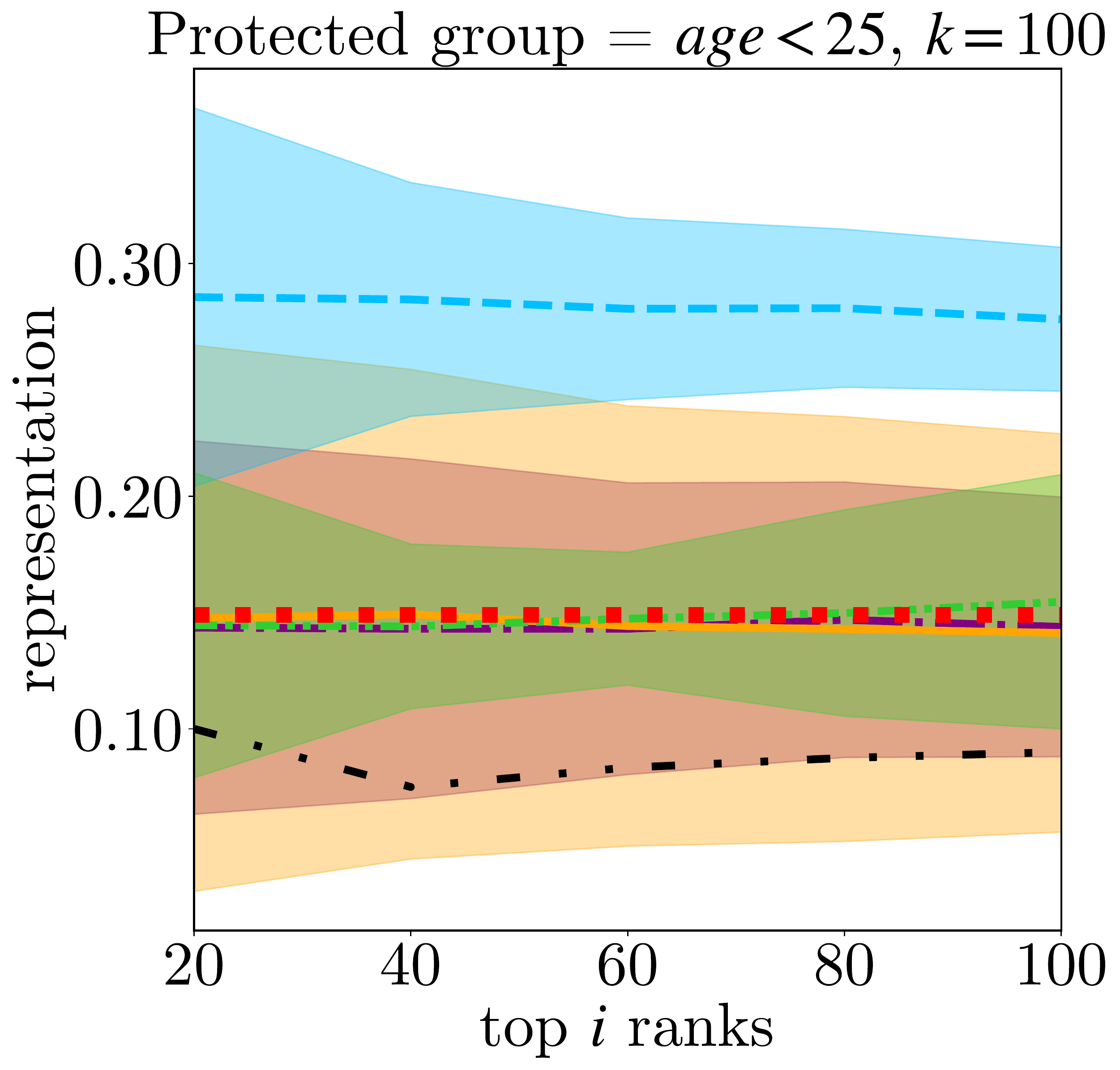} 
		\label{fig:german_25_rep}
	\end{subfigure}
	\begin{subfigure}[b]{0.48\linewidth}
		\centering
		\includegraphics[scale=0.13]{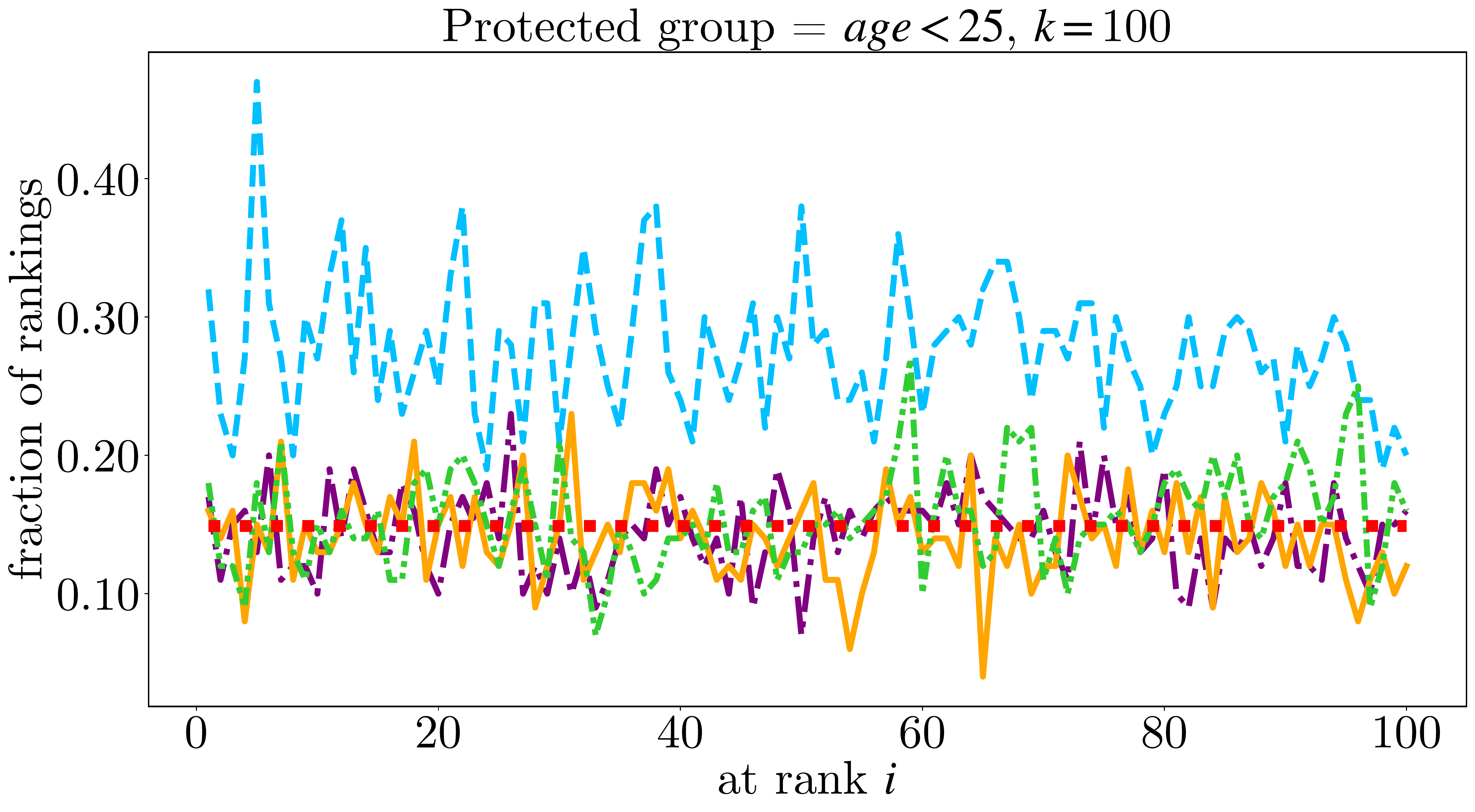} 
		\label{fig:german_35_rep}
	\end{subfigure}
	\begin{subfigure}[b]{0.25\linewidth}
		\centering
		\includegraphics[scale=0.13]{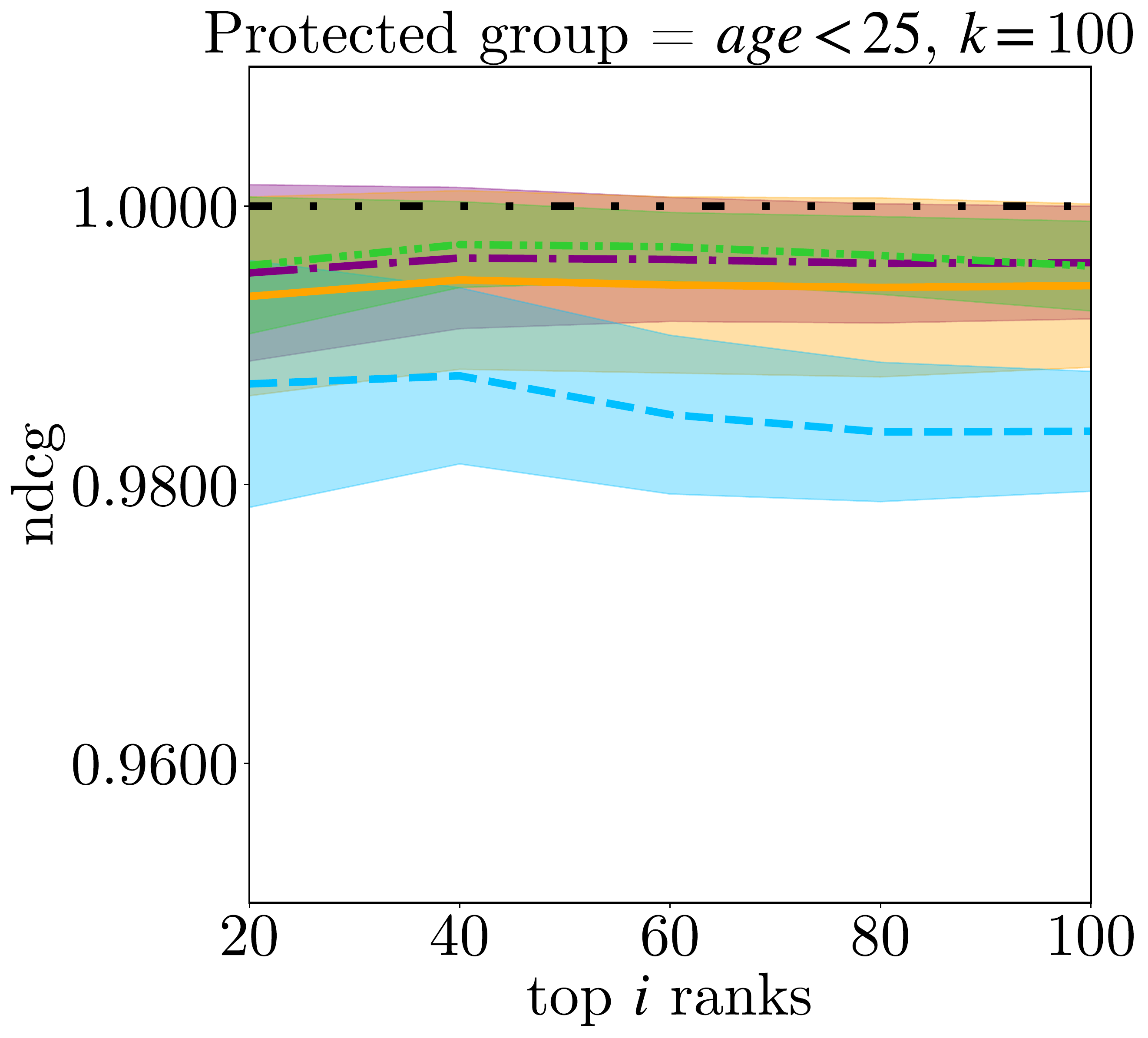} 
		\label{fig:german_35_rep}
	\end{subfigure}

	\begin{subfigure}[b]{0.25\linewidth}
		\centering
		\includegraphics[scale=0.13]{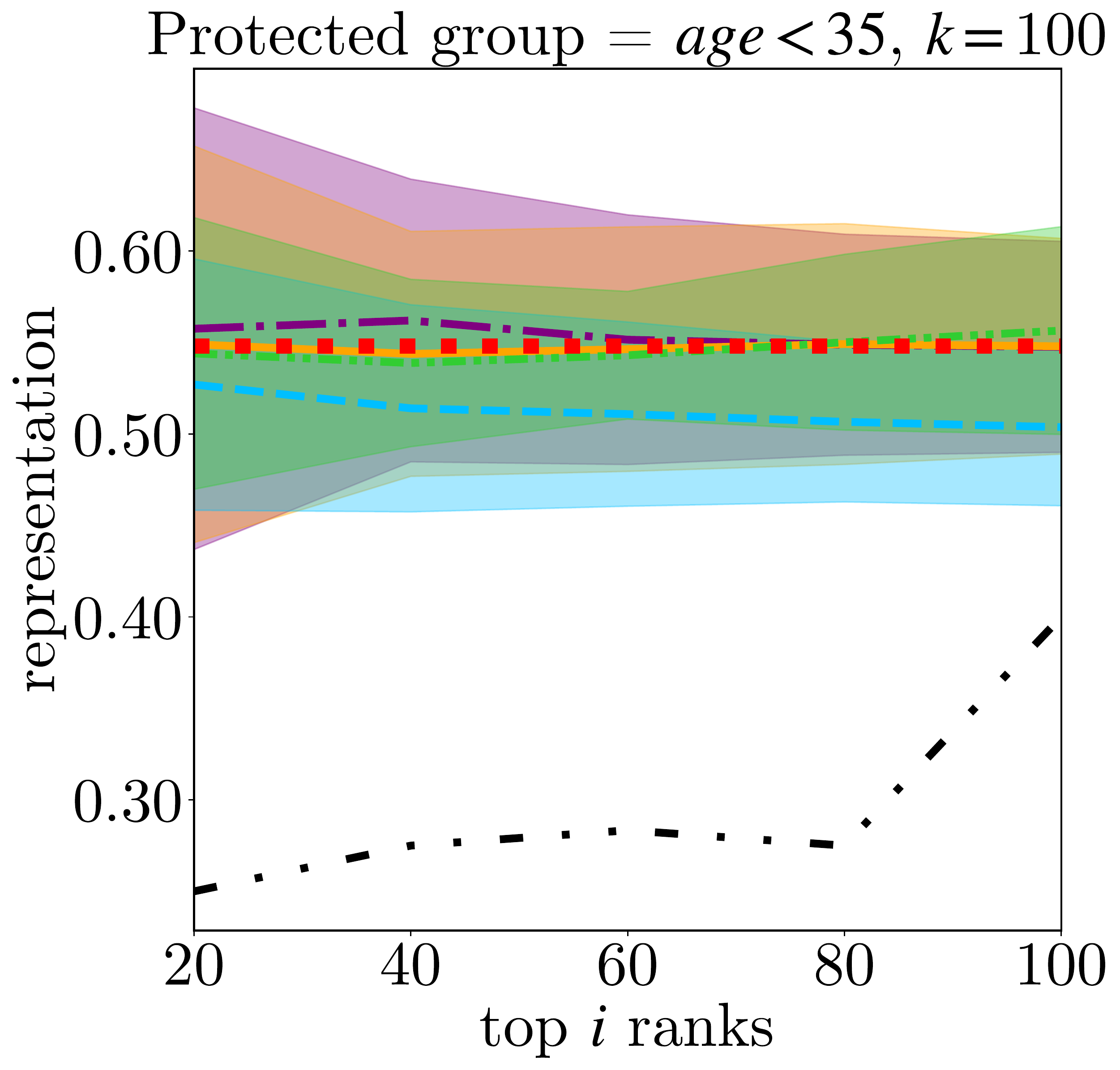} 
		\label{fig:german_25_prop}
	\end{subfigure}
	\begin{subfigure}[b]{0.48\linewidth}
		\centering
		\includegraphics[scale=0.13]{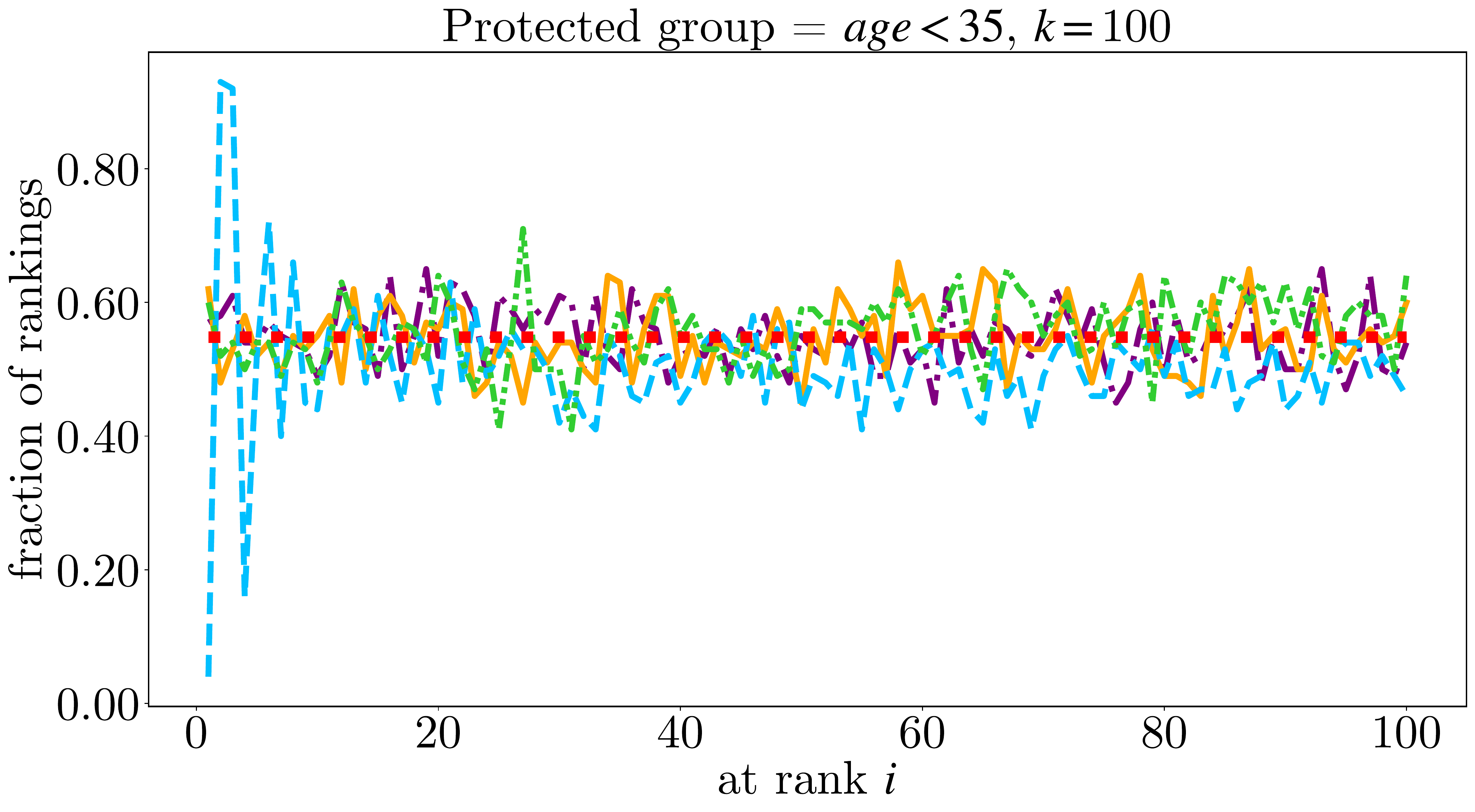} 
		\label{fig:german_35_prop}
	\end{subfigure}
	\begin{subfigure}[b]{0.25\linewidth}
		\centering
		\includegraphics[scale=0.13]{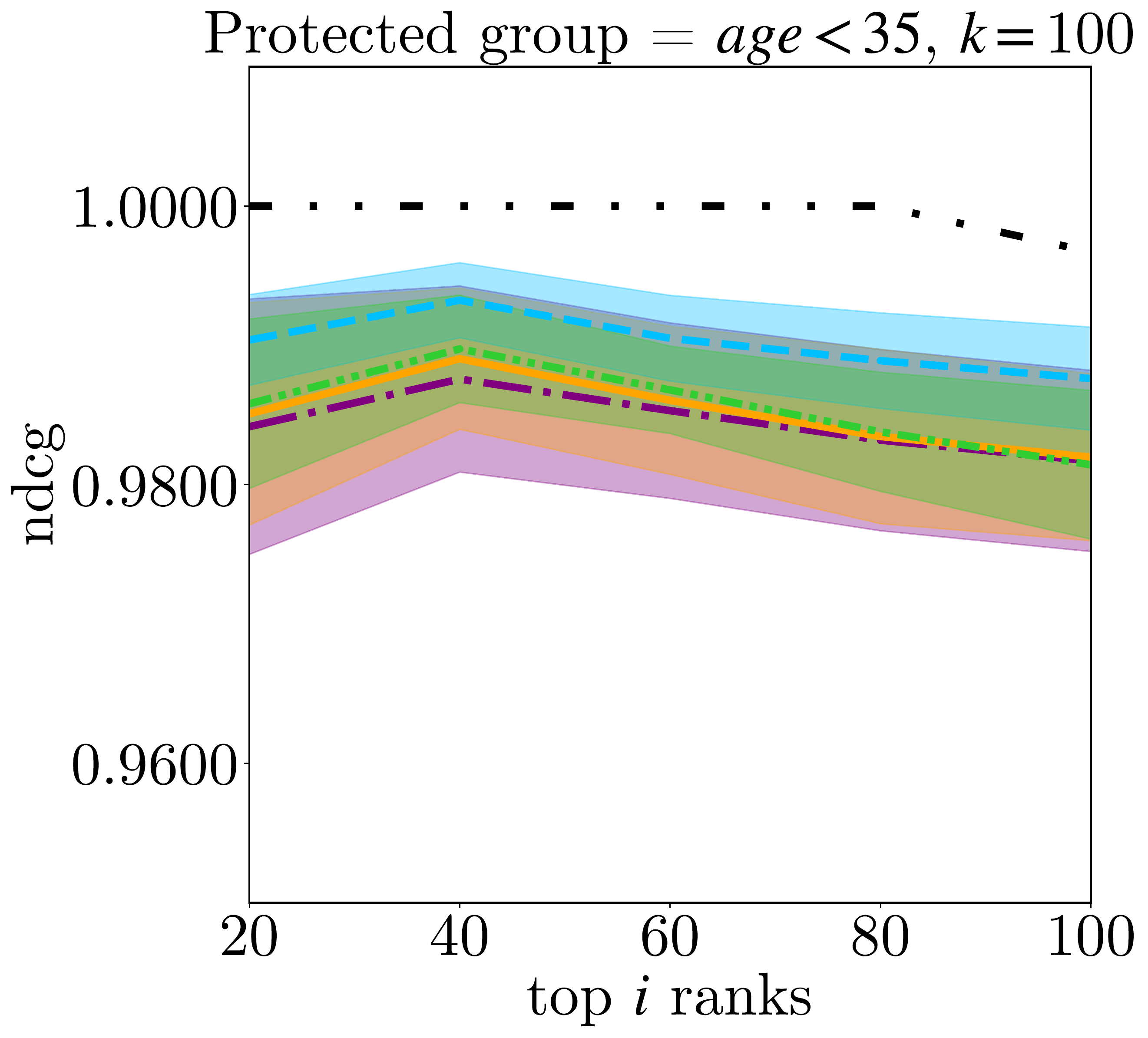} 
		\label{fig:german_35_prop}
	\end{subfigure}

	\caption{Results on the German Credit Risk dataset with \textit{age} $< 25$ as the protected group in the first row and \textit{age} $< 35$ as the protected group in the first row. For Fair $\epsilon$-greedy we use $\epsilon =0.5$.}
	\label{fig:german_binary_eps05}
\end{figure*}
\begin{figure*}[t]
	\centering
	\begin{subfigure}[b]{\linewidth}
		\centering
		\includegraphics[scale=0.12]{legend.pdf} 
	\end{subfigure}
	
	\begin{subfigure}[b]{0.25\linewidth}
		\centering
		\includegraphics[scale=0.13]{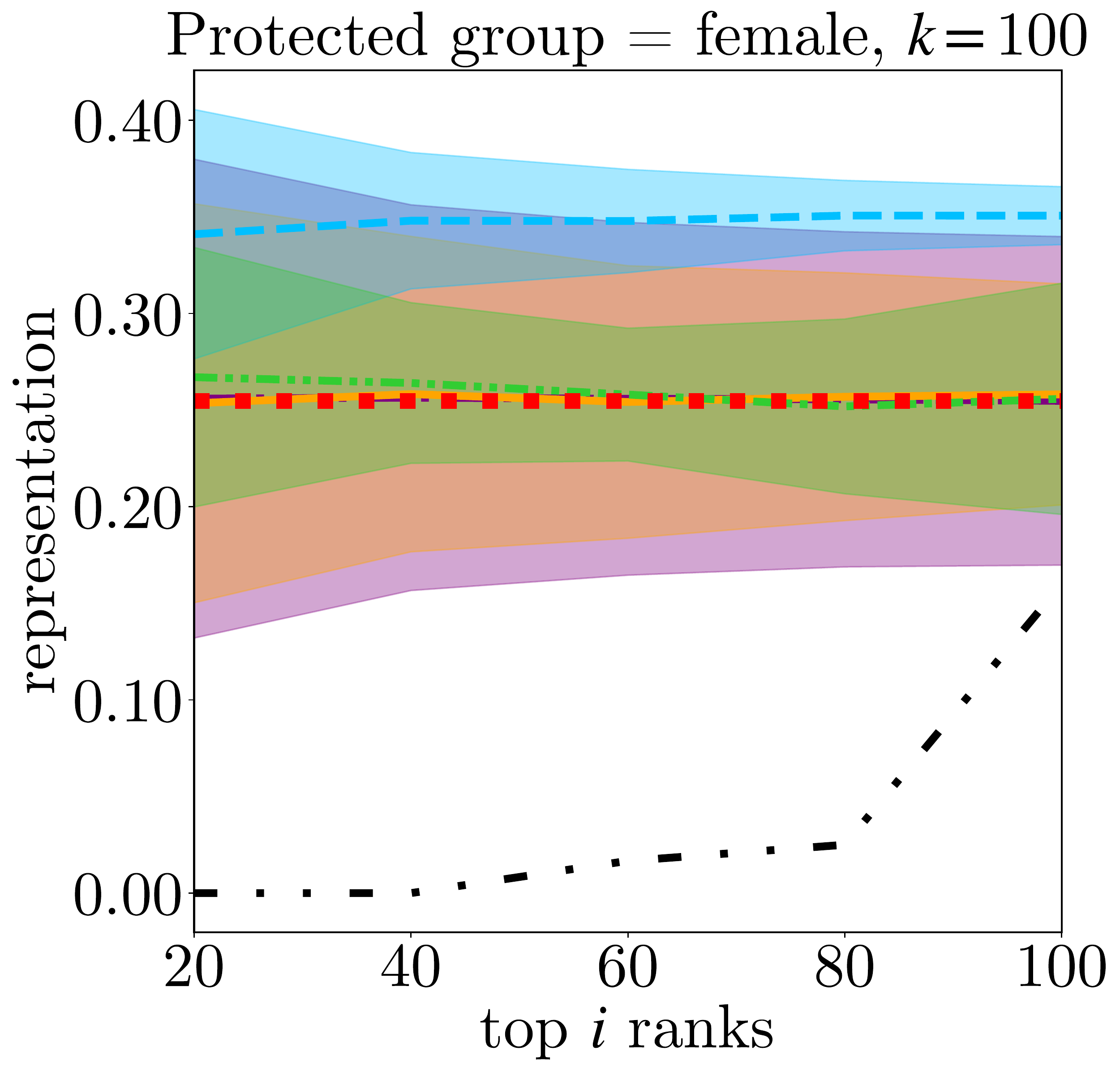} 
		\label{fig:german_25_rep}
	\end{subfigure}
	\begin{subfigure}[b]{0.48\linewidth}
		\centering
		\includegraphics[scale=0.13]{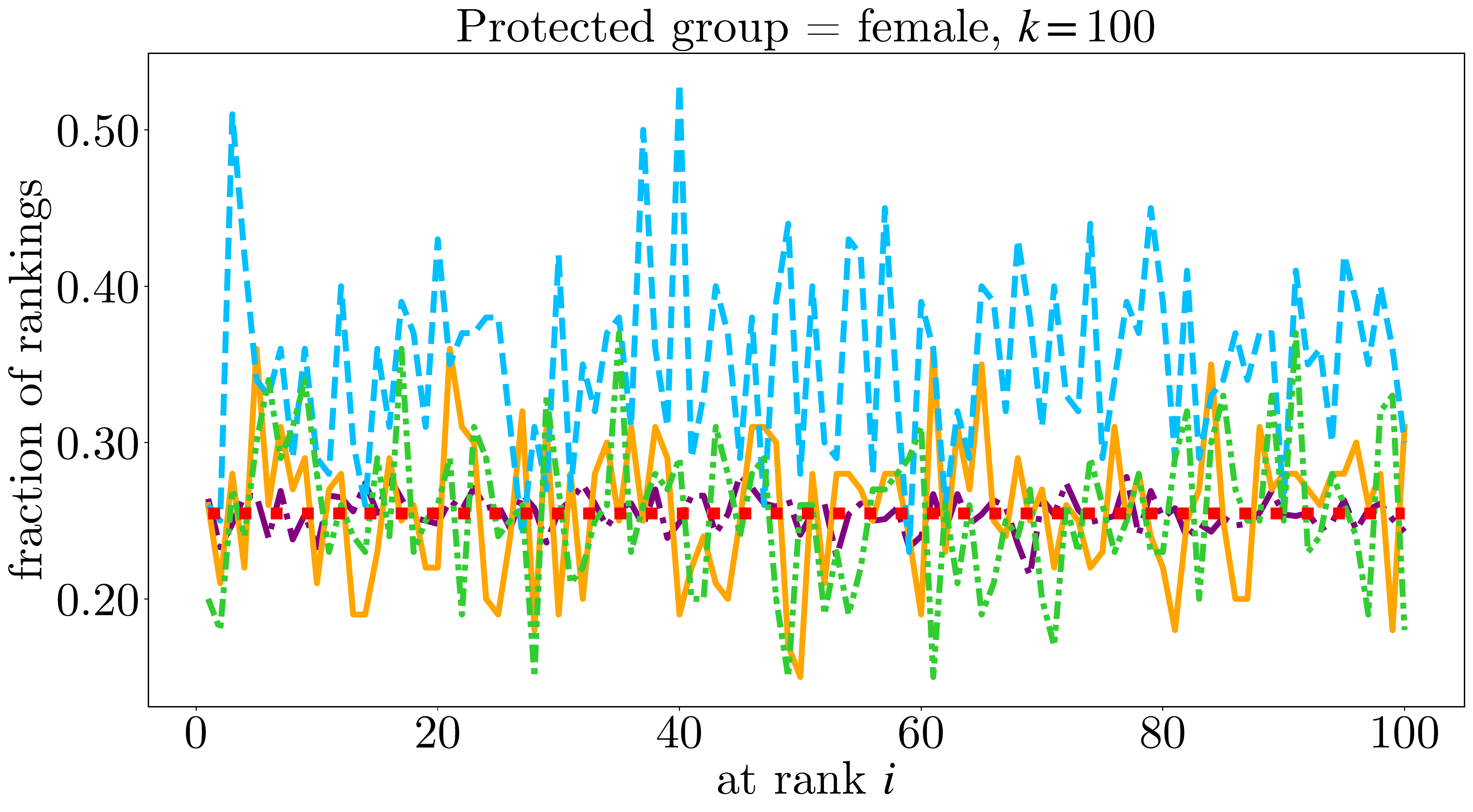} 
		\label{fig:german_35_rep}
	\end{subfigure}
	\begin{subfigure}[b]{0.25\linewidth}
		\centering
		\includegraphics[scale=0.13]{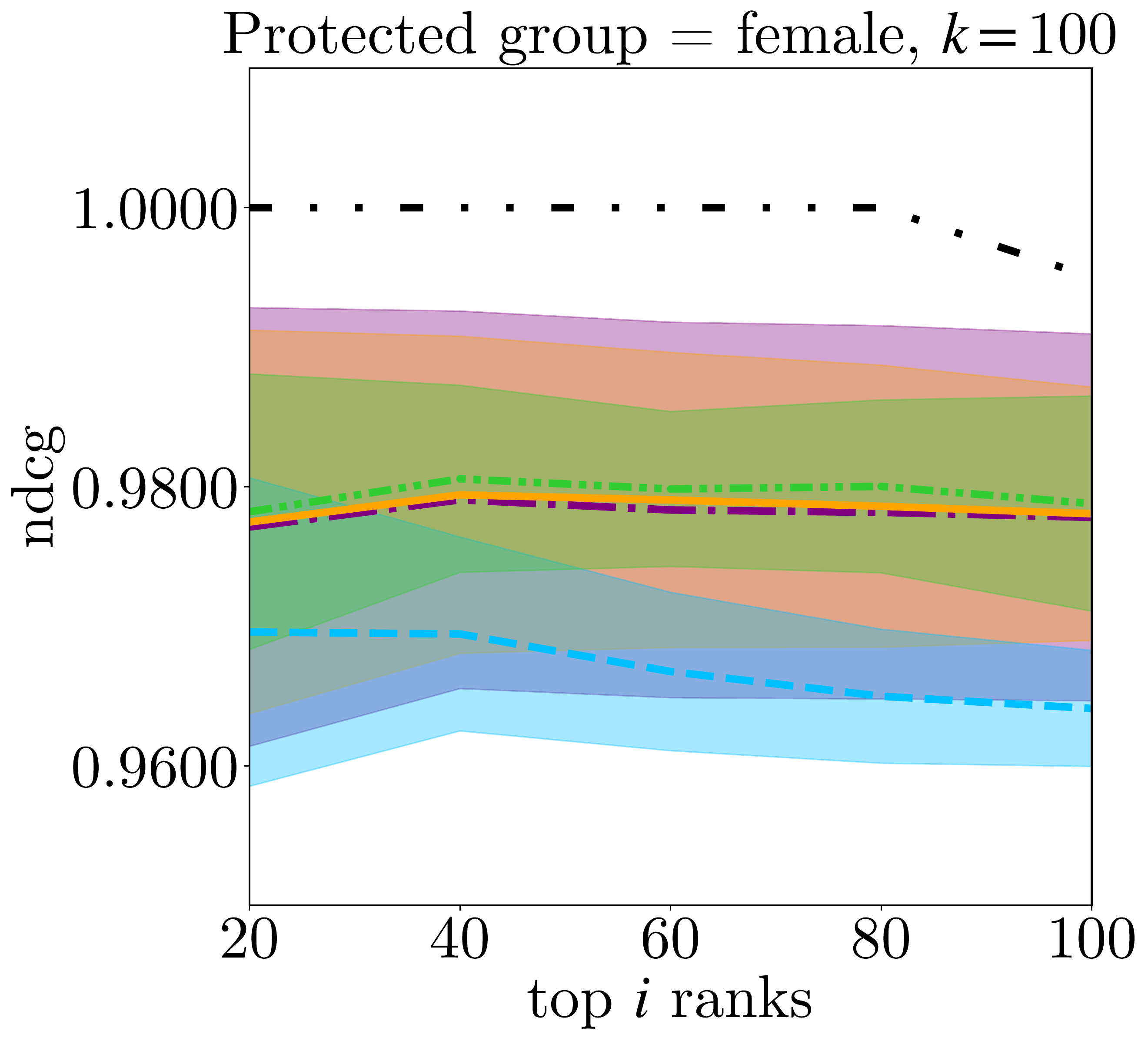} 
		\label{fig:german_35_rep}
	\end{subfigure}
	
	\caption{Results on the JEE 2009 dataset with \textit{gender} as the protected group. For Fair $\epsilon$-greedy we use $\epsilon =0.5$.}
	\label{fig:jee_gender_eps05}
\end{figure*}

\end{document}